%% file: main_neutral.tex
\documentclass{article}

\usepackage[T1]{fontenc}
\usepackage[htt]{hyphenat}

\usepackage{microtype}
\usepackage{graphicx}
\usepackage{booktabs} 
\usepackage{natbib}
\usepackage{algorithm}
\usepackage[noend]{algpseudocode}
\usepackage[ruled,algo2e]{algorithm2e}%

\usepackage[normalem]{ulem}
\usepackage{siunitx}
\usepackage{tocloft}
\usepackage{mathtools}
\usepackage{mathrsfs}
\mathtoolsset{showonlyrefs}

\let\oldaddcontentsline\addcontentsline
\newcommand{\stoptocentries}{\renewcommand{\addcontentsline}[3]{}}
\newcommand{\starttocentries}{\let\addcontentsline\oldaddcontentsline}

\input{preamble}

\definecolor{lightpurple}{RGB}{168, 141, 201}

\newcommand{\fX}{\bm{X}}
\usepackage{xspace}
\newcommand*{\eg}{{\it e.g.}\@\xspace}
\newcommand*{\ie}{{\it i.e.}\@\xspace}
\DeclarePairedDelimiterX{\infdivx}[2]{(}{)}{%
  #1\;\delimsize\|\;#2%
}

\newcommand*{\tran}{^{\mkern-1.5mu\mathsf{T}}}

\usetikzlibrary{positioning}

\usepackage{breakcites}
\usepackage{authblk}
\usepackage[margin=1.3in]{geometry}
\usepackage{enumitem}

\usepackage{hyperref}

\hypersetup{
     colorlinks = true,
     linkcolor = brown,
     anchorcolor = blue,
     citecolor = teal,
     filecolor = blue,
     urlcolor = black
     }

\usepackage{adjustbox}

\title{A Taxonomy of Loss Functions for Stochastic Optimal Control}

\author[1]{Carles Domingo-Enrich}
\affil[1]{Microsoft Research New England\footnote{Work done partially while at Meta FAIR.}}

\begin{document}

\maketitle


\begin{abstract}%
  Stochastic optimal control (SOC) aims to direct the behavior of noisy systems and has widespread applications in science, engineering, and artificial intelligence. In particular, reward fine-tuning of diffusion and flow matching models and sampling from unnormalized methods can be recast as SOC problems. A recent work has introduced Adjoint Matching \citep{domingoenrich2024adjoint}, a loss function for SOC problems that vastly outperforms existing loss functions in the reward fine-tuning setup. The goal of this work is to clarify the connections between all the existing (and some new) SOC loss functions. Namely, we show that SOC loss functions can be grouped into classes that share the same gradient in expectation, which means that their optimization landscape is the same; they only differ in their gradient variance. We perform simple SOC experiments to understand the strengths and weaknesses of different loss functions.
\end{abstract}

\stoptocentries

\section{Introduction}
\label{sec:intro}

Stochastic optimal control (SOC) focuses on guiding the behavior of a system affected by randomness to minimize a specified cost function. This approach has diverse applications across various fields in science and engineering, such as simulating rare events in molecular dynamics \citep{hartmann2014characterization,hartmann2012efficient,zhang2014applications,holdijk2023stochastic}, finance and economics \citep{pham2009continuous,fleming2004stochastic}, stochastic filtering and data assimilation \citep{mitter1996filtering,reich2019data}, nonconvex optimization \citep{chaudhari2018deep}, and power systems and energy markets \citep{belloni2004power,powell2016tutorial}. It is also widely used in robotics \citep{theodorou2011aniterative,gorodetsky2018high}. Additionally, stochastic optimal control has significantly influenced related areas like mean-field games \citep{carmona2018probabilistic}, optimal transport \citep{villani2003topics,villani2008optimal}, backward stochastic differential equations (BSDEs) \citep{carmona2016lectures}, and large deviations \citep{feng2006large}.

Within machine learning, a relevant and recent application of SOC is performing reward fine-tuning of diffusion and flow matching models \citep{domingoenrich2024adjoint,uehara2024finetuning}. Reward fine-tuning is a task in which we are given access to a pretrained generative model that generates $p_{\mathrm{base}}$ and a reward model $r$, and we want to modify the pretrained model so that it generates $p^*(x) \propto p_{\mathrm{base}}(x) \exp(r(x))$. When the model generates discrete data, such as in large language models, the way to achieve this is through KL-regularized reinforcement learning (RL) \citep{ziegler2020finetuning,stiennon2020learning,ouyang2022training,bai2022training}, and the default method is proximal policy optimization (PPO; \cite{schulman2017proximal}). Since SOC is equivalent to KL-regularized RL for systems with continuous-time trajectories described by stochastic differential equations (see e.g. \cite[App.~C]{domingoenrich2024adjoint}), any classical KL-regularized RL algorithms can be translated to an SOC algorithm. However, SOC problems have a distinctive feature with respect to classical KL-regularized RL ones: the reward model $r(x)$ is differentiable because its input is continuous. Exploiting this feature is critical to obtain high-performing deep learning SOC algorithms: \cite[Tab.~1,2,3]{uehara2024finetuning} shows that even a non-state-of-the-art SOC-native algorithm clearly outperforms PPO with KL regularization. 

The state-of-the-art SOC (and overall) method for diffusion and flow reward fine-tuning is Adjoint Matching \citep{domingoenrich2024adjoint}. In fact, Adjoint Matching is the only SOC-based reward fine-tuning algorithm that outperforms alternative fine-tuning algorithms such as DRaFT \citep{clark2024directly} and ReFL \citep{xu2023imagereward}, and does so by a big margin (see \cite[Tab.~2,~Fig.~5]{domingoenrich2024adjoint}). Their work also claims that a particular noise schedule must be used when fine-tuning such models, denoted as the memoryless noise schedule, regardless of the noise schedule that one intends to use at inference time.

These recent developments beg for a systematic overview and comparison of all deep learning methods for SOC problems, both at a theoretical and empirical level. It is critical to understand the strengths and weaknesses of each approach in order to use the appropriate method for each setting. To this end, the contributions of our paper are as follows:
\begin{itemize}
    \item We review existing SOC loss functions (\autoref{subsec:existing_losses_app}) and introduce three novel SOC loss functions (\autoref{subsec:new_losses_app}): the Work-SOCM loss, the Cost-SOCM loss, and the Unweighted SOCM loss.
    \item In \autoref{sec:taxonomy}, we group SOC loss functions into classes that share the same gradient in expectation, which means that their optimization landscape and convergence properties are the same; they only differ in their gradient variance. This results in a taxonomy of loss functions, which we summarize in \autoref{fig:SOC_taxonomy}.
    \item We test the SOC loss functions on a few simple SOC problems that capture relevant features of different practical SOC problems (high dimensions, large costs, multimodality). We observe that despite methods sharing the same gradient in expectation, their convergence speed may differ dramatically due to the gradient variance of each. We also see that different features of the SOC problem favor particular loss functions.
\end{itemize}

\paragraph{Related work} The literature on solving stochastic optimal control problems is rich and storied, starting in the 1950s with the work of Richard Bellman \citep{bellman1957}, that introduced the Hamilton-Jacobi-Bellman (HJB) equation, a partial differential equation (PDE) whose solution determines the optimal control. For low-dimensional control problems ($d \leq 3$), it is possible to grid the domain and use a numerical PDE solver to find a solution to the HJB equation \citep{bonnans2004afast,ma2020finite,jensen2013onthe,debrabant2013semilagrangian,carlini2020asemilagrangian}. To tackle high dimensional problems, works starting in the 2000s designed methods based on forward-backward stochastic differential equations (FBSDEs), initially using least-squares Monte Carlo (see \citet[Chapter 3]{pham2009continuous}, and \citet{gobet2016monte,gobet2005regression,zhang2004numerical}). 
\citet{weinan2017deep,han2018solving} used neural networks to solve SOC problems, leveraging FBSDEs as well. Ever since, a lot of deep learning methods to solve SOC problems have been proposed. We discuss and cite individual methods in \autoref{subsec:existing_losses_app}, and direct the reader to \cite{domingoenrich2023stochastic} and \cite{nüsken2023solving} for more extensive reviews of the related work in this are.

\section{The stochastic optimal control problem}
\label{subsec:soc_problem}

We consider the control-affine problem
\begin{talign} \label{eq:control_problem_def}
    &\min_{u \in \mathcal{U}} \mathbb{E} \big[ \int_0^T 
    \big(\frac{1}{2} \|u(X^u_t,t)\|^2 + f(X^u_t,t) \big) \, \mathrm{d}t + 
    g(X^u_T) \big], \\
    \text{where } &\mathrm{d}X^u_t = (b(X^u_t,t) + \sigma(t) u(X^u_t,t)) \, \mathrm{d}t + \sqrt{\lambda} \sigma(t) \mathrm{d}B_t, \qquad X^u_0 \sim p_0. \label{eq:controlled_SDE}
\end{talign}
and where $X_t^u \in \R^d \to \R^d$ is the state, $u : \R^d \times [0,T]$ is the feedback control and belongs to the set of admissible controls $\mathcal{U}$, $f : \R^d \times [0,T] \to \R$ is the state cost, $g : \R^d \to \R$ is the terminal cost, $b : \R^d \times [0,T] \to \R^d$ is the base drift, and $\sigma : [0,T] \to \R^{d \times d}$ is the invertible diffusion coefficient, and $B=(B_t)_{t\geq 0}$ is a Brownian motion. The objective in equation \eqref{eq:control_problem_def} is known as the control objective.

The \emph{cost functional} is the expected future cost starting from state $x$ at time $t$, using the control $u$:
\begin{talign} \label{eq:cost_functional}
J(u;x,t) := \mathbb{E}_{X} \left[ \int_t^1 
\left(\frac{1}{2} \|u(X_s,s)\|^2  +  f(X_s,s) \right) \, \mathrm{d}s  +  
g(X_1) \;\big|\; X_t = x \right].
\end{talign}
From here, the \emph{value function} is the optimal value of the cost functional%
\footnote{Note that there is a slight difference in terminology between SOC and reinforcement learning, where our cost functional is referred to as the state value function and our value function is the optimal state value function in RL.}
:
\begin{talign}\label{eq:value_fn_defn}
V(x,t) := \min_{u\in \mathcal{U}} J(u;x,t) = J(u^*;x,t),
\end{talign}
where $u^*$ is the \emph{optimal control}, \ie, minimizer of \eqref{eq:control_problem_def}. Furthermore, a classical result is that the value function can be expressed in terms of the \emph{uncontrolled} base process $p^\text{base}$ (\cite{kappen2005path}, see \citealt[Eq.~8,~App.~B]{domingoenrich2023stochastic} for a self-contained proof):
\begin{talign}\label{eq:value_fn_from_uncontrolled}
    V(x, t) = - \log \mathbb{E}_{X} \left[ \exp( - \int_t^1 f(X_s, s) \mathrm{d}s - g(X_1) ) \;\big|\; X_t = x  \right],
\end{talign}
where $X$ denotes the uncontrolled process (the process with $u=0$).
A useful expression for the optimal control is that it is related to the gradient of the value function (see \eg \citealt[App.~B]{domingoenrich2023stochastic}):
\begin{talign} \label{eq:optimal_control}
u^*(x,t) = - \sigma(t)^{\top} \nabla_x V(x,t) = - \sigma(t)^{\top} \nabla_x J(u^*, x,t).
\end{talign}

Solving the SOC problem involves finding the optimal control. The control is parameterized using a neural network. Hence, algorithms are associated to training loss functions. 
All the loss functions that we study can be optimized using a common algorithmic framework, which we describe in Algorithm \ref{alg:IDO}. Essentially, we For more details, we refer the reader to \citet{nüsken2023solving}, which introduced this perspective and named such methods \textit{Iterative Diffusion Optimization} (IDO) techniques. 

\begin{algorithm}[H]
\SetAlgoNoLine 
\SetAlgoNlRelativeSize{0} 
\small{
\KwIn{
Number of iterations $N$, batch size $m$, number of time steps $K$, initial control parameters $\theta_0$, loss function $\mathcal{L}$
}
  \For{$n \in \{0,\dots,N-1\}$}{
    Simulate $m$ trajectories $X^{u_{\theta_n},i}$ of the process $X^{u_{\theta_n}}$ controlled by $u_{\theta_n}$, e.g., using Euler-Maruyama updates

    \lIf{$\mathcal{L}$ is \textit{not} Discrete Adjoint Loss}{detach the $m$ trajectories from the computational graph, so that gradients do not backpropagate}

    Compute an $m$-sample Monte Carlo approximation 
    $\frac{1}{m}\sum_{i=1}^{m} \mathcal{L}(u_{\theta_n};X^{u_{\theta_n},i})$
    of the expected loss $\mathbb{E}_{X^{u_{\theta_n}}} \mathcal{L}(u_{\theta_n}; X^{u_{\theta_n}})$

    Compute the gradient $\nabla_{\theta} 
    \big( \frac{1}{m} \sum_{i=1}^{m} \mathcal{L}(u_{\theta_n};X^{u_{\theta_n},i}) \big)$

    Obtain $\theta_{n+1}$ with via an Adam update on $\theta_n$ (or another stochastic algorithm)
  }
\KwOut{Learned control $u_{\theta_N}$}}
\caption{Iterative Diffusion Optimization (IDO) algorithms for stochastic optimal control}
\label{alg:IDO}
\end{algorithm}

\section{Loss functions for stochastic optimal control}
\label{sec:more_losses}

In this section, we review existing loss functions for stochastic optimal control (\autoref{subsec:existing_losses_app}), and we introduce three new loss functions, providing theoretical guarantees for each of them (\autoref{subsec:new_losses_app}). Unless otherwise indicated, we let $\bar{u} = \texttt{stopgrad}(u)$, which means that the gradients of $\bar{u}$ with respect to the parameters $\theta$ of the control $u$ are artificially set to zero. Some losses admit variants base on the Sticking the Landing (STL) trick (\cite{roeder2017sticking}, \cite[Sec.~2.2.3]{zhou2021actor} see \autoref{subsec:STL}), a variance reduction technique that yields zero variance estimators when trajectories are sampled using the optimal control. \textcolor{blue}{The STL variants correspond to additional terms that we indicate in blue:} removing those terms results in the basic loss function.

\subsection{Existing loss functions}
\label{subsec:existing_losses_app}

\paragraph{The adjoint method} 
The adjoint method utilizes the following objective, which is a Monte Carlo estimate of the control objective \eqref{eq:control_problem_def}:
\begin{talign} \label{eq:L_RE} \mathcal{L}(u ; X^u) := \int_0^T \big(\frac{1}{2} |u(X^u_t,t)|^2 \! + \! f(X^u_t,t) \big) , \mathrm{d}t \! + \! g(X^u_T).
\end{talign} 
The goal is to compute the gradient of $\mathcal{L}(u ; X)$ with respect to the control parameters $\theta$.

For optimize the objective \eqref{eq:L_RE}, two main strategies are used. The \emph{Discrete Adjoint} method follows a “discretize-then-differentiate” approach, retaining the numerical solver in memory for direct differentiation \citep{han2016deep,bierkens2014explicit,chen2016training}. This can be memory-intensive, often requiring gradient checkpointing to reduce usage.

The \emph{Continuous Adjoint} method, however, leverages the continuous nature of SDEs to analytically derive the gradient of the control objective with respect to intermediate states $X_t$. This is represented by an adjoint ODE, following a “differentiate-then-discretize” approach \citep{pontryagin1962mathematical,chen2018neural,li2020scalable}. The adjoint state is defined as:
\begin{talign}
\begin{split} \label{eq:adjoint_state_defn}
&a(t ; X^u, u) := \nabla_{X_t} 
\big(\int_t^T \big(\frac{1}{2} \|u(X^u_{t'},t')\|^2 \! + \! f(X^u_{t'},t') \big) \, \mathrm{d}t' \! + \! g(X^u_1) \big), \\
&\text{where } X^u \text{ solves } \mathrm{d}X^u_t =  \left( b(X^u_t,t) + \sigma(t) u(X^u_t,t) \right) \, \mathrm{d}t + 
    \sigma(t) \mathrm{d}B_t.
\end{split}
\end{talign}
This implies that $\mathbb{E}_{X^u} [ a(t ; X^u, u) | X^u_t = x ] = \nabla_x J(u; x, t)$, where $J$ is the cost functional \eqref{eq:cost_functional}.

The adjoint state evolves according to the dynamics\footnote{Here, the Jacobian matrix $\nabla_x v(x)$ is defined as $\big( \nabla_x v(x) \big)_{ij} = \frac{\partial v_i(x)}{\partial x_j}$.}:
\begin{talign} 
\begin{split} \label{eq:cont_adjoint_1}
    \mathrm{d}a(t;X^u,u) &= - \bigg[ \left(\nabla_{X^u_t} (b (X^u_t,t) + \sigma(t) u(X^u_t,t))\right)\tran{} a(t;X^u,u) 
    \\ &\qquad\quad + \nabla_{X^u_t} \left( f(X^u_t,t) + \frac{1}{2}\|u(X^u_t,t)\|^2 \right) \bigg] \mathrm{d}t \, \textcolor{blue}{- \,
    \nabla_x \bar{u}(X^{\bar{u}}_t,t) \, \mathrm{d}B_t},
\end{split}
    \\ a(T;X^u,u) &= \nabla g(X^u_T). \label{eq:cont_adjoint_2}
\end{talign}
The adjoint state is solved backward in time, and once obtained over the interval $t \in [0, 1]$, the gradient of $\mathcal{L}(u ; \fX)$ with respect to $\theta$ is computed by integrating over the entire time span:
\begin{talign}\label{eq:continuous_adjoint_grads}
    \frac{\mathrm{d} \mathcal{L}}{\mathrm{d} \theta} =  
    \frac{1}{2}\int_0^T \frac{\partial}{\partial \theta} \norm{u(X^{\bar{u}}_t, t)}^2 \mathrm{d} t
    +\int_0^T \frac{\partial u(X^{\bar{u}}_t, t)}{\partial \theta}\tran{} \sigma(t)\tran{} a(t; X^{\bar{u}}, \bar{u}) \mathrm{d} t, \qquad \bar{u} = \texttt{stopgrad}(u),
\end{talign}
The first term in \eqref{eq:continuous_adjoint_grads} corresponds to the direct partial derivative of $\mathcal{L}$ with respect to $\theta$, and the second term captures the derivative through the sampled trajectory $X^u$ \cite[Prop.~6]{domingoenrich2024adjoint}.

Both discrete and continuous adjoint methods converge to the same gradient as the numerical solver step size approaches zero, and scale efficiently to high-dimensional problems. Despite their efficacy in optimizing neural ODE/SDEs, they can be unstable in practice due to the non-convexity of the problem \citep{mohamed2020monte,suh2022differentiable,domingoenrich2023stochastic}. 

An alternative way to derive the loss \eqref{eq:L_RE} is as the KL divergence or relative entropy between $\mathbb{P}^{u}$, the distribution over trajectories induced by the control $u$, and $\mathbb{P}^{u}$, the distribution over trajectories induced by the control $u^*$. That is, $\mathcal{L}(u ; X^u) = \mathbb{E}_{\mathbb{P}^{u^*}} [\log \frac{d\mathbb{P}^{u^*}}{d\mathbb{P}^{u}}]$, which is why some works \citep{nüsken2023solving} refer to this loss as the relative entropy loss. See \citealt[Sec.~5.1]{domingoenrich2024adjoint} for more details on the adjoint method. 

\paragraph{Adjoint Matching} Adjoint Matching was introduced in \citealt[Sec.~5.2]{domingoenrich2024adjoint}, and it is an improvement over the Continuous Adjoint method that is based on two observations. The first one is that the gradient of the Continuous Adjoint loss in \eqref{eq:continuous_adjoint_grads} is equal to the gradient of the basic Adjoint Matching loss, defined as:
\begin{talign}
\begin{split} \label{eq:basic_adjoint_matching}
    \mathcal{L}_{\mathrm{Basic-Adj-Match}}(u; X^{\bar{u}}) &:= \frac{1}{2} \int_0^{T} \big\| u(X^{\bar{u}}_t,t)
    + \sigma(t)\tran{} a(t;X^{\bar{u}},\bar{u}) \big\|^2 \, \mathrm{d}t, \qquad \bar{u} = \texttt{stopgrad}(u),
\end{split}
\end{talign}
The second observation is that the adjoint ODE \eqref{eq:cont_adjoint_1}-\eqref{eq:cont_adjoint_2} can be simplified, giving rise to the lean adjoint ODE, which is as follows:
\begin{talign}
\label{eq:lean_adjoint_1}
\mathrm{d}\tilde{a}(t;X^{\bar{u}},\bar{u}) 
&= - (\tilde{a}(t;X^{\bar{u}},\bar{u})^{\top} \nabla_x b (X^{\bar{u}}_t,t) + \nabla_x f(X^{\bar{u}}_t,t)) \, \mathrm{d}t \textcolor{blue}{- \ 
\nabla_x \bar{u}(X^{\bar{u}}_t,t) \, \mathrm{d}B_t}, \\ 
\label{eq:lean_adjoint_2}
\tilde{a}(T;X^{\bar{u}},\bar{u}) &= \nabla_x g(X^{\bar{u}}_T).
\end{talign}
And then, the Adjoint Matching loss is analogous to the basic Adjoint Matching loss, replacing the adjoint state $a(t;X^{\bar{u}},\bar{u})$ by the lean adjoint state $a(t;X^{\bar{u}})$:
\begin{talign} \label{eq:adjoint_matching}
\mathcal{L}_{\mathrm{Adj-Match}}(u; X^{\bar{u}}) 
:= \frac{1}{2} \int_0^{T} \big\| & u(X^{\bar{u}}_t 
,t)
+ \sigma(t)\tran{} \tilde{a}(t;X^{\bar{u}} 
) \big\|^2 \, \mathrm{d}t, 
\qquad \bar{u} = \texttt{stopgrad}(u).
\end{talign}
While the gradients of the expected losses $\mathbb{E}[\mathcal{L}_{\mathrm{Basic-Adj-Match}}]$ and $\mathbb{E}[\mathcal{L}_{\mathrm{Adj-Match}}]$ are not equal, both losses have the same theoretical guarantees. Namely, the only critical point of the expected losses is the optimal control. A critical point of a loss $\mathcal{L}$ is a control $u$ such that $\frac{\delta}{\delta u} \mathcal{L}(u) = 0$, where $\frac{\delta}{\delta u} \mathcal{L}$ denotes the first variation of the functional $\mathcal{L}$. This guarantee provides the theoretical grounding for gradient-based optimization algorithms to optimize $\mathbb{E}[\mathcal{L}_{\mathrm{Basic-Adj-Match}}]$ and $\mathbb{E}[\mathcal{L}_{\mathrm{Adj-Match}}]$. See \citealt[App.~E.2,~E.3]{domingoenrich2024adjoint} for the proofs.  

\paragraph{The REINFORCE losses} Policy gradient methods, and in particular the archetypal REINFORCE method, are classical RL algorithms that also have analogs in stochastic optimal control, through the connection between stochastic optimal control and KL-regularized reinforcement learning (see \citealt[App.~C]{domingoenrich2024adjoint}).
The regular REINFORCE loss, which we derive in \autoref{prop:derivation_REINFORCE}, reads:
\begin{talign}
\begin{split} \label{eq:REINFORCE}
    \mathcal{L}_{\mathrm{RF}}(u;X^{\bar{u}}) &:= 
    \frac{1}{2} \int_0^T \|u(X^{\bar{u}}_t,t)\|^2 \, \mathrm{d}t \\ &+ 
    \big( \int_0^T 
    \big(\frac{1}{2} \|\bar{u}(X^{\bar{u}}_t,t)\|^2 \! + \! f(X^{\bar{u}}_t,t) \big) \, \mathrm{d}t \! + \! 
    g(X^{\bar{u}}_T) \big) \times 
    \int_0^T \langle u(X^{\bar{u}}_t,t),  \mathrm{d}B_t \rangle, 
\end{split}
\end{talign}
while the expression for the REINFORCE loss that uses only future rewards is as follows:
\begin{talign}
\begin{split} \label{eq:REINFORCE_future_rewards}
    \mathcal{L}_{\mathrm{RFFR}}(u;X^{\bar{u}}) &:= 
    \frac{1}{2} \int_0^T \|u(X^{\bar{u}}_t,t)\|^2 \, \mathrm{d}t \\ &+ 
    \int_0^T \big( \int_t^T 
    \big(\frac{1}{2} \|\bar{u}(X^{\bar{u}}_s,s)\|^2 \! + \! f(X^{\bar{u}}_s,s) \big) \, \mathrm{d}s \! + \! 
    g(X^{\bar{u}}_T) \big) \langle u(X^{\bar{u}}_t,t), \mathrm{d}B_t \rangle, 
\end{split}
\end{talign}

\paragraph{The cross-entropy loss} The cross-entropy loss is defined as the Kullback-Leibler divergence between $\mathbb{P}^{u^*}$ and $\mathbb{P}^{u}$, i.e., flipping the order of the two measures: $\mathbb{E}_{\mathbb{P}^{u^*}} [\log \frac{d\mathbb{P}^{u^*}}{d\mathbb{P}^{u}}]$.
For an arbitrary $v \in \mathcal{U}$, this loss is equivalent to the following one (see e.g. \cite[Prop.~B.6(i)]{domingoenrich2023stochastic}):
\vspace{-3pt}
\begin{talign} 
\begin{split} \label{eq:L_CE}
    &\mathcal{L}_{\mathrm{CE}}(u;X^{\bar{u}}) \! := \! 
    \big( \! - \! 
    \int_0^T \langle u(X^{\bar{u}}_t,t), \mathrm{d}B_t \rangle \! - \! 
    \int_0^T \langle u(X^{\bar{u}}_t,t), \bar{u}(X^{\bar{u}}_t,t) \rangle \, \mathrm{d}t \\ &\qquad\qquad\qquad\qquad\qquad \!
    + \! \frac{1}{2} \int_0^T \|u(X^{\bar{u}}_t,t)\|^2 \, \mathrm{d}t
    \big) \times 
    \alpha(\bar{u},X^{\bar{u}},B), 
\end{split} \\
\begin{split} \label{eq:importance_weight_def}
    &\text{where} \;\; \alpha(\bar{u},X^{\bar{u}},B) = \exp \big( - 
    \big(\int_0^T (\frac{1}{2}  \|\bar{u}(X^{\bar{u} }_t,t)\|^2 + f(X^{\bar{u} }_t,t) \big) \, \mathrm{d}s + g(X^{\bar{u}}_T) - 
    \int_0^T \langle \bar{u}(X^{\bar{u}}_t,t), \mathrm{d}B_t \rangle \big).
\end{split}
\end{talign}
The cross-entropy loss has a rich literature \citep{hartmann2017variational,kappen2016adaptive,rubinstein2013cross,zhang2014applications} and has been recently used in applications such as molecular dynamics \citep{holdijk2023stochastic}. However, the variance of the factor $\alpha$ blows up when 
$f$ or $g$ are large or the dimension is high, which hinders the practical applicability of the loss in large-scale settings.

\paragraph{Variance and log-variance losses} 
The \textit{variance} and the \textit{log-variance losses} are defined as $\tilde{\mathcal{L}}_{\mathrm{Var}_{\bar{u}}}(u) = \mathrm{Var}_{\mathbb{P}^{\bar{u}}} ( \frac{\mathrm{d}\mathbb{P}^{u^*}}{\mathrm{d}\mathbb{P}^{\bar{u}}} )$ and $\tilde{\mathcal{L}}_{\mathrm{Var}_{\bar{u}}}^{\mathrm{log}}(u) = \mathrm{Var}_{\mathbb{P}^{\bar{u}}} ( \log \frac{\mathrm{d}\mathbb{P}^{u^*}}{\mathrm{d}\mathbb{P}^u} )$
whenever $\mathbb{E}_{\mathbb{P}^{\bar{u}}} | \frac{\mathrm{d}\mathbb{P}^{u^*}}{\mathrm{d}\mathbb{P}^u} | < + \infty$ and $\mathbb{E}_{\mathbb{P}^{\bar{u}}} | \log \frac{d\mathbb{P}^{u^*}}{d\mathbb{P}^u} | < + \infty$, respectively. Define
\begin{talign}
    \begin{split} \label{eq:tilde_Y_def}
    \tilde{Y}_T^{u,\bar{u}} \! &= \! - 
    \int_0^T \langle u(X^{\bar{u}}_t,t), \bar{u}(X^{\bar{u}}_t,t) \rangle \, \mathrm{d}t \! - \! 
    \int_0^T f(X^{\bar{u}}_t,t) \, \mathrm{d}t \! - \! 
    \int_0^T \langle u(X^{\bar{u}}_t,t), \mathrm{d}B_t \rangle 
    \! + \! \frac{
    1}{2} \int_0^T \| u(X^{\bar{u}}_t,t)\|^2 \, \mathrm{d}t. 
    \end{split}
\end{talign}
Then, $\tilde{\mathcal{L}}_{\mathrm{Var}_{\bar{u}}}$ and $\tilde{\mathcal{L}}_{\mathrm{Var}_{\bar{u}}}^{\mathrm{log}}$ are equivalent, respectively, to the following losses (see e.g. \cite[Prop.~B.6(i)]{domingoenrich2023stochastic}):
\begin{talign} \label{eq:variance_loss}
\mathcal{L}_{\mathrm{Var}_{\bar{u}}}(u) &:=
\mathrm{Var} \big( \exp \big( \tilde{Y}^{u,\bar{u}}_T - 
g(X^{\bar{u}}_T) \big) \big), \\ \mathcal{L}^{\mathrm{log}}_{\mathrm{Var}_{\bar{u}}}(u) &:= \mathrm{Var}\big( \tilde{Y}^{u,{\bar{u}}}_T - 
g(X^{\bar{u}}_T) \big), \label{eq:log_variance_loss}
\end{talign}
The variance and log-variance losses were introduced by \citet{nüsken2023solving}. 
Unlike for the cross-entropy loss, the choice of the control $\bar{u}$ does lead to different losses. 
When using $\mathcal{L}_{\mathrm{Var}_{\bar{u}}}$ or $\mathcal{L}^{\mathrm{log}}_{\mathrm{Var}_{\bar{u}}}$ in Algorithm \ref{alg:IDO}, the variance is computed across the $m$ trajectories in each batch. Note that when $f,g$ are large, 
$\mathcal{L}_{\mathrm{Var}_{\bar{u}}}$ is very large and has itself very large variance, which causes issues.

\paragraph{Moment loss} The moment loss is defined as
\begin{talign} \label{eq:moment_loss}
    \mathcal{L}_{\mathrm{Mom}_{\bar{u}}}(u,y_0) = 
    (\tilde{Y}^{u,\bar{u}}_T + y_0 - 
    g(X^{\bar{u}}_T))^2,
\end{talign}
where $\tilde{Y}^{u,{\bar{u}}}_T$ is defined in \eqref{eq:tilde_Y_def}. Note the similarity with the log-variance loss \eqref{eq:log_variance_loss}; the optimal value of $y_0$ for a fixed $u$ is $y_0^* = \mathbb{E}[
g(X^{\bar{u}}_T) - \tilde{Y}^{u,\bar{u}}_T]$, and plugging this into the expectation of \eqref{eq:moment_loss} yields exactly the log-variance loss.  The moment loss was introduced by \citet[Section III.B]{hartmann2019variational}, generalizing the FBSDE method pioneered by \citet{weinan2017deep,han2018solving}, which corresponds to setting $\bar{u} = 0$.

\paragraph{Stochastic optimal control matching (SOCM) loss} This loss, introduced by \cite{domingoenrich2023stochastic}, is formally similar to the Continuous Adjoint loss, in that the control $u(x,t)$ is trained to approximate a certain vector field. The SOCM is as follows:
\begin{talign}
    \begin{split} \label{eq:SOCM_loss}
        \mathcal{L}_{\mathrm{SOCM}}(u,M) &:= 
        \int_0^{T} \big\| u(X^{\bar{u}}_t,t)
        + \sigma(t)^{\top} 
        \omega(t,\bar{u},X^{\bar{u}},B,M_t) \big\|^2 \, \mathrm{d}t  
        \times \alpha(\bar{u}, X^{\bar{u}}, B) \big]~,
    \end{split}
    \end{talign}
    where $\alpha(\bar{u}, X^{\bar{u}}, B)$ is defined in \eqref{eq:importance_weight_def}, and
    \begin{talign} \label{eq:matching_vector_field_def}
        \begin{split}
        \omega(t,\bar{u},X^{\bar{u}},B,M_t) &= \int_t^T M_t(s) \nabla_x f(X^{\bar{u}}_s,s) \, \mathrm{d}s + M_t(T) \nabla g(X^v_T) 
        \\ &\quad - \int_t^T (M_t(s) \nabla_x b(X^v_s, s) -\partial_s M_t(s)) (\sigma^{-1})^{\top}(s) \bar{u}(X^{\bar{u}}_s,s) \, \mathrm{d}s
        \\ &\quad - 
        \int_t^T (M_t(s) \nabla_x b(X^v_s, s) - \partial_s M_t(s)) (\sigma^{-1})^{\top}(s) \mathrm{d}B_s,
        \end{split}
    \end{talign}
    and $M$ satisfy the following assumption:
    \begin{assumption} \label{ass:M}
        For each $t \in [0,T]$, $M_t : [t, T ] \to \R^{d \times d}$ is an arbitrary matrix-valued differentiable function such that $M_t(t) = \mathrm{Id}$. More generally, we can take $M_t$ such that for any $s \in [t,T]$, $M_t(s)$ depends on the trajectory $X^{\bar{u}}$ up to time $s$. 
    \end{assumption}
    Observe that $u(x,t)$ is trained to approximate 
    \begin{talign} \label{eq:u_star_rewritten}
    \frac{- \sigma(t)^{\top}\mathbb{E}[\omega(t,\bar{u},X^{\bar{u}},B,M_t) \alpha(\bar{u},X^{\bar{u}},B) | X^{\bar{u}}_t=x]}{\mathbb{E}[ \alpha(\bar{u},X^{\bar{u}},B) | X^{\bar{u}}_t=x]}. 
    \end{talign}
    Leveraging the path-wise reparameterization trick (\cite[Prop.~1]{domingoenrich2023stochastic}, \autoref{lem:cond_exp_rewritten}, \autoref{prop:cond_exp_rewritten}), one sees that for any $\bar{u} \in \mathcal{U}$ and $M$ satisfying \autoref{ass:M}, the term in \eqref{eq:u_star_rewritten} is equal to the optimal control $u^*(x,t)$. Minimizing the loss $\mathcal{L}_{\mathrm{SOCM}}$ with respect to $M$ serves the purpose of minimizing the variance of the target vector field, making it easier and faster to learn. This is the reason the SOCM loss outperforms all other losses in simple experiments (\autoref{sec:simple_exp}). On the flip side, like for the cross entropy loss, the variance blowup of the factor $\alpha$ makes this loss impractical at large scale.

    \paragraph{SOCM-Adjoint loss} The SOCM-Adjoint loss, introduced by \cite{domingoenrich2023stochastic}, is of this form:
\begin{talign}
    \begin{split} \label{eq:SOCM_adjoint_loss}
        \mathcal{L}_{\mathrm{SOCM-Adj}}(u) &:= 
        \int_0^{T} \big\| u(X^{\bar{u}}_t,t)
        + \sigma(t)^{\top} \tilde{a}(t,X^{\bar{u}}) \big\|^2 \, \mathrm{d}t 
        \times \alpha(\bar{u}, X^{\bar{u}}, B),
    \end{split} 
    \end{talign}
    where $\alpha(\bar{u},X^{\bar{u}},B)$ is the importance weight defined in \eqref{eq:importance_weight_def}, 
    and $\tilde{a}(t,X^{\bar{u}})$ is the solution of the lean adjoint ODE \eqref{eq:lean_adjoint_1}-\eqref{eq:lean_adjoint_2}.
    Remark that the only difference between the SOCM-Adjoint loss and the Adjoint Matching loss is the importance weight $\alpha(\bar{u},X^{\bar{u}},B)$; the variance of this factor blows up when 
    $f$ or $g$ are large or the dimension is high, which hinders the practical applicability of the loss in large-scale settings.
    Observe also that $u(x,t)$ is trained to approximate the vector field $\frac{- \sigma(t)^{\top}\mathbb{E}[\tilde{a}(t,X^v) \alpha(\bar{u},X^{\bar{u}},B) | X^{\bar{u}}_t=x]}{\mathbb{E}[ \alpha(\bar{u},X^{\bar{u}},B) | X^{\bar{u}}_t=x]}$. 
    which is equal to the optimal control $u^*(x,t)$, because $\frac{\mathbb{E}[\tilde{a}(t,X^{\bar{u}}) \alpha(\bar{u},X^{\bar{u}},B) | X^{\bar{u}}_t=x]}{\mathbb{E}[ \alpha(\bar{u},X^{\bar{u}},B) | X^{\bar{u}}_t=x]} = \nabla V(x,t)$ for any $v \in \mathcal{U}$. 

\subsection{New loss functions}
\label{subsec:new_losses_app}

\paragraph{Work-SOCM loss} In this loss, the control $u$ is also trained to approximate the vector field $-\sigma(t)^{\top} \mathbb{E}[\tilde{\xi}(t,X^{\bar{u}},B,M_t)|X^{\bar{u}}_t = x]$: 
\begin{talign}
\begin{split} \label{eq:Q_learning_2}
    \mathcal{L}_{\mathrm{Work-SOCM}}(u,M) &:= 
    \int_0^{T} \big\| u(X^{\bar{u}}_t,t) 
    + \sigma(t)^{\top} \tilde{\xi}(t,X^{\bar{u}},B,M_t) \big\|^2 \, \mathrm{d}t,
    \quad \bar{u} = \texttt{stopgrad}(u),
\end{split}
\end{talign}
where
\begin{talign}
\begin{split}
    \tilde{\xi}(t,X^v,B,M_t) &= \int_t^T M_t(s) \nabla_x 
    f(X^v_s,s)
    \, \mathrm{d}s + M_t(T) \nabla g(X^v_T) 
    \\ &\quad + \! 
    \big( \int_t^T 
    f(X^v_s,s)  
    \, \mathrm{d}s \! + \! g(X^v_T) 
    \big)
    \times \! \big( \int_t^T (M_t(s) \nabla_x b(X^v_s,s) 
    - \partial_s M_t(s)) (\sigma^{-1})^{\top}(s) \mathrm{d}B_s \big).
\end{split}    
\end{talign}
and $M$ is a family of matrix-valued functions that satisfies \autoref{ass:M}.
That is, $u(x,t)$ is trained to approximate $-\sigma(t)^{\top} \mathbb{E}[\tilde{\xi}(t,X^{\bar{u}},B,M_t)|X^{\bar{u}}_t = x]$. The following proposition connects the Work-SOCM loss to the Adjoint Matching loss, which allows us to deduce theoretical guarantees for the former.

\begin{proposition} \label{prop:work_adjoint_app}
    The gradients of the losses $\mathcal{L}_{\mathrm{Adj-Match}}$ and $\mathcal{L}_{\mathrm{Work-SOCM}}$ are equal in expectation, and in particular, for any $x \in \mathbb{R}^d$, $t \in T$, and $M$ fulfilling \autoref{ass:M}, we have that
    \begin{talign} \label{eq:tilde_a_tilde_xi}
        \mathbb{E}[\tilde{a}(t,X^{\bar{u}}) | X^{\bar{u}}_t = x] = \mathbb{E}[\tilde{\xi}(t,X^{\bar{u}},B,M_t) | X^{\bar{u}}_t = x] = 
        \mathbb{E}[\nabla_{X_t} \big(\int_t^T f(X_s,s) \, \mathrm{d}s + g(X_T) \big) \frac{\mathrm{d}\mathbb{P}^v}{\mathrm{d}\mathbb{P}}(X) |X_t = x].
    \end{talign}
    Hence, the only critical point of the loss $\mathcal{L}_{\mathrm{Work-SOCM}}$ is the optimal control $u^*$.
\end{proposition}

The reason for the term \textit{work} in the name of this loss is that equation \eqref{eq:tilde_a_tilde_xi} involves the work functional $\mathcal{W}(X,t) = \int_t^T f(X_s,s) \, \mathrm{d}s + g(X_T)$.

\paragraph{Cost-SOCM loss} The control $u$ is trained to approximate the vector field $-\sigma(t)^{\top} \mathbb{E}[\xi(t,X^{\bar{u}},B,M_t)|X^{\bar{u}}_t = x]$:
\begin{talign}
\begin{split} \label{eq:Q_learning}
    \mathcal{L}_{\mathrm{Cost-SOCM}}(u,M) &:= 
    \int_0^{T} \big\| u(X^v_t,t) 
    + \sigma(t)^{\top} \xi(t,\bar{u},X^{\bar{u}},B,M_t) \big\|^2 \, \mathrm{d}t, 
    \quad \bar{u} = \texttt{stopgrad}(u),
\end{split}
\end{talign}
where 
\begin{talign} 
\begin{split} \label{eq:cost_SOCM_matching_vector_field_def}
    &\xi(t,\bar{u},X^{\bar{u}},B,M_t) \\ &= \int_t^T M_t(s) \nabla_x (f(X^{\bar{u}}_s,s) + \frac{1}{2}\|\bar{u}(X^{\bar{u}}_s,s)\|^2) \, \mathrm{d}s + M_t(T) \nabla g(X^{\bar{u}}_T) \, \textcolor{blue}{+ \, 
    \int_t^T M_t(s) \nabla_x \bar{u}(X^{\bar{u}}_s,s) \, \mathrm{d}B_s}
    \\ &\qquad - \! 
    \big( \int_t^T (f(X^{\bar{u}}_s,s) \! + \! \frac{1}{2}\|\bar{u}(X^{\bar{u}}_s,s)\|^2) \, \mathrm{d}s \! + \! g(X^{\bar{u}}_T) \textcolor{blue}{ \, + \, 
    \int_t^T \bar{u}(X^{\bar{u}}_s,s) \, \mathrm{d}B_s} \big) \\ &\qquad\qquad\quad \times \! \big( \int_t^T (M_t(s) \nabla_x (b(X^{\bar{u}}_s,s) + \sigma(s) \bar{u}(X^{\bar{u}}_s,s)) - \partial_s M_t(s)) (\sigma^{-1})^{\top}(s) \mathrm{d}B_s \big),
\end{split}
\end{talign}
and $M$ is a family of matrix-valued functions that satisfies \autoref{ass:M}.
As shown by \autoref{prop:SOCM_cost}, the expectation $\mathbb{E}[\xi(t,X^{\bar{u}},B,M_t)|X^{\bar{u}}_t = x]$ is equal to the gradient of the expected cost functional $\nabla_x J(\bar{u};x,t)$, which means that in expectation, $\mathcal{L}_{\mathrm{Cost-SOCM}}$ has the same gradients as $\mathcal{L}_{\mathrm{Cont-Adj}}$. 
\begin{proposition} \label{prop:SOCM_cost}
In expectation, the gradients of the loss $\mathcal{L}_{\mathrm{Cost-SOCM}}$ are equal to the gradients of the continuous adjoint loss $\mathcal{L}_{\mathrm{Cont-Adj}}$. In particular, we have that 
\begin{talign}
\mathbb{E}[\xi(t,\bar{u},X^{\bar{u}},B,M_t) | X^{\bar{u}}_t = x] = \nabla \mathbb{E} \big[ \int_0^T \big(\frac{1}{2} \|\bar{u}(X^{\bar{u}}_t,t)\|^2 \! + \! f(X^{\bar{u}}_t,t) \big) \, \mathrm{d}t \! + \! g(X^{\bar{u}}_T) | X^{\bar{u}}_0 = x \big].
\end{talign}
\end{proposition}

\paragraph{Unweighted SOCM loss} As implied by its name, the difference between the Unweighted SOCM loss and the SOCM loss is that the former does not include the factor $\alpha$:
\begin{talign}
\begin{split} \label{eq:UW_SOCM}
    \mathcal{L}_{\mathrm{UW-SOCM}}(u,M) &:= 
    \int_0^{T} \big\| u(X^{\bar{u}}_t,t) 
    + \sigma(t)^{\top} \omega(t,\bar{u},X^{\bar{u}},B,M_t) \big\|^2 \, \mathrm{d}t, 
    \qquad \bar{u} = \texttt{stopgrad}(u).
\end{split}
\end{talign}
The upside of removing $\alpha$ is that the gradient variance does not blow up even with large scale problems. The downside is that the theoretical guarantees of this loss are weak: 
\begin{proposition} \label{prop:UW_SOCM_optimality}
    For any $M$, the optimal control $u^*$ is a critical point of $\mathcal{L}_{\mathrm{UW-SOCM}}(\cdot,M)$. 
\end{proposition}
\autoref{prop:UW_SOCM_optimality} does not preclude the existence of other critical points beyond $u^*$. In fact, from the expression of $\omega$ \eqref{eq:matching_vector_field_def}, we see that if we choose $M_t$ such that $M_t(T) = 0$, the loss becomes underspecified because it has no information about the terminal cost $g$ (the optimal control for any cost $g$ is a critical point!). In practice, UW-SOCM is the best-performing algorithm in several experimental settings of \autoref{sec:simple_exp}, but fails to convergence in other ones.

\section{A taxonomy of loss functions} \label{sec:taxonomy}

The SOC losses presented in \autoref{subsec:existing_losses_app} and \autoref{subsec:new_losses_app} are connected to each other in multiple ways. These connections have been partially addressed in those sections, but it is useful for researchers to have a clear, systematic, complete picture of all the losses that are available. \autoref{thm:main} specifies the sets of losses that have equal gradients in expectation, up to numerical errors (see proofs in \autoref{sec:proofs_taxonomy}).  

\begin{theorem}[A taxonomy of SOC losses] \label{thm:main}
    In expectation and in the limit where SDE simulation stepsize goes to zero, the gradients of the following sets of loss functions are equal up to constant factors: 
    \begin{itemize}[leftmargin=0.5cm]
    \item Class I: Discrete Adjoint \eqref{eq:L_RE}, Continuous Adjoint \eqref{eq:basic_adjoint_matching}, REINFORCE \eqref{eq:REINFORCE}, REINFORCE (future rewards) \eqref{eq:REINFORCE_future_rewards}, 
    Cost-SOCM
    \eqref{eq:Q_learning}.
    \item Class II: Adjoint Matching \eqref{eq:adjoint_matching} and Work-SOCM \eqref{eq:Q_learning_2}.
    \item Class III: SOCM \eqref{eq:SOCM_loss}, SOCM-Adjoint \eqref{eq:SOCM_adjoint_loss}, Cross Entropy \eqref{eq:L_CE}.
    \item Class IV: Log-variance \eqref{eq:log_variance_loss}, Moment \eqref{eq:moment_loss} (if batch size $\to \infty$ in the former and $y_0$ is optimized instantaneously in the latter).
    \item Class V: Variance \eqref{eq:variance_loss}.
    \item Class VI: Unweighted SOCM \eqref{eq:UW_SOCM}.
    \end{itemize}
\end{theorem}

\begin{figure}[h]
\hspace{-1cm}
\begin{tikzpicture}[
    blocka/.style={draw, thick,
    text width=1.09\textwidth, 
    rectangle, rounded corners, minimum height=1cm},
    blockb/.style={draw, thick, 
    text width=1.09\textwidth, 
    rectangle, rounded corners, minimum height=1cm},
    blockc/.style={draw, thick, 
    text width=0.155\textwidth, 
    align=left, 
    rectangle, rounded corners, minimum height=1cm},
    blockd/.style={draw, thick, 
    text width=0.33\textwidth,
    align=center, 
    rectangle, rounded corners, minimum height=1cm},
    blocke/.style={draw, thick, 
    text width=0.71\textwidth,
    rectangle, rounded corners, minimum height=1cm},
    blockf/.style={draw, thick, 
    text width=0.33\textwidth,
    align=left, 
    rectangle, rounded corners, minimum height=1cm},
    line/.style={draw, thick, dashed},
    scalable/.style={text=blue},
    nonscalable/.style={text=red},
    node distance=0.5cm,
    label/.style={above=of groupA, align=center}
]

\node[blocka] (groupA) {
    \qquad\qquad\textcolor{blue}{REINFORCE} \qquad\qquad\qquad\qquad \ \ \textcolor{blue}{\uline{Cost-SOCM (+STL)}} \qquad\qquad\qquad\qquad \textcolor{blue}{Discrete Adjoint} \\ \quad \textcolor{blue}{REINFORCE (future rewards)} \qquad\qquad\qquad\qquad\qquad\qquad\qquad\qquad\qquad\quad 
    \textcolor{blue}{Continuous Adjoint (+STL)}
};
\node[blocke, below=0.2cm of groupA.south east, anchor=north east] (groupA2) {
    \qquad\qquad\quad  \textcolor{blue}{\uline{Work-SOCM}} \qquad\qquad\qquad\quad \ \textcolor{blue}{Adjoint Matching (+STL)} 
};
\node[blockb, below=0.2cm of groupA2.south east, anchor=north east] (groupB) {
    \qquad\qquad\textcolor{red}{Cross Entropy} \quad\qquad\qquad\qquad\qquad\qquad \textcolor{red}{SOCM} \qquad\qquad\qquad\qquad\qquad\qquad \textcolor{red}{SOCM-Adjoint}
};
\node[blockc, below=0.2cm of groupB.south west, anchor=north west] (groupC) {
    \ \textcolor{blue}{Log-variance} \\ \quad \  \textcolor{blue}{Moment}
};
\node[blockc, right=0.2cm of groupC] (groupE) {
    \quad \ \textcolor{red}{Variance}
};
\node[blockf, below=0.2cm of groupB.south east, anchor=north east] (groupF) {
    \qquad\quad \textcolor{blue}{\uline{Unweighted SOCM}}
};

\node[label] (labelA) at ([yshift=-0.2cm,xshift=-0.37\textwidth]groupA.north) {\footnotesize \textbf{No differentiability requirements}};
\node[label] (labelB) at ([yshift=-0.3cm]groupA.north) {\footnotesize \textbf{Base drift, state cost differentiable,} \\ \footnotesize \textbf{terminal cost optionally diff.}};
\node[label] (labelC) at ([yshift=-0.3cm,xshift=0.37\textwidth]groupA.north) {\footnotesize \textbf{Base drift, state \& terminal cost} \\  \small \textbf{differentiable}};

\draw[line] ([yshift=0.7cm,xshift=2.8cm]labelA.south) -- ([yshift=-5.1cm,xshift=2.8cm]labelA.south);
\draw[line] ([yshift=0.8cm,xshift=-2.7cm]labelC.south) -- ([yshift=-5.0cm,xshift=-2.7cm]labelC.south);

\end{tikzpicture}

\caption{Training losses for stochastic optimal control problems. Losses in blue scale to high-dimensions, while losses in red do not, as the gradient variance blows up exponentially with the dimension. By \autoref{thm:main}, losses in the same block (there are five different blocks) are equal in expectation, i.e. taking infinite batch size would yield the same gradient update. Novel losses are underlined, and losses that admit a Sticking The Landing version are identified with the suffix (+STL).}
\label{fig:SOC_taxonomy}
\end{figure}

Note that theoretical guarantees proven for a single loss function apply to all the loss functions in the class, because they all share the same gradient in expectation. That allows us to make unified statements about the convergence properties for each class, relying on the theoretical guarantees that we review and claim in \autoref{sec:more_losses}.

Loss functions in Class I are explicitly or implicitly optimizing the control objective, or equivalently the KL divergence between $\mathbb{P}^{u}$ and $\mathbb{P}^{u^*}$, through gradient descent. Even though the control objective is not convex in function space, the only critical point of all these losses is the optimal control. Similarly, loss functions in Class III are optimizing the KL divergence between $\mathbb{P}^{u^*}$ and $\mathbb{P}^{u}$. This is a strongly convex functional with respect to the control $u$, and the only critical point of all these losses is the optimal control. Loss functions in Class II have a single critical point which is the optimal control, but they cannot be regarded as performing gradient-based optimization on any objective functional.  

The information in \autoref{thm:main} is summarized in \autoref{fig:SOC_taxonomy}. Additionally, the diagram contains information on the losses that require differentiating through the state and terminal costs, those that do not, and those that differentiate the state cost and may or may not differentiate the terminal costs. The latter are SOCM based losses in which the reparameterization matrices $M$ can be chosen such that $M_t(T) = 0$ for all $t \in [0,T]$, which means that we do not need to evaluate the gradient of $g$.

\section{Simple stochastic optimal control experiments} \label{sec:simple_exp}

We benchmark all the loss functions in our taxonomy in \autoref{fig:SOC_taxonomy} on four experimental settings where we have access to the ground truth optimal control, which means that we can compute the control $L^2$ error incurred by each algorithm throughout training. The settings are \textsc{Quadratic Ornstein Uhlenbeck, easy} and \textsc{hard}, and \textsc{Double Well, easy} and \textsc{hard}, and \textsc{Linear}, and were used in \cite{domingoenrich2023stochastic,nüsken2023solving}. The code can be found at \url{https://github.com/facebookresearch/SOC-matching/tree/deep-Q-learning}.

\autoref{fig:control_L2_errors} contains all the plots; errors are averaged using an exponential moving average. The top left subfigure corresponds to an easy setting where all algorithms perform well. The best-performing losses are SOCM and UW-SOCM, which we introduce in \autoref{subsec:new_losses_app}. The top right subfigure is on a more challenging setting in which the running cost $f$ and the terminal cost $g$ are larger, where losses that do not scale well (those in red in \autoref{fig:SOC_taxonomy}) struggle. Larger costs make these losses fail completely, as gradient variance is too high. UW-SOCM achieves the lowest error. 

The middle subfigures are for settings where the optimal process has 1024 modes. Algorithms need to explore all the modes, which is why the control $L^2$ errors decrease slower than in the top subfigures, and also the reason why some plots have occasional bounces. In the middle left subfigure, Adjoint Matching achieves the lowest error, and UW-SOCM is the best loss among those for which the error decreases monotonically. In the middle right subfigure, UW-SOCM shows unstable behavior which we attribute to its weak theoretical guarantees (\autoref{prop:UW_SOCM_optimality}), and the best loss is SOCM. In the bottom figure, SOCM, UW-SOCM, and SOCM-Adjoint perform similarly.

In summary, while SOCM and UW-SOCM are the best loss functions in some settings, they struggle in others. The failure modes of SOCM are high dimensions and high-magnitude cost functions, both of which cause the importance weight $\alpha$ to have high variance, and one failure mode of UW-SOCM is multimodal problems, as the control fails to converge. In comparison, the performance of Continuous Adjoint and Adjoint Matching is not as good in simple settings such as \textsc{Quadratic OU, easy} and \textsc{Linear}, but they behave well across the board. Overall, SOCM-Cost and SOCM-Work have a similar performance, but at a higher computational cost, and the Log-Variance, Variance, Moment and Discrete Adjoint losses perform noticeably worse. Among the losses that converge, the worst ones across the board are the two REINFORCE losses. Thus, our experiments show that the Continuous Adjoint loss and the REINFORCE loss, which have the same gradient according to our taxonomy, have completely different behaviors, and this is due to different gradient variances.
In \autoref{subsec:additional_exp} we show additional plots and include more detailed comparisons among different loss functions, and in \autoref{subsec:detail_exp} we include more information about the experiments.

\begin{figure} 
    \centering
    \makebox[\textwidth][c]{%
        \includegraphics[width=0.56\textwidth]{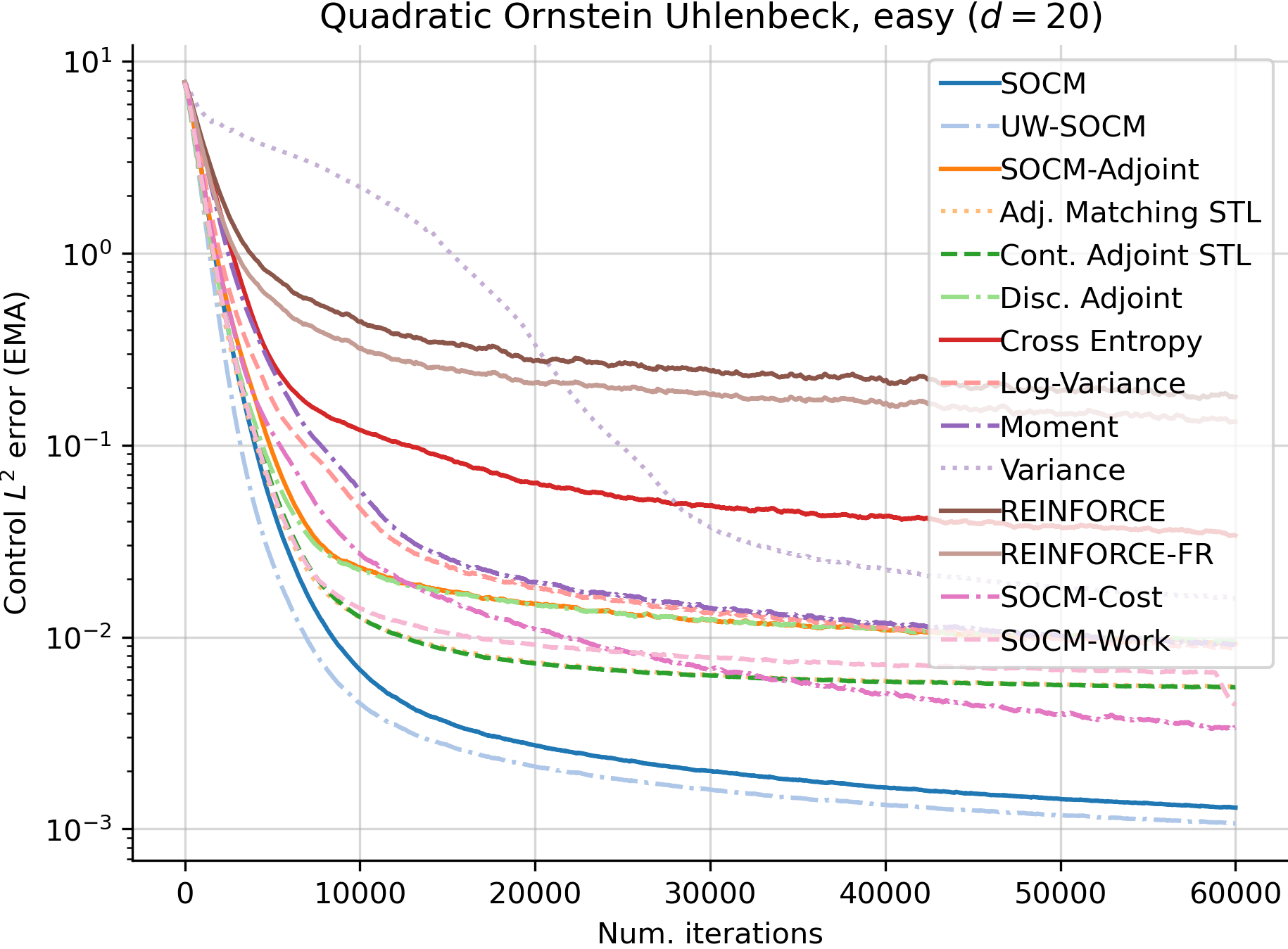}%
        \includegraphics[width=0.56\textwidth]{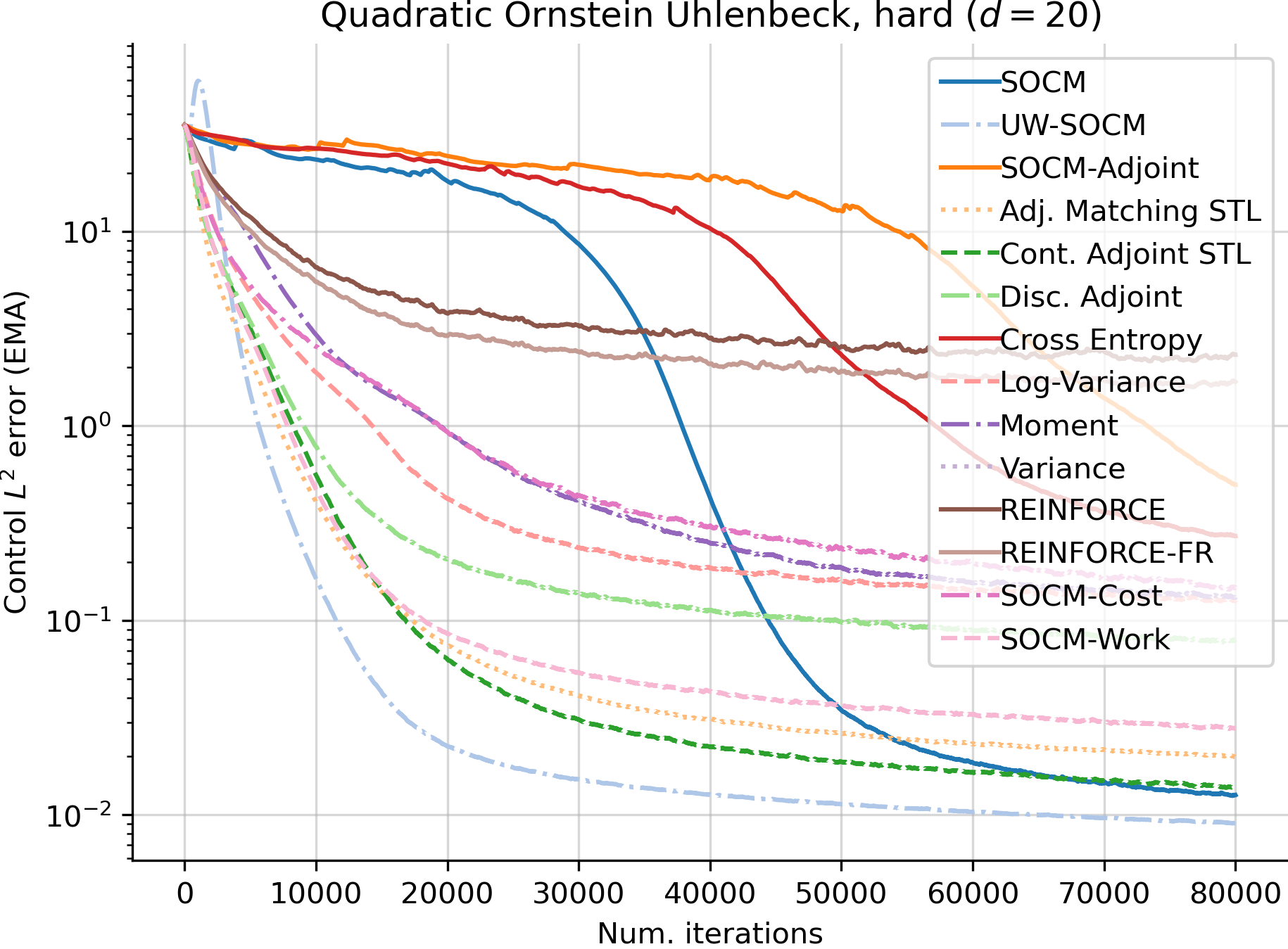}%
    }\\
    \makebox[\textwidth][c]{%
        \includegraphics[width=0.56\textwidth]{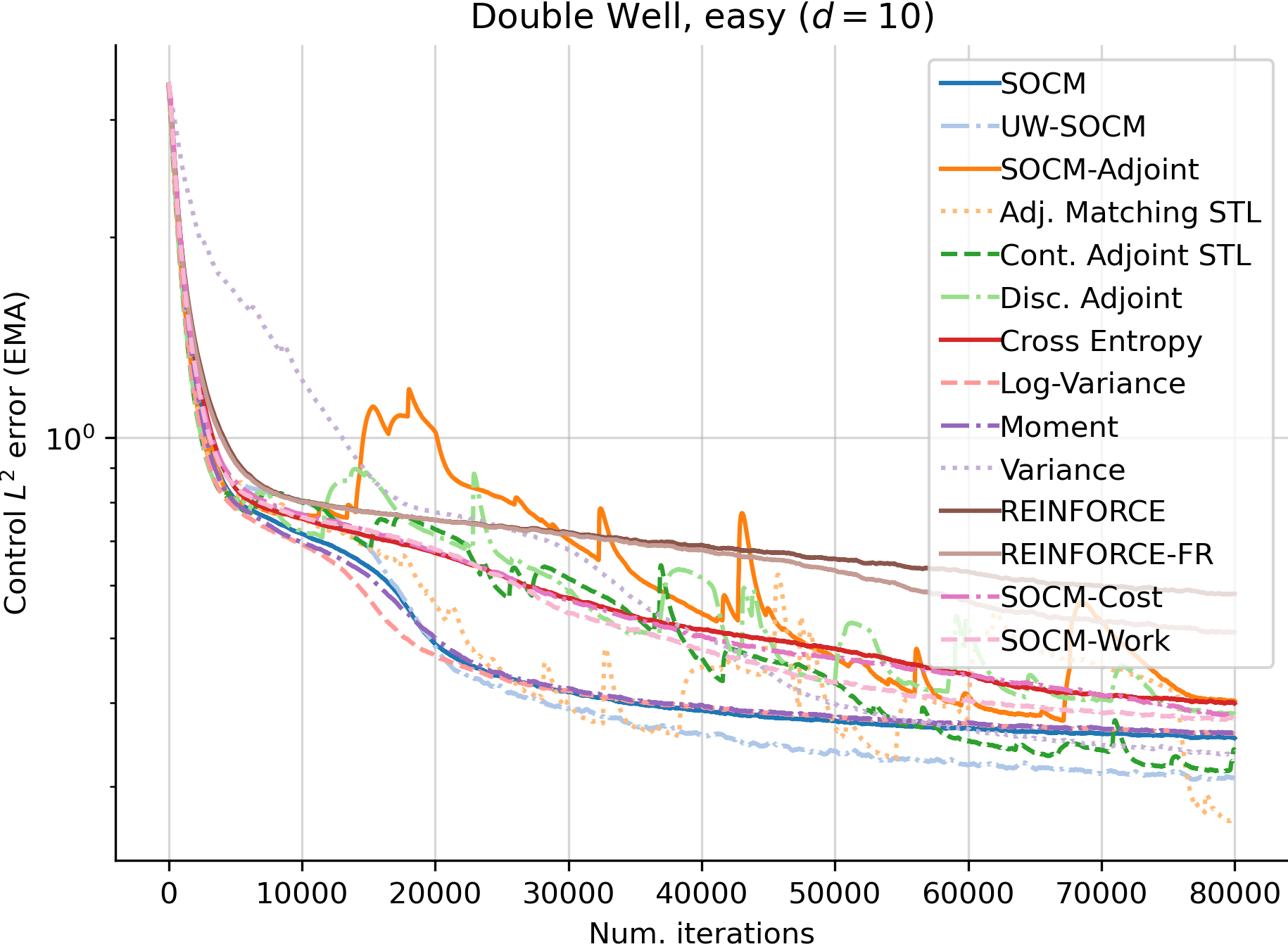}%
        \includegraphics[width=0.56\textwidth]{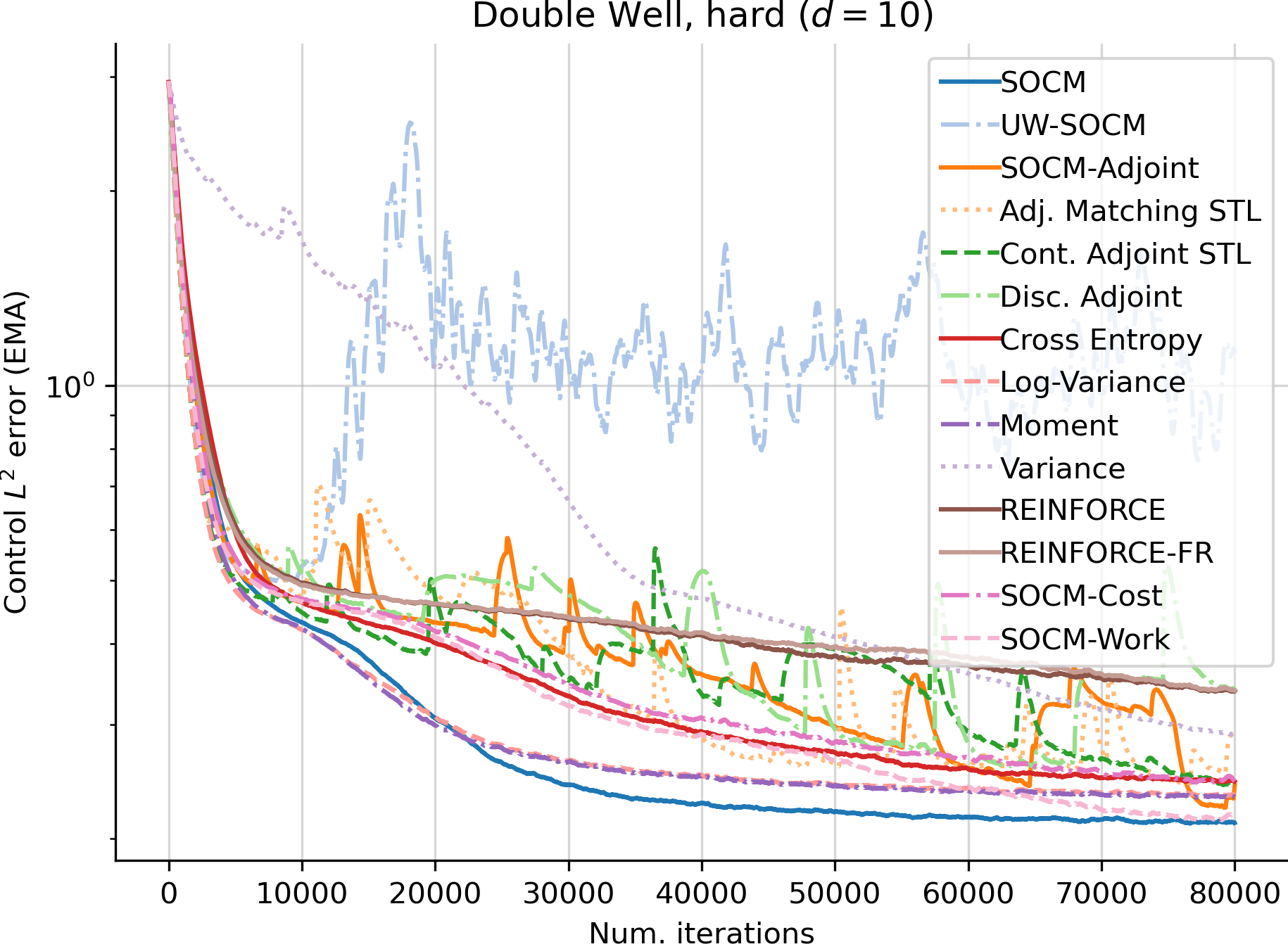}%
    }\\
    \makebox[\textwidth][c]{%
        \includegraphics[width=0.56\textwidth]{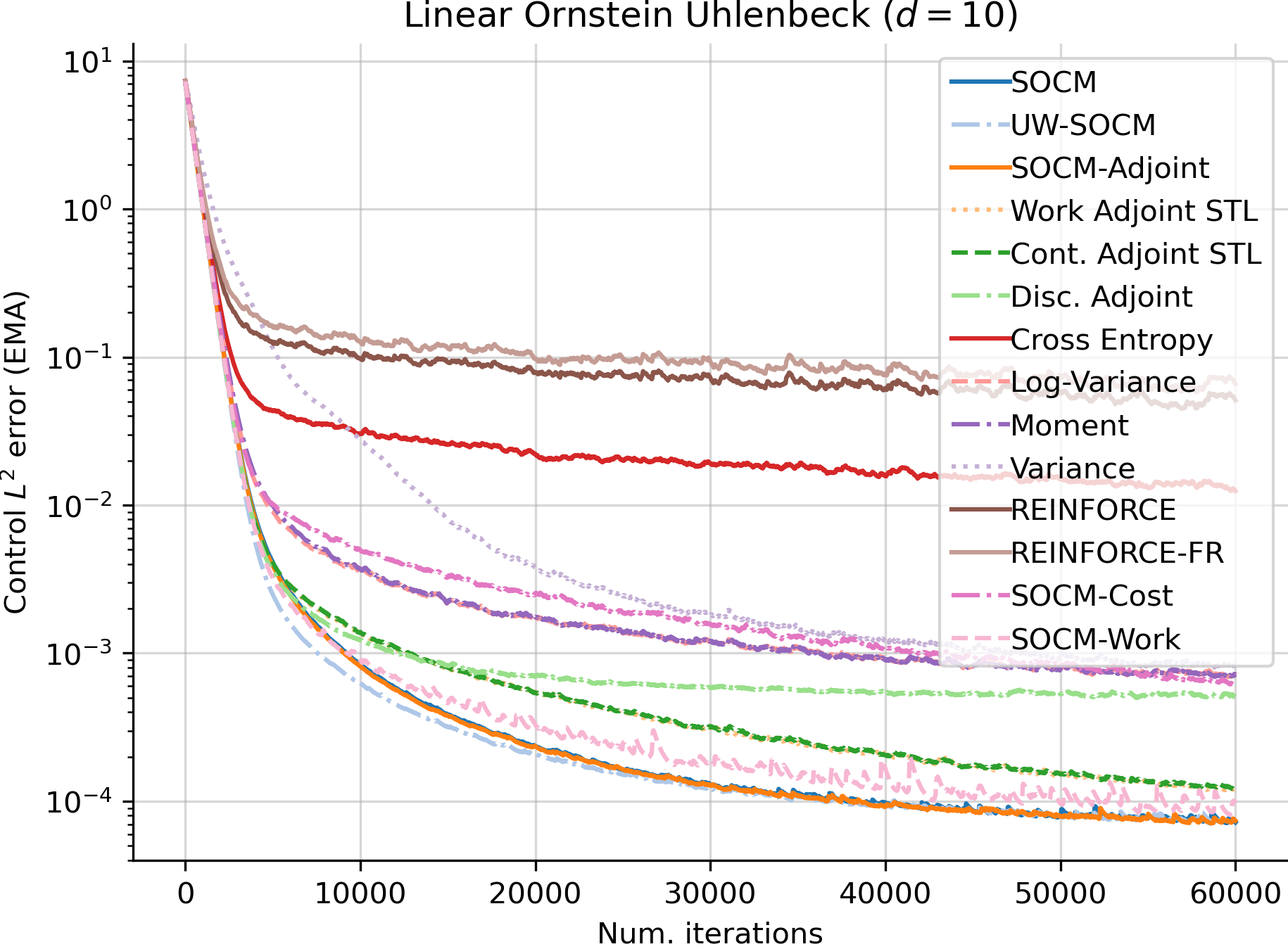}%
    }
    \caption{Control $L^2$ error incurred by each loss function throughout training, on five different settings.}
    \label{fig:control_L2_errors}
\end{figure}

\section{Conclusion}\label{sec:conc}
In this paper, we clarify the connections between existing and new deep learning loss functions to solve SOC problems, which have been recently applied to fine-tune diffusion and flow matching models. In particular, we observe that loss functions can be clustered into classes that share the same gradient in expectation. Qualitatively, all the losses in each class have the same convergence behavior, which we study. Quantitatively, losses within a class have gradient variances of different magnitudes, which translates to different convergence speeds. We compare all the losses on five different synthetic SOC problems.  

\bibliography{biblio}

\clearpage
\newpage
\appendix

\input{appendix}

\end{document}

%% file: preamble.tex
\usepackage{amsmath}
\usepackage{amssymb}
\usepackage{graphicx}
\usepackage{mathtools}
\usepackage{amsfonts}
\usepackage{amsthm,bm}
\usepackage{color}
\usepackage{enumitem}
\usepackage{dsfont}
\usepackage{pgfplots}
\pgfplotsset{width=10cm,compat=1.9}
 \usepackage{pifont}
\usepackage{thmtools} 
\usepackage{thm-restate}
\usepackage{sidecap}

\newcommand{\R}{\mathbb{R}}

\newcommand{\norm}[1]{\lVert#1\rVert}

\renewcommand{\phi}{\varphi}

\graphicspath {{figures/}}

\usepackage{hyperref}
\usepackage{caption}
\usepackage[super]{nth}
\newtheorem{lemma}{Lemma}

\newtheorem{theorem}{Theorem}

\newtheorem{assumption}{Assumption}
\newtheorem{proposition}{Proposition}

\newtheorem{remark}{Remark}
\newtheorem{corollary}{Corollary}

\usepackage{subcaption}

\mathtoolsset{showonlyrefs}

\newenvironment{talign*}
 {\csname align*\endcsname}
 {\endalign}
\newenvironment{talign}
 {\csname align\endcsname}
 {\endalign}

%% file: appendix.tex
\tableofcontents

\starttocentries

\section{Preliminaries}
\begin{theorem}[Girsanov theorem] \label{thm:Girsanov}
    Let $\bm{W} = {(W_t)}_{t \in [0,T]}$ be a standard Wiener process, and let $\mathbb{P}$ be its induced probability measure over $C([0,T];\mathbb{R}^d)$, known as the Wiener measure. 
    Let $(\Omega,\mathcal{F})$ be the $\sigma$-algebra associated to $B_T$. Let ${\theta_s}$ be a locally-$\mathcal{H}_2$ process which is adapted to the natural filtration of the Brownian motion $(B_t)_{t \geq 0}$.
    For any $F \in \mathcal{F}$, define the measure
    \begin{talign}
        \mathbb{Q}(F) = \mathbb{E}_{\mathbb{P}}[ 
        \exp \big( \int_0^T \theta_t \, \mathrm{d}B_t - \frac{1}{2} \int_0^T \|\theta_t\|^2 \, \mathrm{d}t \big)
        \mathbf{1}_F]~.
    \end{talign}
    $\mathbb{Q}$ is a probability measure. 
    Under 
    $\mathbb{Q}$, the stochastic process $\{\tilde{W}(t)\}_{0 \leq t \leq T}$ defined as
    \begin{talign}
        \tilde{W}(t) = W(t) - \int_0^t \theta_s \, \mathrm{d}s
    \end{talign}
    is a standard Wiener process.
\end{theorem}

\begin{corollary}[Girsanov theorem for SDEs] 
\label{cor:girsanov_sdes}
    If the two SDEs
    \begin{talign}
    \mathrm{d}X_{t} &= b_1 (X_{t},t) \, \mathrm{d}t + \sigma (X_{t},t) \, \mathrm{d}B_{t}, \qquad X_0 = x_{\mathrm{init}}  \\
    dY_{t} &= (b_1 (Y_{t},t) + b_2 (Y_{t},t)) \, \mathrm{d}t + \sigma (Y_{t},t) \, \mathrm{d}B_{t}, \qquad Y_0 = x_{\mathrm{init}}
    \end{talign}
    admit unique strong solutions on $[0,T]$, then for any bounded continuous functional $\Phi$ on $C([0,T])$, we have that
    \begin{talign} 
    \begin{split} \label{eq:X_to_Y}
        \mathbb{E}[\Phi(\bm{X})] &= \mathbb{E}\big[ \Phi(\bm{Y}) \exp \big( - \int_0^T \sigma(Y_{t},t)^{-1} b_2 (Y_{t},t) \, \mathrm{d}B_t - \frac{1}{2} \int_0^T \|\sigma(Y_{t},t)^{-1} b_2 (Y_{t},t)\|^2 \, \mathrm{d}t \big) \big] \\ &= \mathbb{E}\big[ \Phi(\bm{Y}) \exp \big( - \int_0^T \sigma(Y_{t},t)^{-1} b_2 (Y_{t},t) \, d\tilde{B}_t + \frac{1}{2} \int_0^T \|\sigma(Y_{t},t)^{-1} b_2 (Y_{t},t)\|^2 \, \mathrm{d}t \big) \big], 
    \end{split}
    \end{talign}
    where $\tilde{B}_t = B_t + \int_0^t \sigma(Y_{s},s)^{-1} b_2 (Y_{s},s) \, \mathrm{d}s$. More generally, $b_1$ and $b_2$ can be random processes that are adapted to filtration of $\bm{B}$.
\end{corollary}

\begin{theorem}[Hamilton-Jacobi-Bellman equation] \label{thm:HJB}
    If we define the infinitesimal generator
    \begin{talign}
    \mathcal{L} := \frac{1}{2} \sum_{i,j=1}^{d} (\sigma \sigma^{\top})_{ij} (t) \partial_{x_i} \partial_{x_j} + \sum_{i=1}^{d} b_i(x,t) \partial_{x_i},
    \end{talign}
    the value function $V$ for the SOC
    problem \eqref{eq:control_problem_def}-\eqref{eq:controlled_SDE} solves the following Hamilton-Jacobi-Bellman (HJB) partial differential equation:
    \begin{talign}
    \begin{split} \label{eq:HJB_setup}
        &\partial_t V(x,t) = - \mathcal{L} V(x,t) + \frac{1}{2} \| (\sigma^{\top} \nabla V) (x,t) \|^2 - f(x,t), \\
        &V(x,T) = g(x).
    \end{split}
    \end{talign}
\end{theorem}

\begin{theorem}[Path-wise reparameterization trick, \cite{domingoenrich2023stochastic}, Prop.~C.3] \label{lem:cond_exp_rewritten}
Let $(\Omega, \mathcal{F}, \mathbb{P})$ be a probability space, and $\bm{B} : \Omega \times [0,T] \to \R^d$ be a Brownian motion. Let $X : \Omega \times [0,T] \to \R^d$ be a process that satisfies the SDE $\mathrm{d}X_{t} = b (X_{t},t) \, \mathrm{d}t + \sigma (t) \, \mathrm{d}B_{t}$, and let $\psi : \Omega \times \R^d \times [0,T] \to \R^d$ be an arbitrary random process such that:
\begin{itemize}
    \item For all $z \in \R^d$, the process $\psi(\cdot,z,\cdot) : \Omega \times [0,T] \to \R^d$ is adapted to the filtration $({\mathcal{F}}_{s})_{s \in [0,T]}$ of the Brownian motion $\bm{B}$.
    \item For all $\omega \in \Omega$, $\psi(\omega,\cdot,\cdot) : \R^d \times [0,T] \to \R^{d}$ is a twice-continuously differentiable function such that $\psi(\omega,z,0) = z$ for all $z \in \R^d$, and $\psi(\omega,0,s) = 0$ for all $s \in [0,T]$. 
\end{itemize}
Let $F : C([0,T];\R^d) \to \R$ be a Fréchet-differentiable functional. We use the notation $\bm{X} + \psi(z,\cdot) = (X_s(\omega) + \psi(\omega,z,s))_{s \in [0,T]}$ to denote the shifted process, but omit the dependency of $\psi$ on $\omega$. Then,
    \begin{talign} 
    \begin{split} \label{eq:pathwise_rt}
        &\nabla_{x} \mathbb{E}\big[ \exp \big( 
        - F(\bm{X})
        \big) \big| X_0 = x \big] \\ &= \! \mathbb{E}\big[ \big( 
        \! - \! \nabla_z F(\bm{X} \! + \! \psi(z,\cdot)) \rvert_{z=0}
        \! + \! \int_0^T (\nabla_z \psi(0,s) \nabla_x b(X_s,s) \! - \! \nabla_z \partial_s \psi(0,s)) (\sigma^{-1})^{\top} (s) \mathrm{d}B_s \big) 
        \\ &\qquad\qquad \times 
        \exp \big( - F(\bm{X}) \big) 
    \big| X_0 = x \big].
    \end{split}
    \end{talign}
\end{theorem}

\begin{corollary}[Path-wise reparameterization trick for stochastic optimal control, \cite{domingoenrich2023stochastic}, Prop.~1]\label{prop:cond_exp_rewritten}
    For each $t \in [0,T]$, let $M_t : [t,T] \to \R^{d\times d}$ be an arbitrary continuously differentiable function matrix-valued function such that $M_t(t) = \mathrm{Id}$. We have that
    \begin{talign} 
    \begin{split} \label{eq:cond_exp_rewritten_SOC}
        &\nabla_{x} \mathbb{E}\big[ \exp \big( - \int_t^T f(X_s,s) \, \mathrm{d}s - g(X_T) \big) \big| X_t = x \big] \\ &= \mathbb{E}\big[ \big( - \int_t^T M_t(s) \nabla_x f(X_s,s) \, \mathrm{d}s - M_t(T) \nabla g(X_T) \\ &\qquad\qquad + \int_t^T (M_t(s) \nabla_x b(X_s,s) - \partial_s M_t(s)) (\sigma^{-1})^{\top}(s) \mathrm{d}B_s \big) \\ &\qquad\qquad \times \exp \big( - \int_t^T f(X_s,s) \, \mathrm{d}s - g(X_T) \big) 
    \big| X_t = x \big].
    \end{split}
    \end{talign}
\end{corollary}

\begin{theorem} [Adjoint method for SDEs, Lemma~8 of \cite{domingoenrich2023stochastic}, \cite{li2020scalable,kidger2021neural}]
\label{lem:adjoint_method_sdes}
    Let $\bm{X} : \Omega \times [0,T] \to \R^d$ be a stochastic process that satisfies the SDE $\mathrm{d}X_t = b(X_t,t) \, \mathrm{d}t + \sigma(t) \, \mathrm{d}B_t$, with initial condition $X_0 = x$.
    We define the random process $a : \Omega \times [0,T] \to \R^{d}$ such that for all $\omega \in \Omega$, using the short-hand $a(t) := a(\omega,t)$,
    \begin{talign}
        da_t(\omega) &= \big( - \nabla_{x} b(X_t(\omega),t) a_t(\omega) - \nabla_x f(X_t(\omega),t) \big) \, \mathrm{d}t  - \nabla_x h(X_t(\omega),t) \, \mathrm{d}B_t, \\ a_T(\omega) &= \nabla_x g(X_T(\omega)),
    \end{talign}
    we have that 
    \begin{talign}
        &\nabla_{x} \mathbb{E} \big[ \int_0^T f(X_t(\omega),t) \, \mathrm{d}t + \int_0^T \langle h(X_t(\omega),t), \, \mathrm{d}B_t \rangle + g(X_T(\omega)) | X_0(\omega) = x \big] = \mathbb{E} \big[ a_0(\omega) \big], \\
    \begin{split}
        &\nabla_{x} \mathbb{E} \big[ \exp \big( - \int_0^T f(X_t(\omega),t) \, \mathrm{d}t - \int_0^T \langle h(X_t(\omega),t), \, \mathrm{d}B_t \rangle - g(X_T(\omega)) \big) | X_0(\omega) = x \big] \\ &= - \mathbb{E} \big[ a_0(\omega) \exp \big( - \int_0^T f(X_t(\omega),t) \, \mathrm{d}t - \int_0^T \langle h(X_t(\omega),t), \, \mathrm{d}B_t \rangle - g(X_T(\omega)) \big) | X_0(\omega) = x \big].
    \end{split}
    \end{talign}
\end{theorem}


\section{Sticking the Landing trick} 
\label{subsec:STL}
\begin{remark}[Sticking the Landing trick] \label{rem:STL}
    For any $v \in \mathcal{U}$, we have the following:
    \begin{talign}
        &J(v;x,t) := \mathbb{E}\big[ \int_t^T \big(\frac{1}{2} \|v(X^{v}_s,s)\|^2 \! + \! f(X^{v}_s,s) \big) \, \mathrm{d}s \! + \! g(X^{v}_T) | X^v_t = x \big] \\ &= \mathbb{E}\big[ \int_t^T \big(\frac{1}{2} \|v(X^{v}_s,s)\|^2 \! + \! f(X^{v}_s,s) \big) \, \mathrm{d}s \! + \! 
        \int_t^T \langle v(X^{v}_s,s), \mathrm{d}B_s \rangle \! + \! g(X^{v}_T) | X^v_t = x \big]
    \end{talign}
    And for the optimal control $u^*$, we have that 
    \begin{talign} \label{eq:STL_main}
        \int_t^T \big(\frac{1}{2} \|u^*(X^{u^*}_s,s)\|^2 \! + \! f(X^{u^*}_s,s) \big) \, \mathrm{d}s \! + \!
        \int_t^T \langle u^*(X^{u^*}_s,s), \mathrm{d}B_s \rangle \! + \! g(X^{u^*}_T) = J(u^*;x,t) = V(X^{u^*}_t,t).
    \end{talign}
    Hence, the left-hand side of \eqref{eq:STL_main} is a zero-variance estimator of $J(u^*;x,t)$. For $v$ close to $u^*$, one can argue that  
    \begin{talign}
        \int_t^T \big(\frac{1}{2} \|v(X^{v}_s,s)\|^2 \! + \! f(X^{v}_s,s) \big) \, \mathrm{d}s \! + \! 
        \int_t^T \langle v(X^{v}_s,s), \mathrm{d}B_s \rangle \! + \! g(X^{v}_T)
    \end{talign}
    is a low-variance estimator of $J(v;x,t)$, and hence, it is convenient to use this expression to replace $\int_t^T \big(\frac{1}{2} \|v(X^{v}_s,s)\|^2 \! + \! f(X^{v}_s,s) \big) \, \mathrm{d}s \! + \! g(X^{v}_T)$, whenever needed. This amounts to including an additional term $
    \int_t^T \langle v(X^{v}_s,s), \mathrm{d}B_s \rangle$.
\end{remark}
\begin{proof}
    To prove \eqref{eq:STL_main}, we can assume without loss of generality that $t = 0$. We use that (see e.g. \cite[Lemma~4]{domingoenrich2023stochastic})
    \begin{talign}
        \forall v \in \mathcal{U}, \qquad \frac{d\mathbb{P}^v}{d\mathbb{P}}(X) &= \exp \big( 
        \int_0^T \langle v(X_t,t), \mathrm{d}B_t \rangle - \frac{
        1}{2} \int_0^T \|v(X_t,t)\|^2 \, \mathrm{d}t \big), \\
        \frac{d\mathbb{P}}{d\mathbb{P}^{u^*}}(X^{u^*}) &= \exp \big( 
        \big( - V(X_0^{u^*},0) + \int_t^T f(X^{u^*}_s,s) \, \mathrm{d}s + g(X^{u^*}_T) \big) \big),
    \end{talign}
    which means that
    \begin{talign}
    \begin{split} \label{eq:1_derivative}
        1 &= \frac{d\mathbb{P}^{u^*}}{d\mathbb{P}}(X^{u^*}) \frac{d\mathbb{P}}{d\mathbb{P}^{u^*}}(X^{u^*}) \\ &= \exp \big( 
        \int_0^T \langle v(X^{u^*}_t,t), \mathrm{d}B^{u^*}_t \rangle - \frac{
        1}{2} \int_0^T \|v(X^{u^*}_t,t)\|^2 \, \mathrm{d}t \\ &\qquad\quad + 
        \big( - V(X_0^{u^*},0) + \int_t^T f(X^{u^*}_s,s) \, \mathrm{d}s + g(X^{u^*}_T) \big) \big)
        \\ &= \exp \big( 
        \int_0^T \langle v(X^{u^*}_t,t), \mathrm{d}B_t \rangle + \frac{
        1}{2} \int_0^T \|v(X^{u^*}_t,t)\|^2 \, \mathrm{d}t \\ &\qquad\quad + 
        \big( - V(X_0^{u^*},0) + \int_t^T f(X^{u^*}_s,s) \, \mathrm{d}s + g(X^{u^*}_T) \big) \big)
    \end{split}
    \end{talign}
    where we used that $B^{u^*}_t = B_t + 
    \int_0^t v(X^{u^*}_t,t) \, dt$ or equivalently, $dB^{u^*}_t = dB_t + 
    v(X^{u^*}_t,t) \, dt$. Equation \eqref{eq:STL_main} can be deduced directly from \eqref{eq:1_derivative}.
\end{proof}

\subsection{Sticking the Landing trick for the Continuous Adjoint loss}
To derive the STL version of the Continuous Adjoint loss, we need to modify slightly the argument from \citealt[Lem.~5]{domingoenrich2024adjoint}. We need to replace $\int_t^T \big(\frac{1}{2} \|v(X^{u}_s,s)\|^2 \! + \! f(X^{u}_s,s) \big) \, \mathrm{d}s \! + \! g(X^{u}_T)$ by $\int_t^T \big(\frac{1}{2} \|v(X^{u}_s,s)\|^2 \! + \! f(X^{u}_s,s) \big) \, \mathrm{d}s \! + \! \int_t^T \langle v(X^{u}_s,s), \mathrm{d}B_s \rangle \! + \! g(X^{u}_T)$ everywhere in the proof. We end up with the following expression for the STL adjoint ODE:
\begin{talign} 
\begin{split} \label{eq:cont_adjoint_1_STL}
    da(t;X^{\bar{u}},\bar{u}) \! &= \! - \bigg[ \big(\nabla_x (b (X^{\bar{u}}_t,t) \! + \! \sigma(t) u(X^{\bar{u}}_t,t)) \big)^{\top} a(t;X^{\bar{u}},\bar{u}) \! + \! \nabla_x (f(X^{\bar{u}}_t,t) \! + \! \frac{1}{2}\|\bar{u}(X^{\bar{u}}_t,t)\|^2) \bigg] \, \mathrm{d}t 
    \\  &\qquad 
    \textcolor{blue}{- \, 
    \nabla_x \bar{u}(X^{\bar{u}}_t,t) \, \mathrm{d}B_t}, 
\end{split}
    \\ a(T,X^{\bar{u}}) &= \nabla g(X^v_T). \label{eq:cont_adjoint_2_STL}
\end{talign}
Since the STL estimator has zero variance at the optimal control, the STL adjoint state also has zero variance when $\bar{u}$ is the optimal control.

\subsection{Sticking the Landing trick for Adjoint Matching loss}
To derive the STL version of the Adjoint Matching loss, we apply the same argument as in \cite[Sec.~E.3]{domingoenrich2024adjoint}, but we start from the STL adjoint ODE instead. We obtain the following expression for the lean STL adjoint ODE:
\begin{talign}
\label{eq:work_adjoint_1_STL}
d\tilde{a}(t;X^{\bar{u}},\bar{u}) &= - (\nabla_x b (X^{\bar{u}}_t,t) \tilde{a}(t,X^{\bar{u}}) + \nabla_x f(X^{\bar{u}}_t,t)) \, \mathrm{d}t \ \textcolor{blue}{- \ 
\nabla_x \bar{u}(X^{\bar{u}}_t,t) \, \mathrm{d}B_t}, \\ \tilde{a}(T;X^{\bar{u}},\bar{u}) &= \nabla g(X^{\bar{u}}_T). \label{eq:work_adjoint_2_STL}    
\end{talign}

\section{Derivation of the REINFORCE loss}
\begin{proposition}[Derivation of the REINFORCE loss \eqref{eq:REINFORCE}] \label{prop:derivation_REINFORCE}
    The loss $\mathcal{L}_{\mathrm{RF}}$ in \eqref{eq:REINFORCE} is the continuous-time continuous-space analog of the REINFORCE loss for maximum entropy reinforcement learning.
\end{proposition}
\begin{proof}
A similar derivation can be found in \cite{borrell2022improving}.
By the Girsanov theorem for SDEs (\autoref{cor:girsanov_sdes}), we can rewrite the control objective \eqref{eq:control_problem_def} in terms of the uncontrolled process (that is the process controlled by $u \equiv 0$, which we denote simply by $X$): 
\begin{talign}
\begin{split} \label{eq:change_of_process}
    &\mathbb{E} \big[ \int_0^T 
    \big(\frac{1}{2} \|u(X^u_t,t)\|^2 \! + \! f(X^u_t,t) \big) \, \mathrm{d}t \! + \! 
    g(X^u_T) \big] \\ &= \mathbb{E} \big[ \big( \int_0^T 
    \big(\frac{1}{2} \|u(X_t,t)\|^2 \! + \! f(X_t,t) \big) \, \mathrm{d}t \! + \! 
    g(X_T) \big) \\ &\qquad \times \exp \big( 
    \int_0^T \langle u(X_t,t), \mathrm{d}B_t \rangle - \frac{
    1}{2} \int_0^T \|u(X_t,t)\|^2 \, \mathrm{d}t \big) \big].
\end{split}
\end{talign}
That is, we performed a change of process between the following pair:
\begin{talign}
\begin{split}
    \mathrm{d}X^u_t \! &= \! (b(X^u_t,t) \! + \! \sigma(t) u(X^u_t,t)) \, \mathrm{d}t \! + \! 
    \sigma(t) \mathrm{d}B_t, \qquad X^u_0 \sim p_0, \\
    \mathrm{d}X_t \! &= \! (b(X_t,t) \! + \! \sigma(t) u(X_t,t) \! - \! \sigma(t) u(X_t,t)) \, \mathrm{d}t \! + \! 
    \sigma(t) \mathrm{d}B_t, \qquad X^u_0 \sim p_0,
\end{split}
\end{talign}
which means that the tilted Brownian motion for the change of process is defined as $\tilde{B}_t = B_t - 
\int_0^t u(X_s,s) \, ds$.

The REINFORCE gradient for classical MaxEnt RL reads
\begin{talign}
    \nabla_{\theta} \mathbb{E}_{\tau \sim \pi_{\theta},p} [\sum_{k=0}^{K} r_k(s_k,a_k) - \mathrm{KL}(\pi_{\theta}(\cdot;s_k,k)||\pi_{\mathrm{base}}(\cdot;s_k,k))],
\end{talign}
where $\pi_{\theta}$ is a policy parameterized by $\theta$. Following the connection between SOC and maximum entropy RL \citep[App.~C]{domingoenrich2024adjoint}, 
the continuous-time continuous-space analog of the entropy-regularized expected reward is the SOC
objective. Thus, if we have a control $u_{\theta}$ parameterized by $\theta$, the REINFORCE control is given by 
\begin{talign}
\nabla_{\theta} \mathbb{E} \big[ \int_0^T 
\big(\frac{1}{2} \|u_{\theta}(X^{u_{\theta}}_t,t)\|^2 \! + \! f(X^{u_{\theta}}_t,t) \big) \, \mathrm{d}t \! + \! g(X^{u_{\theta}}_T) \big].
\end{talign}
We develop this making use of the change of process in \eqref{eq:change_of_process}, which makes the dependence on the control $u_{\theta}$ explicit:
\begin{talign}
\begin{split} \label{eq:girsanov_long}
    &\nabla_{\theta} \mathbb{E} \big[ \int_0^T 
    \big(\frac{1}{2} \|u_{\theta}(X^{u_{\theta}}_t,t)\|^2 \! + \! f(X^{u_{\theta}}_t,t) \big) \, \mathrm{d}t \! + \! 
    g(X^{u_{\theta}}_T) \big] 
    \\ &\stackrel{(i)}{=} \mathbb{E} \big[ \int_0^T \nabla_{\theta} u_{\theta}(X_t,t) u_{\theta}(X_t,t) \, \mathrm{d}t \times \exp \big( 
    \int_0^T \langle u_{\theta}(X_t,t), \mathrm{d}B_t \rangle - \frac{
    1}{2} \int_0^T \|u_{\theta}(X_t,t)\|^2 \, \mathrm{d}t \big) \\ &\qquad + \big( \int_0^T 
    \big(\frac{1}{2} \|u_{\theta}(X_t,t)\|^2 \! + \! f(X_t,t) \big) \, \mathrm{d}t \! + \! 
    g(X_T) \big) \\ &\qquad\quad \times \big( 
    \int_0^T \nabla_{\theta} u_{\theta}(X_t,t) \, \mathrm{d}B_t - \int_0^T \nabla_{\theta} u_{\theta}(X_t,t) u_{\theta}(X_t,t) \, \mathrm{d}t \big) \\ &\qquad\quad \times \exp \big( 
    \int_0^T \langle u_{\theta}(X_t,t), \mathrm{d}B_t \rangle - \frac{
    1}{2} \int_0^T \|u_{\theta}(X_t,t)\|^2 \, \mathrm{d}t \big) \big]
    \\ &\stackrel{(ii)}{=} \mathbb{E} \big[ \int_0^T \nabla_{\theta} u_{\theta}(X_t,t) u_{\theta}(X_t,t) \, \mathrm{d}t \times \exp \big( 
    \int_0^T \langle u_{\theta}(X_t,t), \mathrm{d}B_t \rangle - \frac{
    1}{2} \int_0^T \|u_{\theta}(X_t,t)\|^2 \, \mathrm{d}t \big) \\ &\qquad + \big( \int_0^T 
    \big(\frac{1}{2} \|u_{\theta}(X_t,t)\|^2 \! + \! f(X_t,t) \big) \, \mathrm{d}t \! + \! 
    g(X_T) \big) \\ &\qquad\quad \times \big( 
    \int_0^T \nabla_{\theta} u_{\theta}(X_t,t) \, (\underbrace{\mathrm{d}B_t  \ - \ 
    u_{\theta}(X_t,t) \, dt}_{\mathrm{d}\tilde{B}_t} 
   + \ 
    u_{\theta}(X_t,t) \, dt) \\ &\qquad\qquad\quad -
    \int_0^T \nabla_{\theta} u_{\theta}(X_t,t) u_{\theta}(X_t,t) \, \mathrm{d}t \big) \\ &\qquad\quad \times \exp \big( 
    \int_0^T \langle u_{\theta}(X_t,t), \mathrm{d}B_t \rangle - \frac{
    1}{2} \int_0^T \|u_{\theta}(X_t,t)\|^2 \, \mathrm{d}t \big) \big]
    \\ &\stackrel{(iii)}{=} \mathbb{E} \big[ \int_0^T \nabla_{\theta} u_{\theta}(X^{u_{\theta}}_t,t) u_{\theta}(X^{u_{\theta}}_t,t) \, \mathrm{d}t \\ &\qquad + \big( \int_0^T 
    \big(\frac{1}{2} \|u_{\theta}(X^{u_{\theta}}_t,t)\|^2 \! + \! f(X^{u_{\theta}}_t,t) \big) \, \mathrm{d}t \! + \! 
    g(X^{u_{\theta}}_T) \big) \\ &\qquad\quad \times \big( 
    \int_0^T \nabla_{\theta} u_{\theta}(X^{u_{\theta}}_t,t) \, (\mathrm{d}B_t + 
    u_{\theta}(X^{u_{\theta}}_t,t) \, dt) \\ &\qquad\qquad\quad - 
    \int_0^T \nabla_{\theta} u_{\theta}(X^{u_{\theta}}_t,t) u_{\theta}(X^{u_{\theta}}_t,t) \, \mathrm{d}t \big)
    \big]
    \\ &\stackrel{(iv)}{=} \mathbb{E} \big[ \int_0^T \nabla_{\theta} u_{\theta}(X^{u_{\theta}}_t,t) u_{\theta}(X^{u_{\theta}}_t,t) \, \mathrm{d}t \\ &\qquad + 
    \big( \int_0^T 
    \big(\frac{1}{2} \|u_{\theta}(X^{u_{\theta}}_t,t)\|^2 \! + \! f(X^{u_{\theta}}_t,t) \big) \, \mathrm{d}t \! + \! 
    g(X^{u_{\theta}}_T) \big) 
    \times 
    \int_0^T \nabla_{\theta} u_{\theta}(X^{u_{\theta}}_t,t) \, \mathrm{d}B_t 
    \big]
\end{split}
\end{talign}
In this derivation,
\begin{itemize}[leftmargin=0.5cm]
    \item Equality $(i)$ holds by taking the derivative of the right-hand side of \eqref{eq:change_of_process} with respect to $\theta$.
    \item Equality $(ii)$ holds by adding and subtracting $
    u_{\theta}(X_t,t) \, dt$ to $dB_t$, in order to express the stochastic integral in terms of the tilted Brownian motion $\tilde{B}_t$. 
    \item Equality $(iii)$ holds by applying the Girsanov theorem in the opposite direction, with a change of process from $X$ to $X^{u_{\theta}}$. In these case, we need to use the full version of the theorem (\autoref{thm:Girsanov}), which states that the tilted Brownian motion $\tilde{B}_t$ is a standard Brownian motion when the expectation contains the importance weight $\exp \big( 
    \int_0^T \langle u_{\theta}(X_t,t), \mathrm{d}B_t \rangle - \frac{
    1}{2} \int_0^T \|u_{\theta}(X_t,t)\|^2 \, \mathrm{d}t \big)$.
    \item Equality $(iv)$ is a straight-forward simplification.
\end{itemize}
To conclude the proof, observe that
\begin{talign}
\begin{split}
    &\mathbb{E} \big[ \int_0^T \nabla_{\theta} u_{\theta}(X^{u_{\theta}}_t,t) u_{\theta}(X^{u_{\theta}}_t,t) \, \mathrm{d}t \\ &\qquad + 
    \big( \int_0^T 
    \big(\frac{1}{2} \|u_{\theta}(X^{u_{\theta}}_t,t)\|^2 \! + \! f(X^{u_{\theta}}_t,t) \big) \, \mathrm{d}t \! + \! 
    g(X^{u_{\theta}}_T) \big) 
    \times 
    \int_0^T \nabla_{\theta} u_{\theta}(X^{u_{\theta}}_t,t) \, \mathrm{d}B_t 
    \big] \\ &= \nabla_{\theta} \mathbb{E}\big[ \frac{1}{2} \int_0^T \|u_{\theta}(X^{v}_t,t)\|^2 \, \mathrm{d}t \\ &\qquad + 
    \big( \int_0^T 
    \big(\frac{1}{2} \|v(X^{v}_t,t)\|^2 \! + \! f(X^{v}_t,t) \big) \, \mathrm{d}t \! + \! 
    g(X^{v}_T) \big) \times 
    \int_0^T u_{\theta}(X^{v}_t,t) \, \mathrm{d}B_t \big] \rvert_{v = \mathrm{stopgrad}(u_{\theta})} \\ &= \nabla_{\theta} \mathbb{E}[\mathcal{L}_{\mathrm{RF}}(u_{\theta})].
\end{split}
\end{talign}
\end{proof}

\section{Theoretical guarantees of new SOC loss functions}

    \subsection{Proof of \autoref{prop:work_adjoint_app}: theoretical guarantees of the Work-SOCM loss}
    To prove that the gradients of the losses $\mathcal{L}_{\mathrm{Adj-Match}}$ and $\mathcal{L}_{\mathrm{Work-SOCM}}$ are equal, we first state a characterization of the solution $\tilde{a}(t,X^v)$ of the Adjoint Matching ODE \eqref{eq:lean_adjoint_1}-\eqref{eq:lean_adjoint_2}. By \autoref{lem:adjoint_method_sdes}, we have that
    \begin{talign} \label{eq:a_alternative}
        \tilde{a}(t,X^v) = \nabla_{X_t} \big(\int_t^T f(X_s,s) \, \mathrm{d}s + g(X_T) \big) \rvert_{X = X^v}.
    \end{talign}
    Here, $X$ is the uncontrolled process, satisfying $\mathrm{d}X_t = b(X_t,t) \, \mathrm{d}t + \sigma(t) \, \mathrm{d}B_t$, $X_0=x$. Recall that $\mathcal{W}(t,X) = \int_t^T f(X_s,s) \, \mathrm{d}s + g(X_T)$ is known as the work function, which is why Adjoint Matching is an appropriate name for the resulting loss.

    From \eqref{eq:a_alternative}, we deduce that
    \begin{talign}
    \begin{split} \label{eq:exp_a_stopgrad}
        &\mathbb{E}[\tilde{a}(t,X^v)|X^v_t = x] = \mathbb{E}[\nabla_{X_t} \big(\int_t^T f(X_s,s) \, \mathrm{d}s + g(X_T) 
        \big) \rvert_{X = X^v}|X^v_t = x] \\ &= \mathbb{E}[\nabla_{X_t} \big(\int_t^T f(X_s,s) \, \mathrm{d}s + g(X_T) 
        \big) \frac{\mathrm{d}\mathbb{P}^v}{\mathrm{d}\mathbb{P}}(X) |X_t = x] \\ &= \nabla_{x} \mathbb{E}[ \big(\int_t^T f(X_s,s) \, \mathrm{d}s + g(X_T) 
        \big) \mathrm{stopgrad}(\frac{\mathrm{d}\mathbb{P}^v}{\mathrm{d}\mathbb{P}}(X)) |X_t = x]
    \end{split}
    \end{talign}
    where $\frac{\mathrm{d}\mathbb{P}^v}{\mathrm{d}\mathbb{P}}(X) = \exp\big( 
    \int_t^T \langle v(X_s,s), \mathrm{d}B_s \rangle - \frac{
    1}{2} \int_t^T \|v(X_s,s)\|^2 \, \mathrm{d}s \big)$ by Girsanov's theorem (\autoref{cor:girsanov_sdes}).

    Next, we apply the path-wise reparameterization trick (\autoref{lem:cond_exp_rewritten}), taking the uncontrolled process $X$ between times $t$ and $T$ and defining $F$ such that
    \begin{talign}
    \begin{split}
        \exp(- F(X)) &= \big(\int_t^T f(X_s,s) \, \mathrm{d}s + g(X_T) 
        \big) \mathrm{stopgrad}(\frac{\mathrm{d}\mathbb{P}^v}{\mathrm{d}\mathbb{P}}(X)) \\ \iff F(X) &= - \log \big( \big(\int_t^T f(X_s,s) \, \mathrm{d}s + g(X_T) 
        \big) \mathrm{stopgrad}(\frac{\mathrm{d}\mathbb{P}^v}{\mathrm{d}\mathbb{P}}(X)) \big)
    \end{split}
    \end{talign}
    This means that
    \begin{talign}
    \begin{split}
        &\nabla_z F(X + \psi(z,\cdot)) \rvert_{z=0} = - \frac{\partial_z \big(\int_t^T f(X_s,s) \, \mathrm{d}s + g(X_T) 
        \big) \rvert_{z=0} \mathrm{stopgrad}(\frac{\mathrm{d}\mathbb{P}^v}{\mathrm{d}\mathbb{P}}(X))}{\big(\int_t^T f(X_s,s) \, \mathrm{d}s + g(X_T) 
        \big) \mathrm{stopgrad}(\frac{\mathrm{d}\mathbb{P}^v}{\mathrm{d}\mathbb{P}}(X))} \\ &= - \frac{\int_t^T \nabla_z \psi(0,s) \nabla_x f(X_s,s) \, \mathrm{d}s + \nabla_z \psi(0,T) \nabla_x g(X_T) 
        }{\int_t^T f(X_s,s) \, \mathrm{d}s + g(X_T) 
        }
    \end{split}
    \end{talign}
    We conclude that
    \begin{talign} 
    \begin{split} \label{eq:pathwise_rt_work_adjoint}
        \nabla_{x} \mathbb{E}\big[ \exp \big( 
        - F(X)
        \big) \big| X_0 = x \big] &= \! \mathbb{E}\big[ \big( 
        \! - \! \nabla_z F(X \! + \! \psi(z,\cdot)) \rvert_{z=0}
        \\ &\qquad\quad + \! 
        \int_t^T (\nabla_z \psi(0,s) \nabla_x b(X_s,s) \! - \! \nabla_z \partial_s \psi(0,s)) (\sigma^{-1})^{\top} (s) \mathrm{d}B_s \big) 
        \\ &\qquad\qquad \times 
        \exp \big( - F(X) \big) 
    \big| X_t = x \big] \\ &= \mathbb{E}\big[ 
        \big( \int_t^T 
        \nabla_z \psi(0,s) \nabla_x f(X_s,s) \, \mathrm{d}s + 
        \nabla_z \psi(0,T) \nabla_x g(X_T) 
        \\ &\qquad\quad + \! 
        \int_t^T (\nabla_z \psi(0,s) \nabla_x b(X_s,s) \! - \! \nabla_z \partial_s \psi(0,s)) (\sigma^{-1})^{\top} (s) \mathrm{d}B_s 
        \\ &\qquad\qquad \times \big( \int_t^T f(X_s,s) \, \mathrm{d}s \! + \! 
        g(X_T) 
        \big) \big) \times \mathrm{stopgrad}(\frac{\mathrm{d}\mathbb{P}^v}{\mathrm{d}\mathbb{P}}(X))  
    \big| X_t = x \big]
    \\ &= \mathbb{E}\big[ 
        \int_t^T 
        \nabla_z \psi(0,s) \nabla_x f(X^v_s,s) \, \mathrm{d}s + 
        \nabla_z \psi(0,T) \nabla_x g(X^v_T) 
        \\ &\qquad\quad + \! 
        \int_t^T (\nabla_z \psi(0,s) \nabla_x b(X^v_s,s) \! - \! \nabla_z \partial_s \psi(0,s)) (\sigma^{-1})^{\top} (s) \mathrm{d}B_s 
        \\ &\qquad\qquad \times \big( \int_t^T f(X^v_s,s) \, \mathrm{d}s \! + \! 
        g(X^v_T) 
        \big) \big| X^v_t = x \big]
    \end{split}
    \end{talign}
    If we choose the perturbation $\psi(z,s) = M_t(s)^{\top} z$ with $M_t(t) = \mathrm{Id}$, we have that the conditions in \autoref{lem:cond_exp_rewritten} are satisfied, and we have that $\nabla \psi(z,s) = M_t(s)$. Using this, and putting together \eqref{eq:exp_a_stopgrad} and \eqref{eq:pathwise_rt_work_adjoint}, we obtain that
    \begin{talign}
    \begin{split} \label{eq:tilde_a_tilde_xi_proof}
        \mathbb{E}[\tilde{a}(t,X^v)|X^v_t = x] 
        \! &= \! 
        \mathbb{E}\big[ 
        \int_t^T 
        M_t(s) \nabla_x f(X^v_s,s) \, \mathrm{d}s + 
        M_t(T) \nabla_x g(X^v_T)
        \\ &\qquad\qquad\quad + \! 
        \int_t^T (M_t(s) \nabla_x b(X^v_s,s) \! - \! \nabla_z M_t(s)) (\sigma^{-1})^{\top} (s) \mathrm{d}B_s 
        \\ &\qquad\qquad\quad \times \big( \int_t^T f(X^v_s,s) \, \mathrm{d}s \! + \! 
        g(X^v_T) 
        \big) 
        \big| X^v_t = x \big] \\ &= \mathbb{E}\big[ \tilde{\xi}(t,v,X^v,B,M_t) \big| X^v_t = x \big].
    \end{split}
    \end{talign}
    Now, we can write a decomposition of the expected loss $\mathbb{E}[\mathcal{L}_{\mathrm{Work-SOCM}}]$ 
    by adding and subtracting the conditional expectation of $\sigma(t)^{\top} \tilde{\xi}(t,v,X^v,B,M_t)$, which allows us to use \eqref{eq:tilde_a_tilde_xi_proof}:
    \begin{talign}
    \begin{split} \label{eq:work_SOCM_rewritten}
        \mathbb{E}[\mathcal{L}_{\mathrm{Work-SOCM}}(u)] &:= \mathbb{E} \big[\int_0^{T} \big\| u(X^v_t,t)
        + \sigma(t)^{\top} \mathbb{E}\big[ \tilde{\xi}(t,v,X^v,B,M_t) \big| X^v_t \big] \big\|^2 \, \mathrm{d}t \big] \rvert_{v = \mathrm{stopgrad}(u)} \\ &\; + \mathbb{E} \big[\int_0^{T} \big\| \sigma(t)^{\top} \big( \mathbb{E}\big[ \tilde{\xi}(t,v,X^v,B,M_t) \big| X^v_t \big] - \tilde{\xi}(t,v,X^v,B,M_t) \big) \big\|^2 \, \mathrm{d}t \big] \rvert_{v = \mathrm{stopgrad}(u)} \\ &= \mathbb{E} \big[\int_0^{T} \big\| u(X^v_t,t)
        + \sigma(t)^{\top} \mathbb{E}[\tilde{a}(t,X^v)|X^v_t = x] \big\|^2 \, \mathrm{d}t \big] \rvert_{v = \mathrm{stopgrad}(u)} \\ &\; + \mathbb{E} \big[\int_0^{T} \big\| \sigma(t)^{\top} \big( \mathbb{E}\big[ \tilde{\xi}(t,v,X^v,B,M_t) \big| X^v_t \big] - \tilde{\xi}(t,v,X^v,B,M_t) \big) \big\|^2 \, \mathrm{d}t \big] \rvert_{v = \mathrm{stopgrad}(u)}
        \\ &= \mathbb{E}[\mathcal{L}_{\mathrm{Adj-M}}(u)] - \mathbb{E} \big[\int_0^{T} \big\| \sigma(t)^{\top} \big(\mathbb{E}\big[ \tilde{a}(t,X^v) | X^v_{t} \big] - \tilde{a}(t,X^v)  \big)\big\|^2 \, \mathrm{d}t \big] \rvert_{v = \mathrm{stopgrad}(u)} \\ &\; + \mathbb{E} \big[\int_0^{T} \big\| \sigma(t)^{\top} \big( \mathbb{E}\big[ \tilde{\xi}(t,v,X^v,B,M_t) \big| X^v_t \big] - \tilde{\xi}(t,v,X^v,B,M_t) \big) \big\|^2 \, \mathrm{d}t \big] \rvert_{v = \mathrm{stopgrad}(u)}.
    \end{split}
    \end{talign}

\subsection{Proof of \autoref{prop:SOCM_cost}: theoretical guarantees of the Cost-SOCM loss}
We rely on the path-wise reparameterization trick (\autoref{lem:cond_exp_rewritten}), taking the process to be the process $X^v$ controlled by $v$ between times $t$ and $T$, and 
\begin{talign}
\begin{split}
    \exp(- F(X^v)) &= \int_t^T 
    \big(\frac{1}{2} \|v(X^v_s,s)\|^2 \! + \! f(X^v_s,s) \big) \, \mathrm{d}s \! + \! 
    g(X^v_T) \, \textcolor{blue}{+ \, 
    \int_t^T \langle v(X^v_s,s), \mathrm{d}B_s \rangle} \\ \iff F(X^v) &= - \log \big( \int_t^T 
    \big(\frac{1}{2} \|v(X^v_s,s)\|^2 \! + \! f(X^v_s,s) \big) \, \mathrm{d}s \! + \! 
    g(X^v_T) \, \textcolor{blue}{+ \, 
    \int_t^T \langle v(X^v_s,s), \mathrm{d}B_s \rangle} \big)
\end{split}
\end{talign}
This means that
\begin{talign}
\begin{split}
    &\nabla_z F(X^v + \psi(z,\cdot)) \rvert_{z=0} = - \frac{\partial_z \big( \int_t^T 
    (\frac{1}{2} \|v(X^v_s + \psi(z,s),s)\|^2 + f(X^v_s + \psi(z,s),s) ) \, \mathrm{d}s + 
    g(X^v_T + \psi(z,T)) \, \textcolor{blue}{+ \,
    \int_t^T \langle v(X^v_s,s), \mathrm{d}B_s \rangle} \big) \rvert_{z=0}}{\int_0^T 
    (\frac{1}{2} \|v(X^v_s,s)\|^2 + f(X^v_s,s) ) \, \mathrm{d}t + 
    g(X^v_T) \, \textcolor{blue}{+ \, 
    \int_t^T \langle v(X^v_s,s), \mathrm{d}B_s \rangle}} \\ &= - \frac{ \int_t^T 
    \nabla_z \psi(0,s) \nabla_x (\frac{1}{2} \|v(X^v_s,s)\|^2 + f(X^v_s,s) ) \, \mathrm{d}s + 
    \nabla_z \psi(0,T) \nabla_x g(X^v_T) \, \textcolor{blue}{+ \, 
    \int_t^T \nabla_z \psi(0,s) \nabla_x v(X^v_s,s) \mathrm{d}B_s}}{\int_t^T 
    (\frac{1}{2} \|v(X^v_s,s)\|^2 + f(X^v_s,s) ) \, \mathrm{d}s + 
    g(X^v_T) \, \textcolor{blue}{+ \, 
    \int_t^T \langle v(X^v_s,s), \mathrm{d}B_s \rangle}}
\end{split}
\end{talign}
Let $J(v; x,t) = \mathbb{E}[\int_t^T 
    \big(\frac{1}{2} \|v(X^v_s,s)\|^2 + f(X^v_s,s) \big) \, \mathrm{d}s + 
    g(X^v_T) | X^v_t = x]$ be the cost functional. We conclude that
\begin{talign} 
\begin{split} \label{eq:pathwise_rt_SOCM_cost}
    &\nabla_x J(v;x,t) = \nabla_{x} \mathbb{E}\big[ \exp \big( 
    - F(X^v)
    \big) \big| X_0 = x \big] \\ &= \! \mathbb{E}\big[ \big( 
    \! - \! \nabla_z F(X^v \! + \! \psi(z,\cdot)) \rvert_{z=0}
    \\ &\qquad\quad + \! 
    \int_t^T (\nabla_z \psi(0,s) \nabla_x (b(X^v_s,s) + \sigma(s) v(X^v_s,s)) \! - \! \nabla_z \partial_s \psi(0,s)) (\sigma^{-1})^{\top} (s) \mathrm{d}B_s \big) 
    \\ &\qquad\qquad \times 
    \exp \big( - F(X^v) \big) 
\big| X^v_t = x \big] \\ &= \mathbb{E}\big[ 
    \int_t^T 
    \nabla_z \psi(0,s) \nabla_x (\frac{1}{2} \|v(X^v_s,s)\|^2 + f(X^v_s,s) ) \, \mathrm{d}s + 
    \nabla_z \psi(0,T) \nabla_x g(X^v_T) \, \textcolor{blue}{+ \, 
    \int_t^T \nabla_z \psi(0,s) \nabla_x v(X^v_s,s) \mathrm{d}B_s}
    \\ &\qquad\quad + \! 
    \int_t^T (\nabla_z \psi(0,s) \nabla_x (b(X^v_s,s) + \sigma(s) v(X^v_s,s)) \! - \! \nabla_z \partial_s \psi(0,s)) (\sigma^{-1})^{\top} (s) \mathrm{d}B_s 
    \\ &\qquad\qquad \times \big( \int_t^T 
    \big(\frac{1}{2} \|v(X^v_s,s)\|^2 \! + \! f(X^v_s,s) \big) \, \mathrm{d}s \! + \! 
    g(X^v_T) \, \textcolor{blue}{+ \, 
    \int_t^T \langle v(X^v_s,s), \mathrm{d}B_s \rangle} \big) 
\big| X^v_t = x \big]
\end{split}
\end{talign}
If we choose the perturbation $\psi(z,s) = M_t(s)^{\top} z$ with $M_t(t) = \mathrm{Id}$, we have that the conditions in \autoref{lem:cond_exp_rewritten} are satisfied, and we have that $\nabla \psi(z,s) = M_t(s)$. Thus,
\begin{talign}
\begin{split}
    \nabla_x J(v;x,t) \! &= \! 
    \mathbb{E}\big[ 
    \int_t^T 
    M_t(s) \nabla_x (\frac{1}{2} \|v(X^v_s,s)\|^2 + f(X^v_s,s) ) \, \mathrm{d}s + 
    M_t(T) \nabla_x g(X^v_T) \, \textcolor{blue}{+ \,
    \int_t^T \nabla_z \psi(0,s) \nabla_x v(X^v_s,s) \mathrm{d}B_s}
    \\ &\quad + \! 
    \int_t^T (M_t(s) \nabla_x (b(X^v_s,s) \! + \! \sigma(s) v(X^v_s,s)) \! - \! \nabla_z M_t(s)) (\sigma^{-1})^{\top} (s) \mathrm{d}B_s 
    \\ &\qquad\quad \times \big( \int_t^T 
    \big(\frac{1}{2} \|v(X^v_s,s)\|^2 \! + \! f(X^v_s,s) \big) \, \mathrm{d}s \! + \! 
    g(X^v_T) \, \textcolor{blue}{+ \, 
    \int_t^T \langle v(X^v_s,s), \mathrm{d}B_s \rangle} \big) 
    \big| X^v_t = x \big] \\ &= \mathbb{E}\big[ \xi(t,v,X^v,B,M_t) \big| X^v_t = x \big].
\end{split}
\end{talign}
Let $a(t;X^v,v)$ be the solution of the adjoint ODE \eqref{eq:cont_adjoint_1}-\eqref{eq:cont_adjoint_2}. 
\cite[Lem.~5]{domingoenrich2024adjoint} shows that $\mathbb{E}[a(t,X^v) | X^v_t = x] = \nabla_x J(v;x,t)$, which means that $\mathbb{E}\big[ \xi(t,v,X^v,B,M_t) \big| X^v_t = x \big] = \mathbb{E}[a(t,X^v) | X^v_t = x]$. Writing out the decomposition of the expected loss $\mathbb{E}[\mathcal{L}_{\mathrm{Cost-SOCM}}]$ analogous to \eqref{eq:work_SOCM_rewritten} concludes the proof:
\begin{talign}
\begin{split}
    \mathbb{E}[\mathcal{L}_{\mathrm{Cost-SOCM}}(u,M)] &= \mathbb{E} \big[\int_0^{T} \big\| u(X^v_t,t)
    + \sigma(t)^{\top} \mathbb{E}\big[ \xi(t,v,X^v,B,M_t) | X^v_{t} \big] \big\|^2 \, \mathrm{d}t \big] \rvert_{v = \mathrm{stopgrad}(u)} \\ &\quad + \mathbb{E} \big[\int_0^{T} \big\| \sigma(t)^{\top} \big(\mathbb{E}\big[ \xi(t,v,X^v,B,M_t) | X^v_{t} \big] - \xi(t,v,X^v,B,M_t)  \big)\big\|^2 \, \mathrm{d}t \big] \rvert_{v = \mathrm{stopgrad}(u)} \\ &= \mathbb{E}[\mathcal{L}_{\mathrm{Cont-Adj}}(u,M)] - \mathbb{E} \big[\int_0^{T} \big\| \sigma(t)^{\top} \big(\mathbb{E}\big[ a(t,X^v) | X^v_{t} \big] - a(t,X^v)  \big)\big\|^2 \, \mathrm{d}t \big] \rvert_{v = \mathrm{stopgrad}(u)} \\ &\quad + \mathbb{E} \big[\int_0^{T} \big\| \sigma(t)^{\top} \big(\mathbb{E}\big[ \xi(t,v,X^v,B,M_t) | X^v_{t} \big] - \xi(t,v,X^v,B,M_t)  \big)\big\|^2 \, \mathrm{d}t \big] \rvert_{v = \mathrm{stopgrad}(u)}.
\end{split}
\end{talign}

\subsection{Proof of \autoref{prop:UW_SOCM_optimality}: theoretical guarantees of the UW-SOCM loss} \label{subsec:UW_SOCM_optimality}
    Recall that for any $M$, the unique minimizer of $\mathcal{L}_{\mathrm{SOCM}}(\cdot,M)$ is the optimal control $u^*$, c.f. \cite[proof sketch of Theorem 3.1]{domingoenrich2023stochastic}. We compute the first variation of the loss $\mathcal{L}_{\mathrm{SOCM}}$  with respect to $u$:
    \begin{talign}
        \frac{\delta}{\delta u}\mathbb{E}[\mathcal{L}_{\mathrm{SOCM}}(u,M)](x,t) &= \mathbb{E}[ (u(X^v_t,t) + \sigma(t)^{\top} \omega(t,v,X^v,B,M_t)) \times \alpha(v, X^v, B) | X_t^v = x]~,
    \end{talign}
    where $v$ is arbitrary.
    This is because for any perturbation $\tilde{u}$, we have that
    \begin{talign}
    \begin{split}
        &\lim_{\epsilon \to 0} \mathbb{E} \big[\int_0^{T} \big\| u(X^v_t,t) + \epsilon  \tilde{u}(X^v_t,t)
        + \sigma(t)^{\top} \omega(t,v,X^v,B,M_t) \big\|^2 \, \mathrm{d}t 
        \times \alpha(v, X^v, B) \big] \\ &= \mathbb{E} \big[ \int_0^{T} \big\langle u(X^v_t,t) + \sigma(t)^{\top} \omega(t,v,X^v,B,M_t), \tilde{u}(X^v_t,t) \big\rangle \, \mathrm{d}t 
        \times \alpha(v, X^v, B) \big]
        \\ &= \int_0^{T} \mathbb{E} \big[ \big\langle \mathbb{E}[ (u(X^v_t,t) + \sigma(t)^{\top} \omega(t,v,X^v,B,M_t)) \times \alpha(v, X^v, B) | X_t^v], \tilde{u}(X^v_t,t) \big\rangle \big] \, \mathrm{d}t
    \end{split}
    \end{talign}
    Now, if we set $v = u^*$, we have that $\alpha(u^*, X^{u^*}, B) = \exp \big( - 
    V(X^{u^*}_0,0) \big)$, and thus,
    \begin{talign}
    \begin{split}
        &\frac{\delta}{\delta u}\mathbb{E}[\mathcal{L}_{\mathrm{SOCM}}(u,M)](x,t) \\ &= \mathbb{E}[ (u(X^{u^*}_t,t) + \sigma(t)^{\top} \omega(t,u^*,X^{u^*},B,M_t)) \times \exp \big( - 
        V(X^{u^*}_0,0) \big) | X_t^{u^*} = x] \\ &= \mathbb{E}[ (u(X^{u^*}_t,t) + \sigma(t)^{\top} \omega(t,u^*,X^{u^*},B,M_t)) | X_t^{u^*} = x] 
        \times \mathbb{E}[ \exp \big( - 
        V(X^{u^*}_0,0) \big) | X_t^{u^*} = x]
        \\ &= \big(u(x,t) + \sigma(t)^{\top} \mathbb{E}[\omega(t,u^*,X^{u^*},B,M_t)| X_t^{u^*} = x] \big) 
        \times \mathbb{E}[ \exp \big( - 
        V(X^{u^*}_0,0) \big) | X_t^{u^*} = x]~.  
    \end{split}
    \end{talign}
    Here, the second equality holds by the Markov property, which implies that $X^{u^*}_{[t,T]}$ and $X^{u^*}_0$ are conditionally independent given $X^{u^*}_t$.
    If we follow the same approach for $\mathbb{E}[\mathcal{L}_{\mathrm{UW-SOCM}}]$, we obtain that 
    \begin{talign}
    \begin{split}
        \frac{\delta}{\delta u}\mathbb{E}[\mathcal{L}_{\mathrm{UW-SOCM}}(u,M)](x,t) &= \mathbb{E}[ u(X^v_t,t) + \sigma(t)^{\top} \omega(t,v,X^v,B,M_t) | X_t^v = x] \rvert_{v = \mathrm{stopgrad}(u)} \\ &= u(x,t) + \sigma(t)^{\top} \mathbb{E}[ \omega(t,v,X^v,B,M_t) | X_t^v = x] \rvert_{v = \mathrm{stopgrad}(u)}~.
    \end{split}
    \end{talign}
    Note that $\frac{\delta}{\delta u}\mathbb{E}[\mathcal{L}_{\mathrm{SOCM}}(u^*,M)](x,t)$ is everywhere zero because $u^*$ is the unique optimizer of $\mathbb{E}[\mathcal{L}_{\mathrm{SOCM}}(\cdot,M)]$ for any $M$.
    Since the factor $\mathbb{E}[ \exp \big( - 
    V(X^{u^*}_0,0) \big) | X_t^{u^*} = x]$ is non-negative, the fact that the first variation $\frac{\delta}{\delta u}\mathbb{E}[\mathcal{L}_{\mathrm{SOCM}}(u^*,M)](x,t)$
    is zero everywhere implies that $\frac{\delta}{\delta u}\mathbb{E}[\mathcal{L}_{\mathrm{UW-SOCM}}(u^*,M)](x,t)$ is zero everywhere as well. Hence, $u^*$ is a critical point of $\mathbb{E}[\mathcal{L}_{\mathrm{UW-SOCM}}(\cdot,M)]$.


\section{Proofs of the taxonomy of SOC loss functions}
\label{sec:proofs_taxonomy}
We subdivide the proof of \autoref{thm:main} into showing that pairs of losses share the same gradient in expectation.

\subsection{REINFORCE $\iff$ Discrete Adjoint}

As shown in the proof of \autoref{prop:derivation_REINFORCE}, we have that
\begin{talign}
\begin{split} \label{eq:rf_adj}
    \nabla_{\theta} \mathbb{E}[\mathcal{L}_{\mathrm{RF}}(u_{\theta})] &= 
    \mathbb{E} \big[ \int_0^T \nabla_{\theta} u_{\theta}(X^{u_{\theta}}_t,t) u_{\theta}(X^{u_{\theta}}_t,t) \, \mathrm{d}t \\ &\qquad + 
    \big( \int_0^T 
    \big(\frac{1}{2} \|u_{\theta}(X^{u_{\theta}}_t,t)\|^2 \! + \! f(X^{u_{\theta}}_t,t) \big) \, \mathrm{d}t \! + \! 
    g(X^{u_{\theta}}_T) \big) 
    \times 
    \int_0^T \nabla_{\theta} u_{\theta}(X^{u_{\theta}}_t,t) \, \mathrm{d}B_t 
    \big]
    \\ &= \nabla_{\theta} \mathbb{E} \big[ \int_0^T 
    \big(\frac{1}{2} \|u_{\theta}(X^{u_{\theta}}_t,t)\|^2 \! + \! f(X^{u_{\theta}}_t,t) \big) \, \mathrm{d}t \! + \! 
    g(X^{u_{\theta}}_T) \big]
\end{split}
\end{talign}
and the right-hand side is by definition equal to $\nabla_{\theta} \mathbb{E}[\mathcal{L}_{\mathrm{Disc-Adj}}(u_{\theta})]$. 

\subsection{REINFORCE $\iff$ REINFORCE (future rewards)}
This statement is the analog of a well-known fact in classical RL.
By the definition of $\mathcal{L}_{\mathrm{RFFR}}$ in \eqref{eq:REINFORCE_future_rewards}, we have that 
\begin{talign}
\begin{split} \label{eq:rf_rffr}
    \nabla_{\theta} \mathbb{E}[\mathcal{L}_{\mathrm{RFFR}}(u_{\theta})] &=
    \mathbb{E} \big[ \int_0^T \nabla_{\theta} u_{\theta}(X^{u_{\theta}}_t,t) u_{\theta}(X^{u_{\theta}}_t,t) \, \mathrm{d}t \\ &\qquad
    + 
    \int_0^T \nabla_{\theta} u_{\theta}(X^{u_{\theta}}_t,t) \big( \int_t^T 
    \big(\frac{1}{2} \|u_{\theta}(X^{u_{\theta}}_s,s)\|^2 \! + \! f(X^{u_{\theta}}_s,s) \big) \, \mathrm{d}s \! + \! 
    g(X^{u_{\theta}}_T) \big) \, \mathrm{d}B_t \big].
\end{split}
\end{talign}
Comparing \eqref{eq:rf_adj} with \eqref{eq:rf_rffr}, it suffices to check that 
\begin{talign}
\begin{split}
    &\mathbb{E} \big[
    \int_0^T \nabla_{\theta} u_{\theta}(X^{u_{\theta}}_t,t) \big( \int_0^T 
    \big(\frac{1}{2} \|u_{\theta}(X^{u_{\theta}}_s,s)\|^2 \! + \! f(X^{u_{\theta}}_s,s) \big) \, \mathrm{d}s \! + \! 
    g(X^{u_{\theta}}_T) \big)  \, \mathrm{d}B_t \big] \\ &= \mathbb{E} \big[
    \int_0^T \nabla_{\theta} u_{\theta}(X^{u_{\theta}}_t,t) \big( \int_t^T 
    \big(\frac{1}{2} \|u_{\theta}(X^{u_{\theta}}_s,s)\|^2 \! + \! f(X^{u_{\theta}}_s,s) \big) \, \mathrm{d}s \! + \! 
    g(X^{u_{\theta}}_T) \big)  \, \mathrm{d}B_t \big]
\end{split}
\end{talign}
Or equivalently, it is enough to check that
\begin{talign}
\begin{split} \label{eq:fubini}
    0 &= \mathbb{E} \big[
    \int_0^T \nabla_{\theta} u_{\theta}(X^{u_{\theta}}_t,t) \int_0^t 
    \big(\frac{1}{2} \|u_{\theta}(X^{u_{\theta}}_s,s)\|^2 \! + \! f(X^{u_{\theta}}_s,s) \big) \, \mathrm{d}s  \, \mathrm{d}B_t \big] \\ &= \mathbb{E} \big[ \int_0^T \int_s^T \nabla_{\theta} u_{\theta}(X^{u_{\theta}}_t,t) \, \mathrm{d}B_t \big(\frac{1}{2} \|u_{\theta}(X^{u_{\theta}}_s,s)\|^2 \! + \! f(X^{u_{\theta}}_s,s) \big) \, \mathrm{d}s \big],
\end{split}
\end{talign}
where in the second equality we flipped the order of the deterministic and the stochastic integral using Fubini's theorem. Now, using the tower of expectations property, we can reexpress the right-hand side of \eqref{eq:fubini} as
\begin{talign}
    \mathbb{E} \big[ \int_0^T \mathbb{E}[\int_s^T \nabla_{\theta} u_{\theta}(X^{u_{\theta}}_t,t) \, \mathrm{d}B_t | X^{u_{\theta}}_s] \big(\frac{1}{2} \|u_{\theta}(X^{u_{\theta}}_s,s)\|^2 \! + \! f(X^{u_{\theta}}_s,s) \big) \, \mathrm{d}s \big] = 0.
\end{talign}
This is zero because Ito stochastic integrals are martingales.

\subsection{Discrete Adjoint $\iff$ Continuous Adjoint}
This is shown in \citealt[Sec.~E.1]{domingoenrich2024adjoint}.

\subsection{Continuous Adjoint $\iff$ Cost-SOCM}
This is shown in \autoref{prop:SOCM_cost}.

\subsection{Bonus: REINFORCE (future rewards) $\iff$ Cost-SOCM}
By transitivity, we have already shown that REINFORCE (future rewards) and Cost-SOCM have the same gradients in expectation, because we have shown equality between \textit{(i)} REINFORCE (future rewards) and REINFORCE, \textit{(ii)} REINFORCE and Discrete Adjoint, \textit{(iii)} Discrete Adjoint and Continuous Adjoint, \textit{(iv)} Continuous Adjoint and Cost-SOCM. Still, there is a direct way to show the equivalence between these two losses which helps flesh out their relationship in a clearer way: REINFORCE (future rewards) is simply Cost-SOCM with a particular choice of matrix $M$. The proof follows.

We can write
\begin{talign}
\begin{split} \label{eq:rf_rffr_1}
    &\nabla_{\theta} \mathbb{E}[\mathcal{L}_{\mathrm{Cost-SOCM}}(u_{\theta})] \\ &=
    2 \mathbb{E} \big[ \int_0^T \nabla_{\theta} u_{\theta}(X^{u_{\theta}}_t,t) u_{\theta}(X^{u_{\theta}}_t,t) \, \mathrm{d}t \\ &\quad
    + \int_0^T \nabla_{\theta} u_{\theta}(X^{u_{\theta}}_t,t) \sigma(t)^{\top} \big( \int_t^T M_t(s) \nabla_x (f(X^{u_{\theta}}_s,s) + \frac{1}{2}\|u_{\theta}(X^{u_{\theta}}_s,s)\|^2) \, \mathrm{d}s + M_t(T) \nabla g(X^{u_{\theta}}_T) 
    \\ &\qquad\quad + \! 
    \big( \int_t^T (f(X^{u_{\theta}}_s,s) \! + \! \frac{1}{2}\|u_{\theta}(X^{u_{\theta}}_s,s)\|^2) \, \mathrm{d}s \! + \! g(X^{u_{\theta}}_T) \big) \\ &\qquad\qquad \times \! \big( \int_t^T (M_t(s) \nabla_x (b(X^{u_{\theta}}_s,s) + \sigma(s) u_{\theta}(X^{u_{\theta}}_s,s)) - \partial_s M_t(s)) (\sigma^{-1})^{\top}(s) \mathrm{d}B_s \big) \big) \, dt \big].
\end{split}
\end{talign}
Since $M$ is arbitrary under the conditions that it is continuous and differentiable almost everywhere, and $M_t(t) = \mathrm{Id}$, we are free to set $M$ of the form
\begin{talign}
M^{(h)}_t(s) = 
\begin{cases}
    \big( 1 - \frac{s-t}{h(T-t)} \big) \mathrm{Id} \qquad &\text{if } s \in [t,t+h(T-t)],\\
    0 \qquad &\text{otherwise},
\end{cases}
\end{talign}
we have that
\begin{talign}
\begin{split} \label{eq:rf_rffr_2}
    &\lim_{h \to 0} \int_t^T M^{(h)}_t(s) \nabla_x (f(X^{u_{\theta}}_s,s) + \frac{1}{2}\|u_{\theta}(X^{u_{\theta}}_s,s)\|^2) \, \mathrm{d}s + M^{(h)}_t(T) \nabla g(X^{u_{\theta}}_T) \\ &= \lim_{h \to 0} \int_t^{t+h(T-t)} M^{(h)}_t(s) \nabla_x (f(X^{u_{\theta}}_s,s) + \frac{1}{2}\|u_{\theta}(X^{u_{\theta}}_s,s)\|^2) \, \mathrm{d}s = 0,
\end{split}
\end{talign}
and similarly,
\begin{talign}
\begin{split} \label{eq:rf_rffr_3}
    &\lim_{h \to 0} \int_t^T M_t(s) \nabla_x (b(X^{u_{\theta}}_s,s) + \sigma(s) u_{\theta}(X^{u_{\theta}}_s,s)) (\sigma^{-1})^{\top}(s) \mathrm{d}B_s \\ &= \lim_{h \to 0} \int_t^{t+h(T-t)} M_t(s) \nabla_x (b(X^{u_{\theta}}_s,s) + \sigma(s) u_{\theta}(X^{u_{\theta}}_s,s)) (\sigma^{-1})^{\top}(s) \mathrm{d}B_s = 0.
\end{split}
\end{talign}
Finally,
\begin{talign}
\begin{split} \label{eq:rf_rffr_4}
    &\int_0^T \nabla_{\theta} u_{\theta}(X^{u_{\theta}}_t,t) \sigma(t)^{\top} \big( \int_t^T (f(X^{u_{\theta}}_s,s) \! + \! \frac{1}{2}\|u_{\theta}(X^{u_{\theta}}_s,s)\|^2) \, \mathrm{d}s \! + \! g(X^{u_{\theta}}_T) \big) \\ &\qquad\qquad \times \! \big( \int_t^T  \partial_s M^{(h)}_t(s) (\sigma^{-1})^{\top}(s) \mathrm{d}B_s \big) \, dt
    \\ &= - \int_0^T \nabla_{\theta} u_{\theta}(X^{u_{\theta}}_t,t) \sigma(t)^{\top} \big( \int_t^T (f(X^{u_{\theta}}_{s'},s') \! + \! \frac{1}{2}\|u_{\theta}(X^{u_{\theta}}_{s'},s')\|^2) \, \mathrm{d}s' \! + \! g(X^{u_{\theta}}_T) \big) \\ &\qquad\qquad \times \! \frac{1}{h(T-t)} \big( \int_t^{t+h(T-t)} (\sigma^{-1})^{\top}(s) \mathrm{d}B_s \big) \, dt
    \\ &= - \int_0^T \int_t^{t+h(T-t)} \phi(t) \frac{1}{h(T-t)}  (\sigma^{-1})^{\top}(s) \mathrm{d}B_s \, dt \\ &= - \int_0^{h T} \int_{0}^s \phi(t) \frac{1}{h(T-t)}  (\sigma^{-1})^{\top}(s) \, dB_s - \int_{h T}^{T} \int_{\frac{s-h T}{1-h}}^s \phi(t) \frac{1}{h(T-t)}  (\sigma^{-1})^{\top}(s) \, dB_s
\end{split}
\end{talign}
where we defined $\phi(t) = \nabla_{\theta} u_{\theta}(X^{u_{\theta}}_t,t) \sigma(t)^{\top} \big( \int_t^T (f(X^{u_{\theta}}_{s'},s') \! + \! \frac{1}{2}\|u_{\theta}(X^{u_{\theta}}_{s'},s')\|^2) \, \mathrm{d}s' \! + \! g(X^{u_{\theta}}_T) \big)$ in the second equality, and we used Fubini's theorem in the third equality. Note that
\begin{talign} \label{eq:rf_rffr_5}
    \lim_{h \to 0} \int_0^{h T} \int_{0}^s \phi(t) \frac{1}{h(T-t)}  (\sigma^{-1})^{\top}(s) \, dB_s = 0,
\end{talign}
and
\begin{talign}
\begin{split} \label{eq:rf_rffr_6}
    &\lim_{h \to 0} \int_{h T}^{T} \int_{\frac{s-h T}{1-h}}^s \phi(t) \frac{1}{h(T-t)}  (\sigma^{-1})^{\top}(s) \, dB_s \\ &= \lim_{h \to 0} \int_{0}^{T} 
    \big(s - \frac{s-h T}{1-h} \big) \phi(s) \frac{1}{h(T-s)}  (\sigma^{-1})^{\top}(s) \, dB_s
    \\ &= \lim_{h \to 0} \int_{0}^{T} 
    \big(s - \frac{s-h T}{1-h} \big) \phi(s) \frac{1}{h(T-s)}  (\sigma^{-1})^{\top}(s) \, dB_s
    \\ &= 
    \int_{0}^{T} 
    \phi(s) 
    (\sigma^{-1})^{\top}(s) \, dB_s
    \\ &= \int_{0}^{T} \nabla_{\theta} u_{\theta}(X^{u_{\theta}}_s,s) \sigma(s)^{\top} \big( \int_s^T (f(X^{u_{\theta}}_{s'},s') \! + \! \frac{1}{2}\|u_{\theta}(X^{u_{\theta}}_{s'},s')\|^2) \, \mathrm{d}s' \! + \! g(X^{u_{\theta}}_T) \big) (\sigma^{-1})^{\top}(s) \, dB_s
    \\ &= \int_{0}^{T} \nabla_{\theta} u_{\theta}(X^{u_{\theta}}_s,s) \big( \int_s^T (f(X^{u_{\theta}}_{s'},s') \! + \! \frac{1}{2}\|u_{\theta}(X^{u_{\theta}}_{s'},s')\|^2) \, \mathrm{d}s' \! + \! g(X^{u_{\theta}}_T) \big) \, dB_s
\end{split}
\end{talign}
If we plug \eqref{eq:rf_rffr_5} and \eqref{eq:rf_rffr_6} into \eqref{eq:rf_rffr_4}, and then \eqref{eq:rf_rffr_2}, \eqref{eq:rf_rffr_3} and \eqref{eq:rf_rffr_4} into \eqref{eq:rf_rffr}, we obtain that
\begin{talign}
\begin{split} \label{eq:rf_rffr_7}
    &\nabla_{\theta} \mathbb{E}[\mathcal{L}_{\mathrm{Cost-SOCM}}(u_{\theta})] \\ &=
    2 \mathbb{E} \big[ \int_0^T \nabla_{\theta} u_{\theta}(X^{u_{\theta}}_t,t) u_{\theta}(X^{u_{\theta}}_t,t) \, \mathrm{d}t \\ &\quad
    + 
    \int_{0}^{T} \nabla_{\theta} u_{\theta}(X^{u_{\theta}}_t,t) \big( \int_t^T (f(X^{u_{\theta}}_{s},s) \! + \! \frac{1}{2}\|u_{\theta}(X^{u_{\theta}}_{s},s)\|^2) \, \mathrm{d}s \! + \! g(X^{u_{\theta}}_T) \big) \, dB_t \big] \\ &= 2 \nabla_{\theta} \mathbb{E}[\mathcal{L}_{\mathrm{RFFR}}(u_{\theta})],
\end{split}
\end{talign}
where the last equality concludes the proof and holds by the definition of $\mathcal{L}_{\mathrm{RFFR}}$ (see \eqref{eq:rf_rffr}).

\subsection{Adjoint Matching $\iff$ Work-SOCM}
This is shown in \autoref{prop:work_adjoint_app}.

\subsection{SOCM $\iff$ Cross Entropy}
This result was proven in \cite{domingoenrich2023stochastic}. Namely, \cite[Prop. B.6(ii)]{domingoenrich2023stochastic} shows that
\begin{talign} 
\begin{split} \label{eq:L_CE_characterization}
    \mathbb{E}[\mathcal{L}_{\mathrm{CE}}(u)] &= \frac{
    1}{2} \mathbb{E} \big[ \int_0^T \|u^*(X^{u^*}_t,t) - u(X^{u^*}_t,t)\|^2 \, \mathrm{d}t 
    \exp \big( \! - \! 
    V(X^{u^*}_0, 0) \big) \big]~,
\end{split}
\end{talign}
and \cite[Prop. 3.3]{domingoenrich2023stochastic} shows that
\begin{talign} 
    \begin{split} \label{eq:L_u_decomposition}
        &\mathbb{E}[\mathcal{L}_{\mathrm{SOCM}}(u,M)] =
        \mathbb{E} \big[ \! \int_0^{T} \! \big\| u(X^{u^*}_t,t) \! - \! u^*(X^{u^*}_t,t) \big\|^2 \, \mathrm{d}t 
        \exp \big( - 
        V(X^{u^*}_0, 0) \big) \big] \! + \!
        \mathrm{CondVar}(-\sigma(t)^{\top} \omega;M)~,
    \end{split}
    \end{talign}
    where the second term does not depend on $u$:
    \begin{talign} 
    \begin{split} \label{eq:var_w_M}
        &\mathrm{CondVar}(-\sigma(t)^{\top} \omega;M)
        \! := \! \mathbb{E} \big[ \int_0^{T} \big\| \sigma(t)^{\top} \big( \omega(t,v,X^v,B,M_t) \! - \! 
        \frac{\mathbb{E} [\omega(t,v,X^v,B,M_t) \alpha(v,X^v,B) | X_t^v, t ]}{\mathbb{E} [\alpha(v,X^v,B) | X_t^v, t ]}
        \big) \big\|^2 \, \mathrm{d}t \, \alpha(v,X^v,B) \big].
    \end{split}
    \end{talign}

\subsection{SOCM $\iff$ SOCM-Adjoint}
This is also shown by \cite{domingoenrich2023stochastic}, putting together different results in their work. We reproduce their argument. The proof sketch of their Theorem 3.1 considers the following loss (up to constant factors and constant terms):
\begin{talign} 
\begin{split} \label{eq:loss_tilde_L_sketch}
    \tilde{\mathcal{L}}(u) \!
    &= \! \mathbb{E} \big[ 
    \! \int_0^{T} \! \big( \big\| u(X_t,t) \big\|^2 
    \! - \! 2\langle u(X_t,t), u^*(X_t,t) \rangle 
    \big) \, \mathrm{d}t 
    \! \times \! \exp \big( \! - \! 
    \int_0^T f(X_t,t) \, \mathrm{d}t \! - \! 
    g(X_T) \big) \big] \\
    &= \! \mathbb{E} \big[ 
    \! \int_0^{T} \! \big( \big\| u(X_t,t) \big\|^2 
    \! - \! 2\langle u(X_t,t), u^*(X_t,t) \rangle 
    \big) \, \mathrm{d}t \\ &\qquad 
    \! \times \! \exp \big( \! - \! 
    \int_0^T f(X_t,t) \, \mathrm{d}t \! - \! 
    g(X_T) \big) \big].
    \\ &= \mathbb{E} \big[ \! \int_0^{T} \! \big\| u(X_t,t) \big\|^2  \! \times \! \exp \big( \! - \! 
    \int_0^T f(X_t,t) \, \mathrm{d}t \! - \! 
    g(X_T) \big) \big] \\ &\qquad - 2 
    \mathbb{E} \big[
    \int_0^{T} \big\langle u(X_t,t), \sigma(t)^{\top} 
    \nabla_{x} \mathbb{E}\big[ \exp \big( \! - 
    \! \int_t^T \! f(X_s,s) \, \mathrm{d}s \! - \! 
    g(X_T) \big) \big| X_t = x \big] \big\rangle \, \mathrm{d}t \\ &\qquad\qquad\qquad \times \exp \big( \! - \! 
    \int_0^t f(X_s,s) \, \mathrm{d}s \big) \big].
\end{split}
\end{talign}
On the one-hand, the path-wise reparameterization (\autoref{lem:cond_exp_rewritten}, \cite[Prop.~C.3]{domingoenrich2023stochastic}) trick yields
\begin{talign} 
\begin{split} \label{eq:cond_exp_rewritten_applied}
    &\nabla_{x} \mathbb{E}\big[ \exp \big( \! - \! 
    \int_t^T f(X_s,s) \, \mathrm{d}s - 
    g(X_T) \big) \big| X_t = x \big] \\ &= \! \mathbb{E}\big[ \big( \! - \! 
    \int_t^T M_t(s) \nabla_x f(X_s,s) \, \mathrm{d}s \! - \! 
    M_t(T) \nabla g(X_T) \\ &\qquad  
    + \! 
    \int_t^T (M_t(s) \nabla_x b(X_s,s) \! - \! \partial_s M_t(s)) (\sigma^{-1})^{\top} (s) \mathrm{d}B_s \big) \\ &\qquad \times \exp \big( \! - \! 
    \int_t^T f(X_s,s) \, \mathrm{d}s \! - \! 
    g(X_T) \big) 
    \big| X_t = x \big],
\end{split}
\end{talign}
and plugging this into the right-hand side of \eqref{eq:loss_tilde_L_sketch} yields the SOCM loss $\mathcal{L}_{\mathrm{SOCM}}$ after a change of process from $X$ to $X^v$.
On the other-hand, the same gradient can be estimated using the adjoint method (\autoref{lem:adjoint_method_sdes}, \cite[Lemma~C.6]{domingoenrich2023stochastic}):
\begin{talign}
\begin{split}
    &\nabla_{x} \mathbb{E} \big[ \exp \big( - 
    \int_t^T f(X_t,t) \, dt - 
    g(X_T) \big) | X_t = x \big] \\ &= - 
    \mathbb{E} \big[ a(t,X) \exp \big( - 
    \int_0^T f(X_t,t) \, dt - 
    g(X_T) \big) | X_t = x\big].
\end{split}
\end{talign}
where $a(t,X)$ is the solution of the lean adjoint ODE \eqref{eq:lean_adjoint_1}-\eqref{eq:lean_adjoint_2} (without the SLT term). Plugging this into the right-hand side of \eqref{eq:loss_tilde_L_sketch} yields the SOCM-Adjoint loss $\mathcal{L}_{\mathrm{SOCM-Adj}}$ after a change of process from $X$ to $X^v$.

\subsection{Log-variance $\iff$ Moment}
This is a simple observation that was first made by \cite{nüsken2023solving}. When each batch contains $m$ trajectories, the empirical log-variance loss reads
\begin{talign}
    \mathcal{L}^{\mathrm{log},m}_{\mathrm{Var}_v}(u) &:= \frac{1}{m} \sum_{i=1}^{m} 
    \big( \tilde{Y}^{u,v,(i)}_T - 
    g(X^{v,(i)}_T) \big)^2 
    - \big( \frac{1}{m} \sum_{i=1}^{m} 
    \big( \tilde{Y}^{u,v,(i)}_T - 
    g(X^{v,(i)}_T) \big) \big)^2.
\end{talign}
Thus,
\begin{talign}
\begin{split}
    \mathbb{E}\big[ \mathcal{L}^{\mathrm{log},m}_{\mathrm{Var}_v}(u) \big] &= 
    \mathbb{E}\big[
    \frac{1}{m} \sum_{i=1}^{m} 
    \big( \tilde{Y}^{u,v,(i)}_T - 
    g(X^{v,(i)}_T) \big)^2 
    \big] 
    \\ &\qquad 
    - \mathbb{E}\big[ 
    \big( \frac{1}{m} \sum_{i=1}^{m} 
    \big( \tilde{Y}^{u,v,(i)}_T -
    g(X^{v,(i)}_T) \big) \big)^2
    \big] \\
     &= 
    \mathbb{E}\big[
    \big( \tilde{Y}^{u,v}_T - 
    g(X^{v}_T) \big)^2 
    \big] 
    - \mathbb{E}\big[ 
    \big( \frac{1}{m} \sum_{i=1}^{m} 
    \big( \tilde{Y}^{u,v,(i)}_T - 
    g(X^{v,(i)}_T) \big) \big)^2
    \big].
\end{split}
\end{talign}
Note that as $m \to +\infty$, $\mathbb{E}\big[ \big( \frac{1}{m} \sum_{i=1}^{m} \big( \tilde{Y}^{u,v,(i)}_T - 
g(X^{v,(i)}_T) \big) \big)^2 \big] \to \mathbb{E}\big[ \tilde{Y}^{u,v}_T - 
g(X^{v}_T) \big]^2$, which means that 
\begin{talign}
    \lim_{m \to \infty} \mathcal{L}^{\mathrm{log},m}_{\mathrm{Var}_v}(u) = \mathbb{E}\big[
    \big( \tilde{Y}^{u,v}_T - 
    g(X^{v}_T) \big)^2 
    \big] - \mathbb{E}\big[ \tilde{Y}^{u,v}_T - 
    g(X^{v}_T) \big]^2.
\end{talign}
And for any $u,v$, we have that 
\begin{talign}
\begin{split}
    &\min_{y_0 \in \R} \mathcal{L}_{\mathrm{Mom}_v}(u,y_0) = \min_{y_0 \in \R} \mathbb{E}[(\tilde{Y}^{u,v}_T + y_0 - 
    g(X^v_T))^2] \\ &= \mathbb{E}[(\tilde{Y}^{u,v}_T - 
    g(X^v_T) - \mathbb{E}[\tilde{Y}^{u,v}_T - 
    g(X^v_T)])^2] \\ &= \mathbb{E}\big[
    \big( \tilde{Y}^{u,v}_T - 
    g(X^{v}_T) \big)^2 
    \big] - \mathbb{E}\big[ \tilde{Y}^{u,v}_T - 
    g(X^{v}_T) \big]^2,
\end{split}
\end{talign}
which shows that when $y_0$ is optimized instantaneously, the two losses coincide.

\section{Experimental details}
\label{subsec:detail_exp}
We use the experimental setup of \cite{domingoenrich2023stochastic}, which was partially based on that of \cite{nüsken2023solving}. We use the same hyperparameters and same architectures as \cite{domingoenrich2023stochastic}. We include an additional setting: \textsc{Double well, hard}, which is given by
\begin{align}
    b(x,t) = - \nabla \Psi(x), \quad
    \Psi(x) = \sum_{i=1}^d \kappa_i (x_i^2 - 1)^2, \quad g(x) = \sum_{i=1}^{d} \nu_i (x_i^2 - 1)^2, \quad f(x) = 0, \quad \sigma_0 = \mathrm{I}, 
\end{align}
where $d=10$, and $\kappa_i = 5$, $\nu_i = 6$ for $i \in \{1, 2, 3\}$ and $\kappa_i = 1$, $\nu_i = 2$ for $i \in \{4, \dots, 10\}$. We set $T=1$, $\lambda=1$ and $x_{\mathrm{init}} = 0$. Our \textsc{Double well, easy} setting corresponds to the \textsc{Double well} setting from \cite{domingoenrich2023stochastic}, and it corresponds to setting $\kappa_i = 5$, $\nu_i = 3$ for $i \in \{1, 2, 3\}$ and $\kappa_i = 1$, $\nu_i = 1$ for $i \in \{4, \dots, 10\}$.

\section{Additional experiments on simple SOC settings}
\label{subsec:additional_exp}

\autoref{fig:control_l2_errors_adjoint_methods} shows the control $L^2$ error curves for Adjoint Matching, Continuous Adjoint and Discrete Adjoint, each with and without the Sticking the Landing (STL) trick. We make the following observations:
\begin{itemize}
    \item The Discrete Adjoint loss with STL performs substantially worse than the the Discrete Adjoint loss without STL, while the Continuous Adjoint and Adjoint Matching losses with STL do better. The reason for this is that the Discrete Adjoint loss with STL directly optimizes the loss
    \begin{talign}
        \mathcal{L}(u ; X^u) := \int_0^T \big(\frac{1}{2} |u(X^u_t,t)|^2 \! + \! f(X^u_t,t) \big) \, \mathrm{d}t \! + \! g(X^u_T) + \int_0^T \langle u(X^u_t,t), \mathrm{d}B_t \rangle,
    \end{talign}
    while the gradients of the Continuous Adjoint loss with STL can be regarded as the gradient:
    \begin{talign}
        \mathcal{L}(u ; X^u) := \int_0^T \big(\frac{1}{2} |u(X^u_t,t)|^2 \! + \! f(X^u_t,t) \big) \, \mathrm{d}t \! + \! g(X^u_T) + \int_0^T \langle \text{stopgrad}(u)(X^u_t,t), \mathrm{d}B_t \rangle,
    \end{talign}
    That is, while we backpropagate through the stochastic integral, we do not take gradients with respect to the explicit evaluations of $u$.
    \item The Continuous Adjoint and Adjoint Matching losses without STL perform similarly, and they also perform similarly with STL (and slightly better than without STL). This contrasts with the behavior observed by \citealt[Tab.~2]{domingoenrich2024adjoint}, where Adjoint Matching clearly outperforms the Continuous Adjoint method, arguably thanks to a lower gradient variance.
\end{itemize}

\autoref{fig:control_l2_errors_SOCM_methods} and \autoref{fig:control_l2_errors_SOCM_cost_work} show control $L^2$ error curves for several kinds of SOCM-based algorithms. We use different terms when labeling the algorithms:
\begin{itemize}
    \item The term $M_t = \mathrm{I}$ means that the reparameterization matrices $M$ have not been trained and simply been set to the identity. 
    \item The term \textit{Diag.} means that the reparameterization matrices $M$ have been parameterized as diagonal matrices. This allows to save memory and time at the expense of a less expressive model.
    \item The term \textit{Scalar} means that the reparameterization matrices $M$ have been parameterized as scalar multiples of the identity matrix. This allows to save memory and time at the expense of a less expressive model.
    \item The term $M_t(T) = 0$ means that the architecture of the reparameterization matrices enforces that they are equal to zero at the terminal time $T$. This choice makes it possible to handle non-differentiable terminal costs $g$ (although all the terminal costs in our examples are differentiable). See the code for more details on the architecture.
\end{itemize}
We make the following observations about \autoref{fig:control_l2_errors_SOCM_methods}:
\begin{itemize}
    \item When it converges, Unweighted SOCM converges at the fastest rate in all setting in which it converges, but it fails to converge in the Double Well hard setting. This is consistent with the fact that Unweighted SOCM may not have the optimal control as the unique critical point (\autoref{prop:UW_SOCM_optimality}). 
    \item Unweighted SOCM Diag. $M_t(T) = 0$ performs very poorly, which is expected because it does not use any information about the terminal cost $g$, which means that it cannot converge to the optimal control.
    \item SOCM performs similarly to Unweighted SOCM in the Quadratic OU easy, Double Well easy and Linear settings, but quite differently in the other two. In the Quadratic OU hard setting, SOCM takes significantly longer to converge, because early on the importance weight $\alpha$ has a lot of variance and the gradient is very noisy. In the Double Well hard setting, SOCM is the best method, but Unweighted SOCM fails to converge. Hence, we conclude that SOCM struggles with high cost values, while Unweighted SOCM struggles with multimodality.
    \item The Scalar, Diag. and $M_t=\mathrm{I}$ versions perform worse than the standard one, in general.
    \item SOCM-Adjoint performs worse than SOCM in general.
\end{itemize}
We make the following observations about \autoref{fig:control_l2_errors_SOCM_cost_work}: 
\begin{itemize}
    \item SOCM-Work tends to perform a bit better than SOCM-Cost, except for the Quadratic OU easy setting, where both achieve low error and SOCM-Cost is slightly better. This observation is consistent with Adjoint Matching (which has a the same gradient in expectation as SOCM-Work) performing better than Continuous Adjoint (which has the same gradient in expectation as SOCM-Cost). 
    \item While taking diagonal reparameterization matrices does not cause a significant performance drop in most cases (except for SOCM-Work in the Linear OU setting), enforcing $M_t(T) = 0$ does yield higher control $L^2$ errors.
\end{itemize}

\begin{figure} 
    \centering
    \makebox[\textwidth][c]{%
        \includegraphics[width=0.5\textwidth]{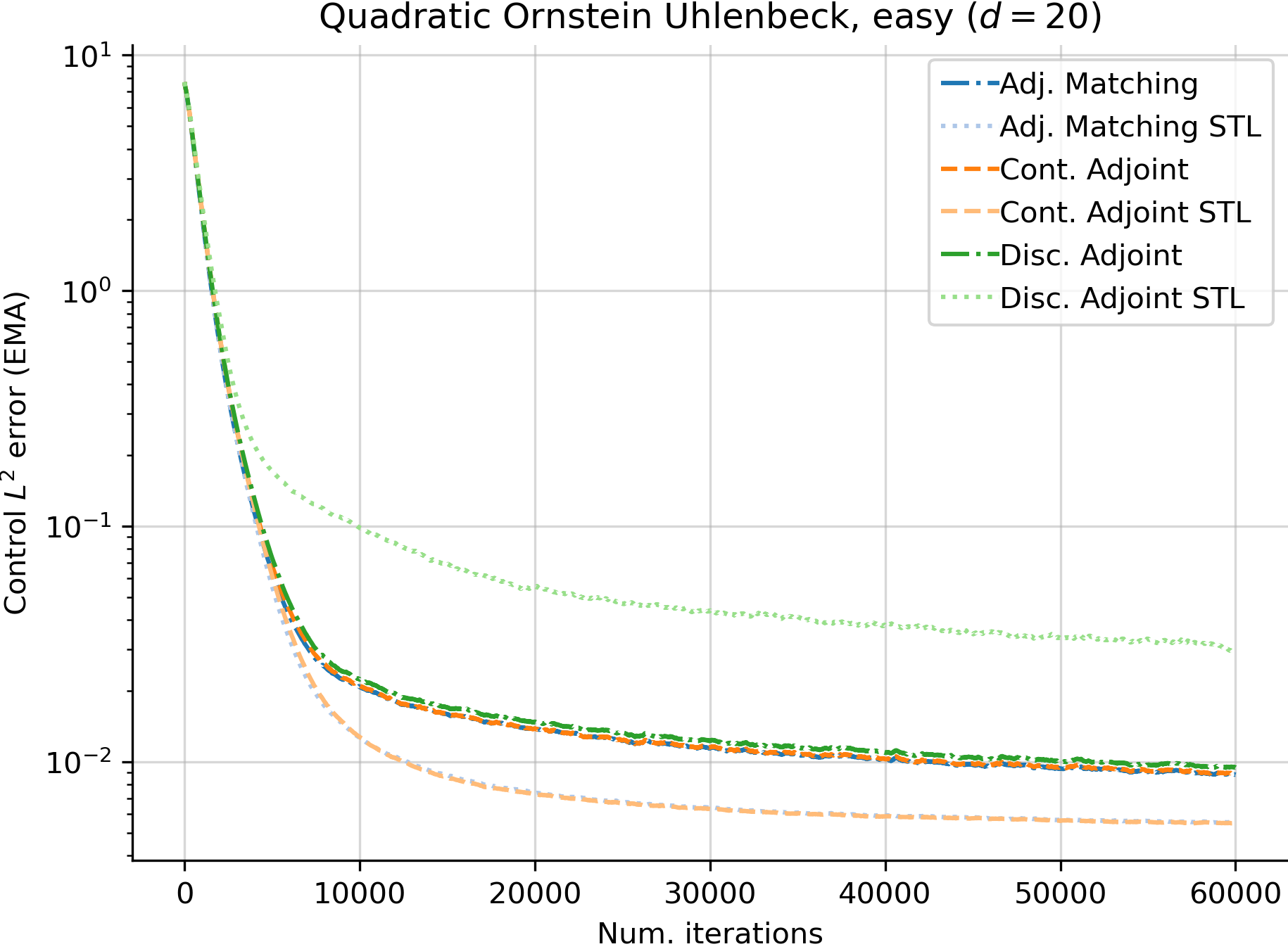}%
        \includegraphics[width=0.5\textwidth]{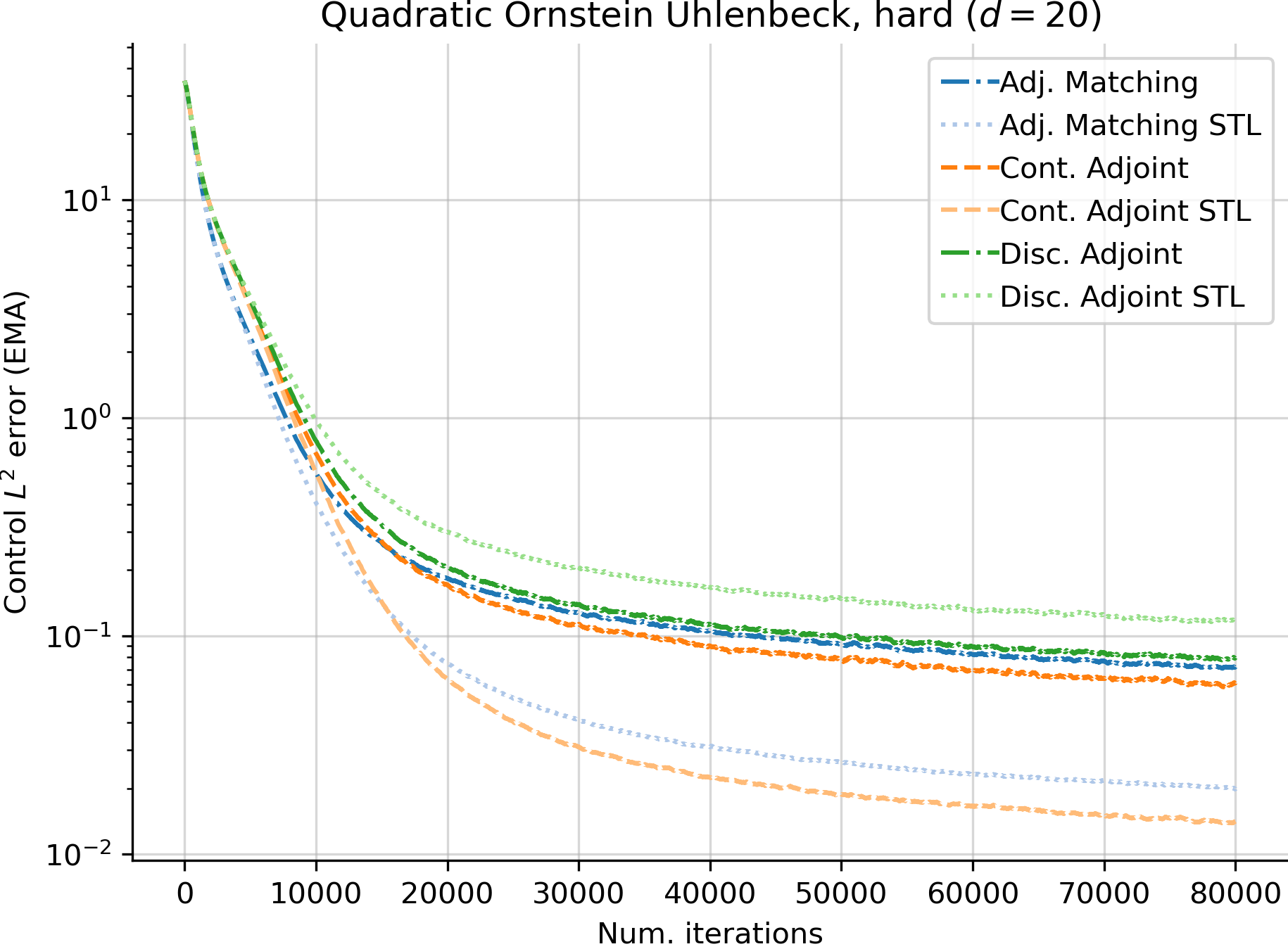}%
    }\\
    \makebox[\textwidth][c]{%
        \includegraphics[width=0.5\textwidth]{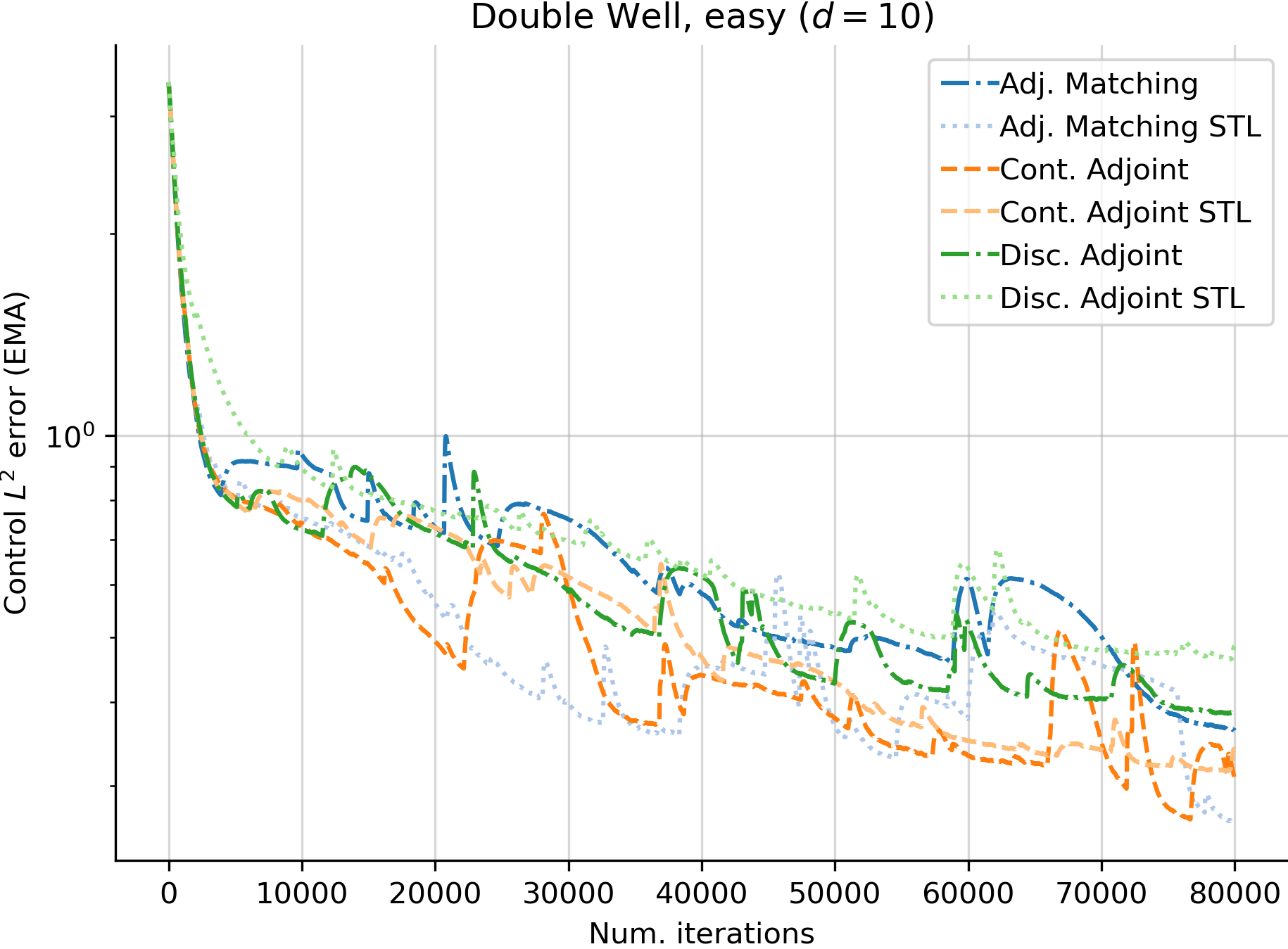}%
        \includegraphics[width=0.5\textwidth]{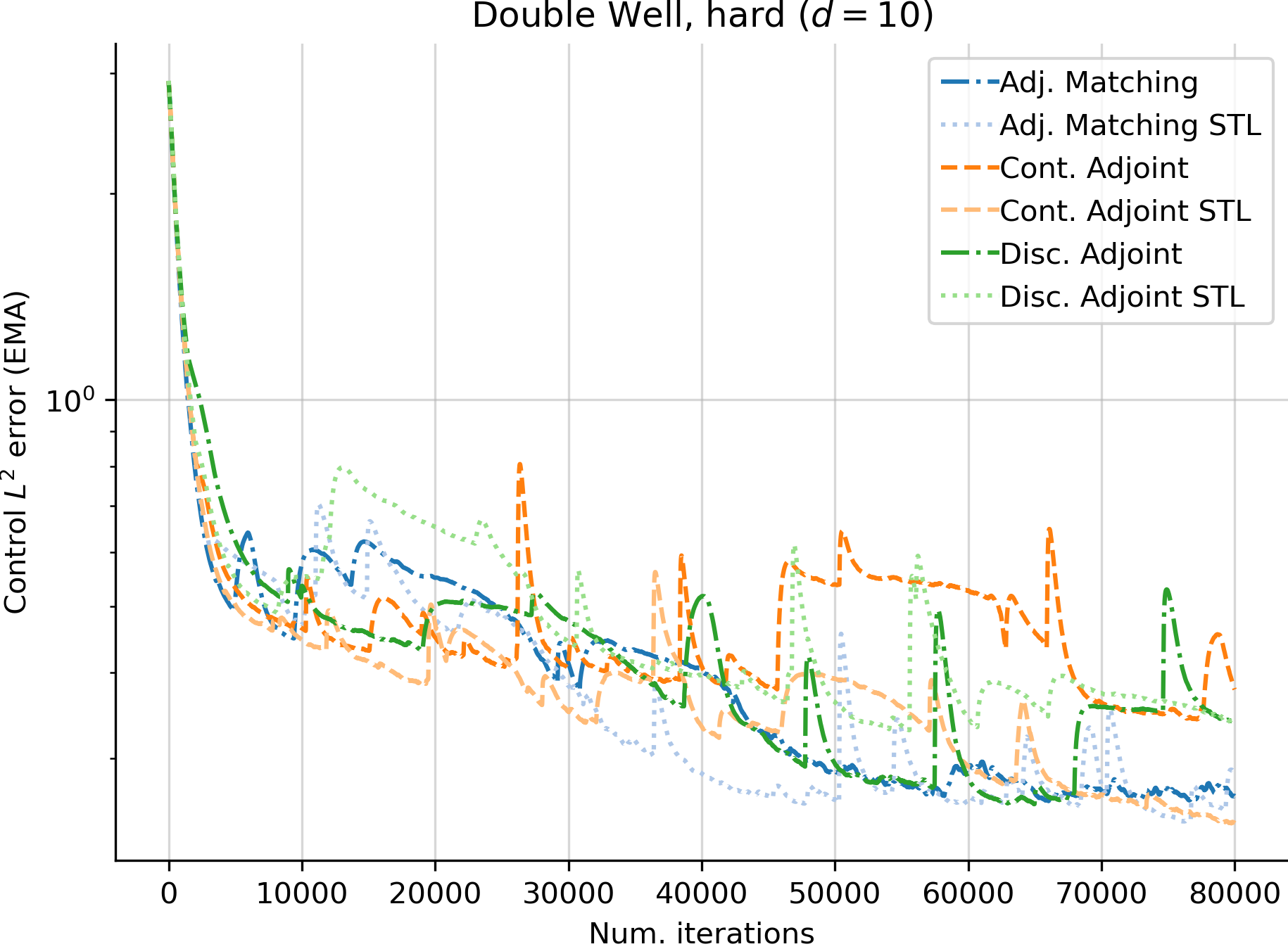}%
    }\\
    \makebox[\textwidth][c]{%
        \includegraphics[width=0.5\textwidth]{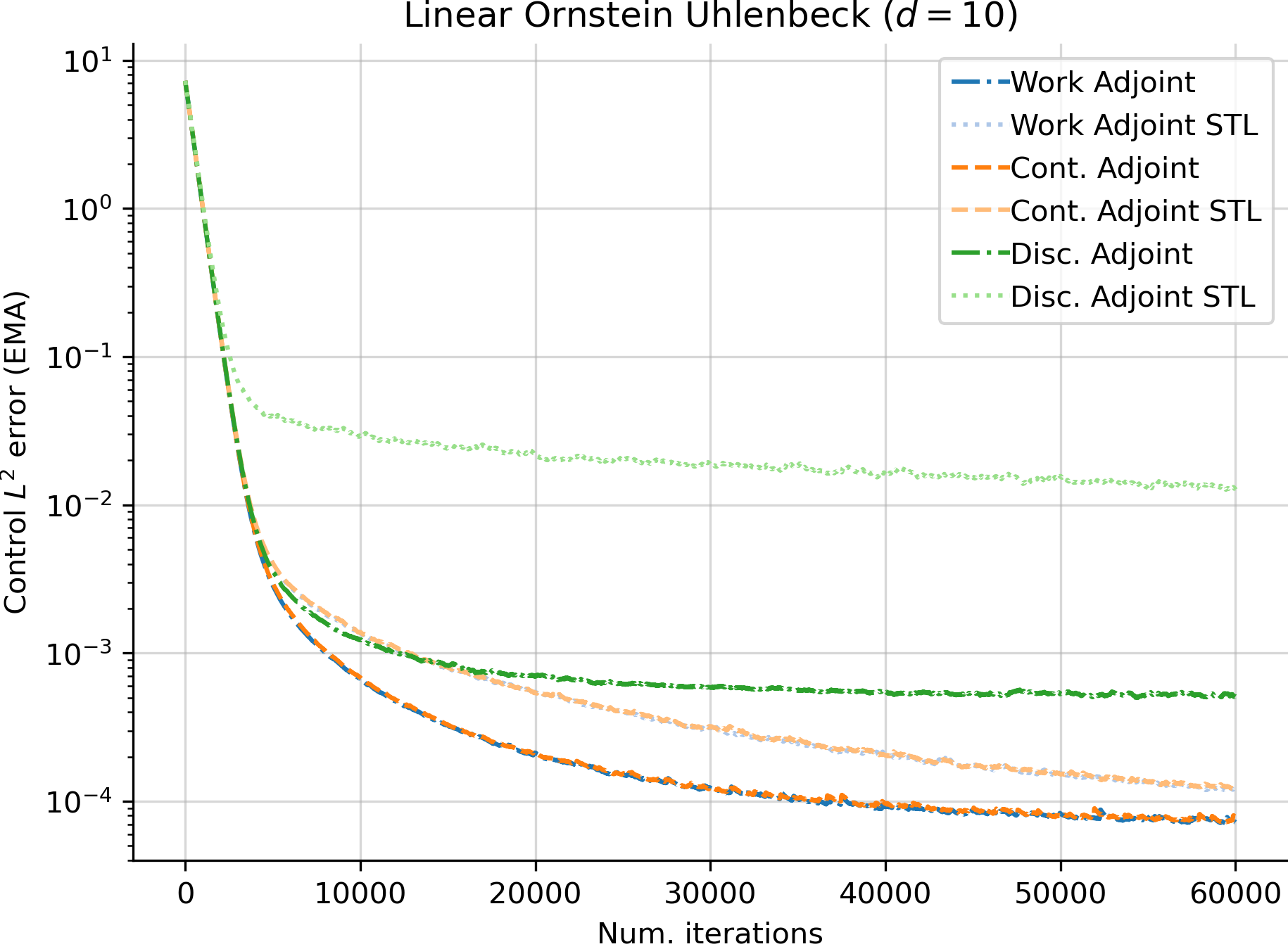}%
    }
    \caption{Control $L^2$ error incurred by the Adjoint Matching, Continuous Adjoint and Discrete Adjoint losses (with and without the Sticking The Landing trick), on five different settings.}
    \label{fig:control_l2_errors_adjoint_methods}
\end{figure}


\begin{figure} 
    \centering
    \makebox[\textwidth][c]{%
        \includegraphics[width=0.5\textwidth]{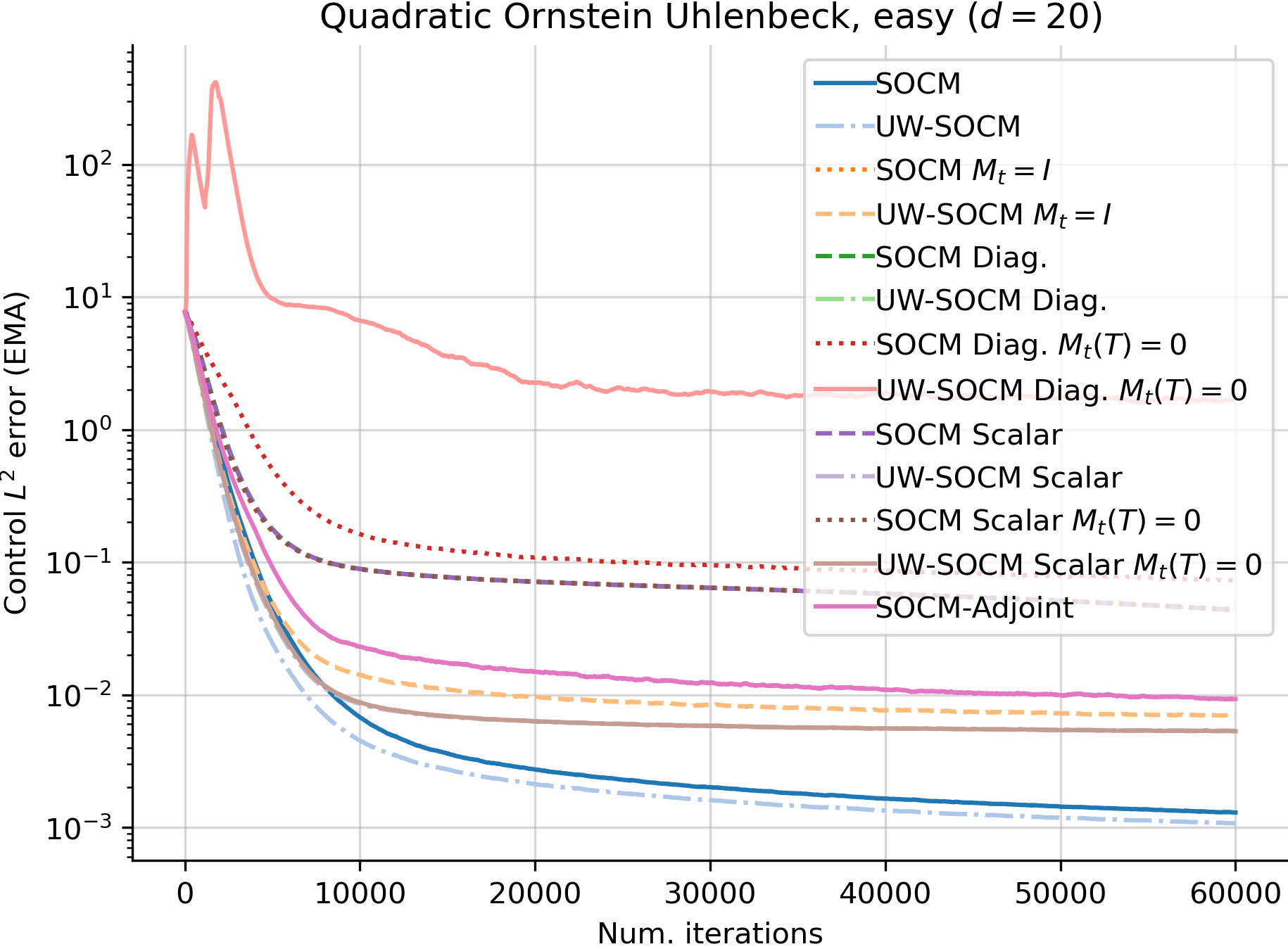}%
        \includegraphics[width=0.5\textwidth]{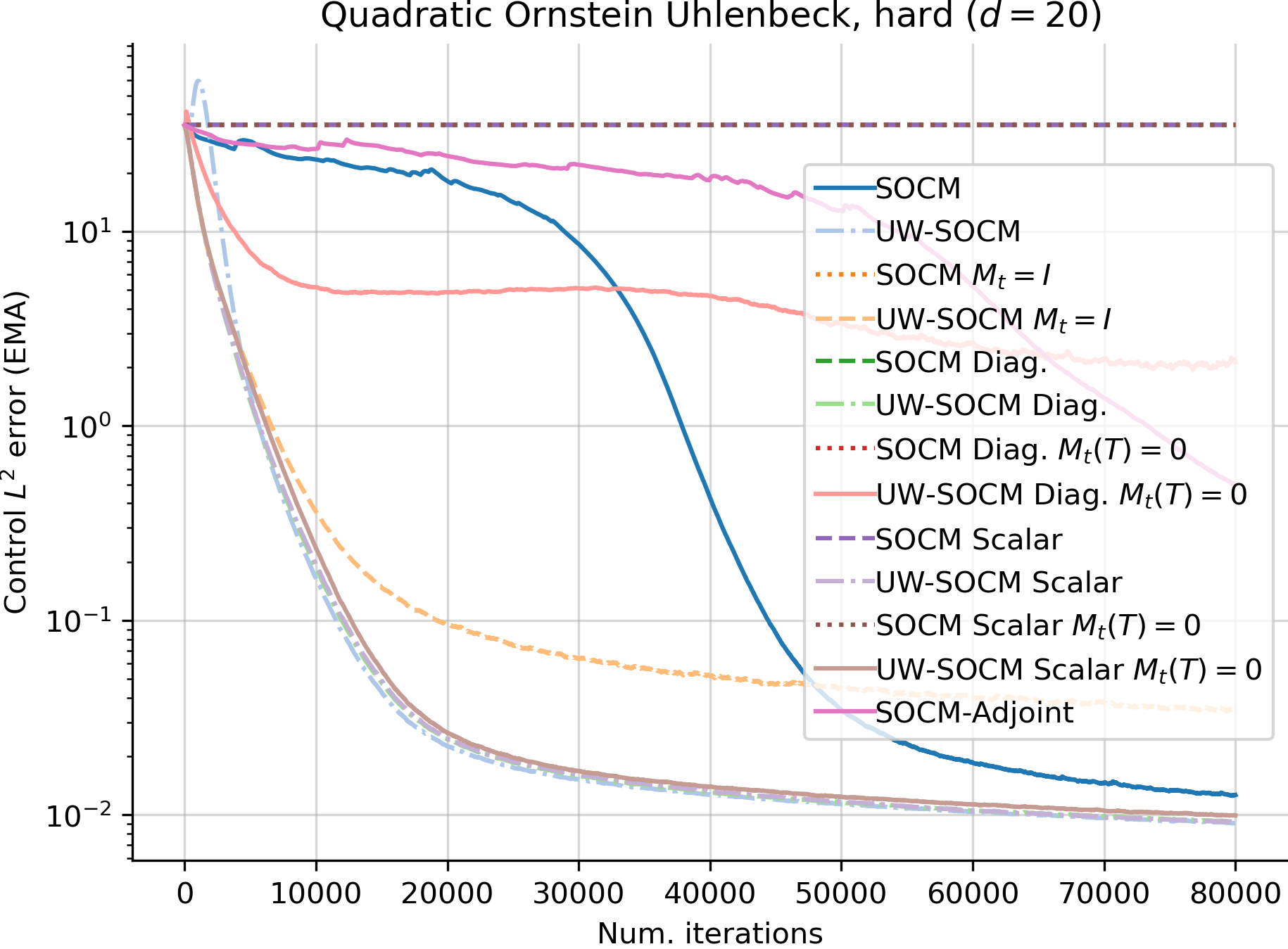}%
    }\\
    \makebox[\textwidth][c]{%
        \includegraphics[width=0.5\textwidth]{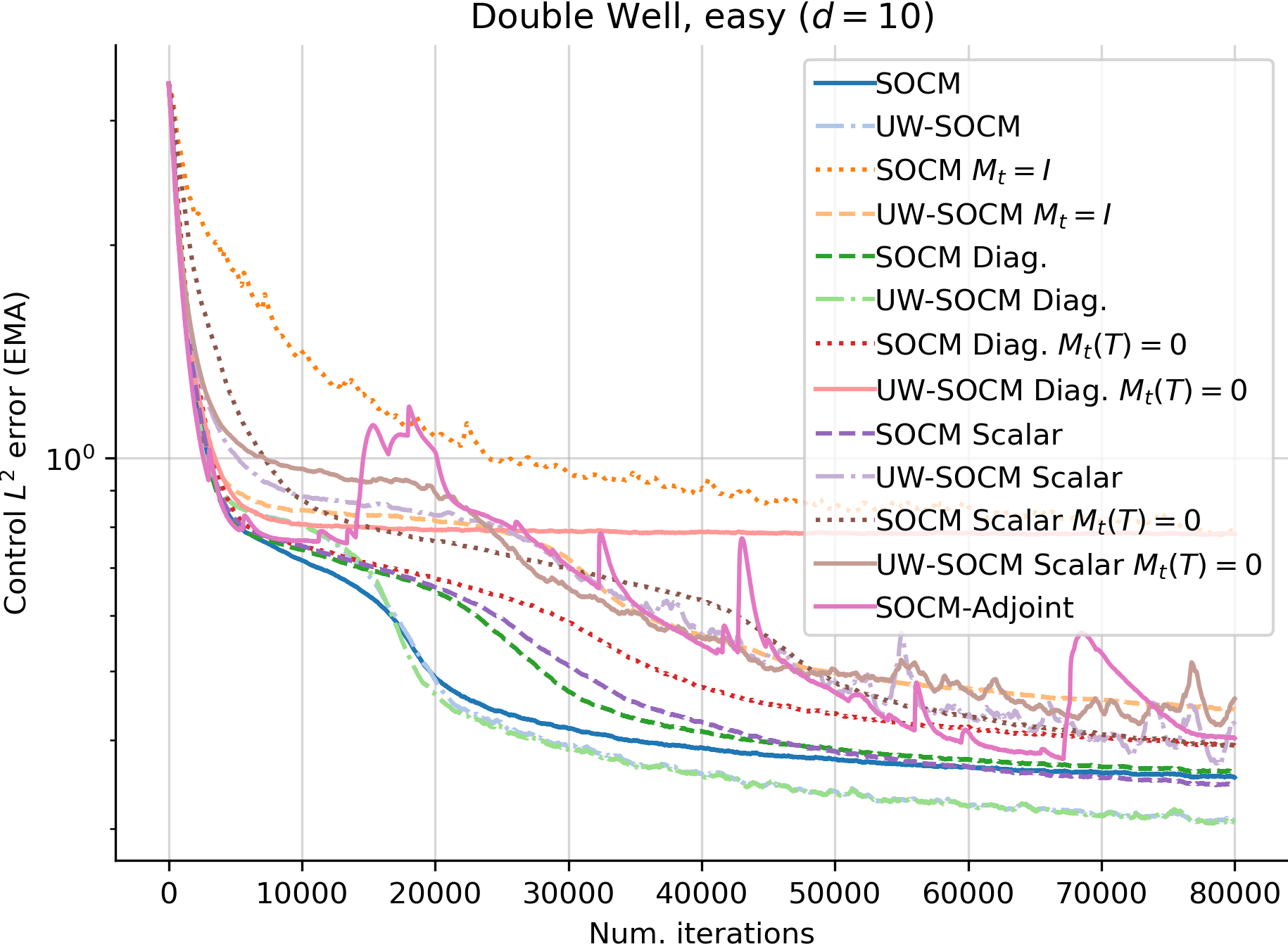}%
        \includegraphics[width=0.5\textwidth]{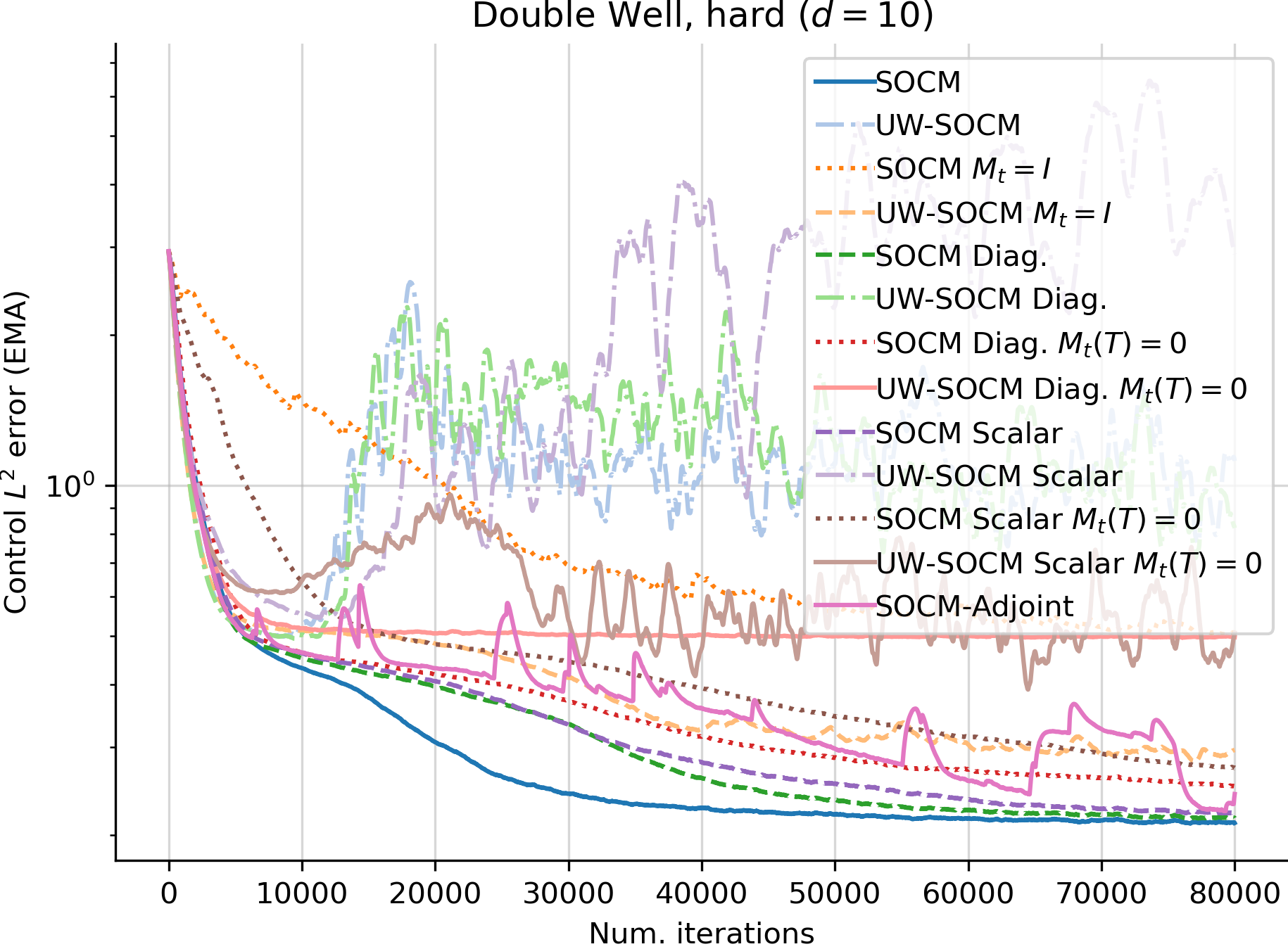}%
    }\\
    \makebox[\textwidth][c]{%
        \includegraphics[width=0.5\textwidth]{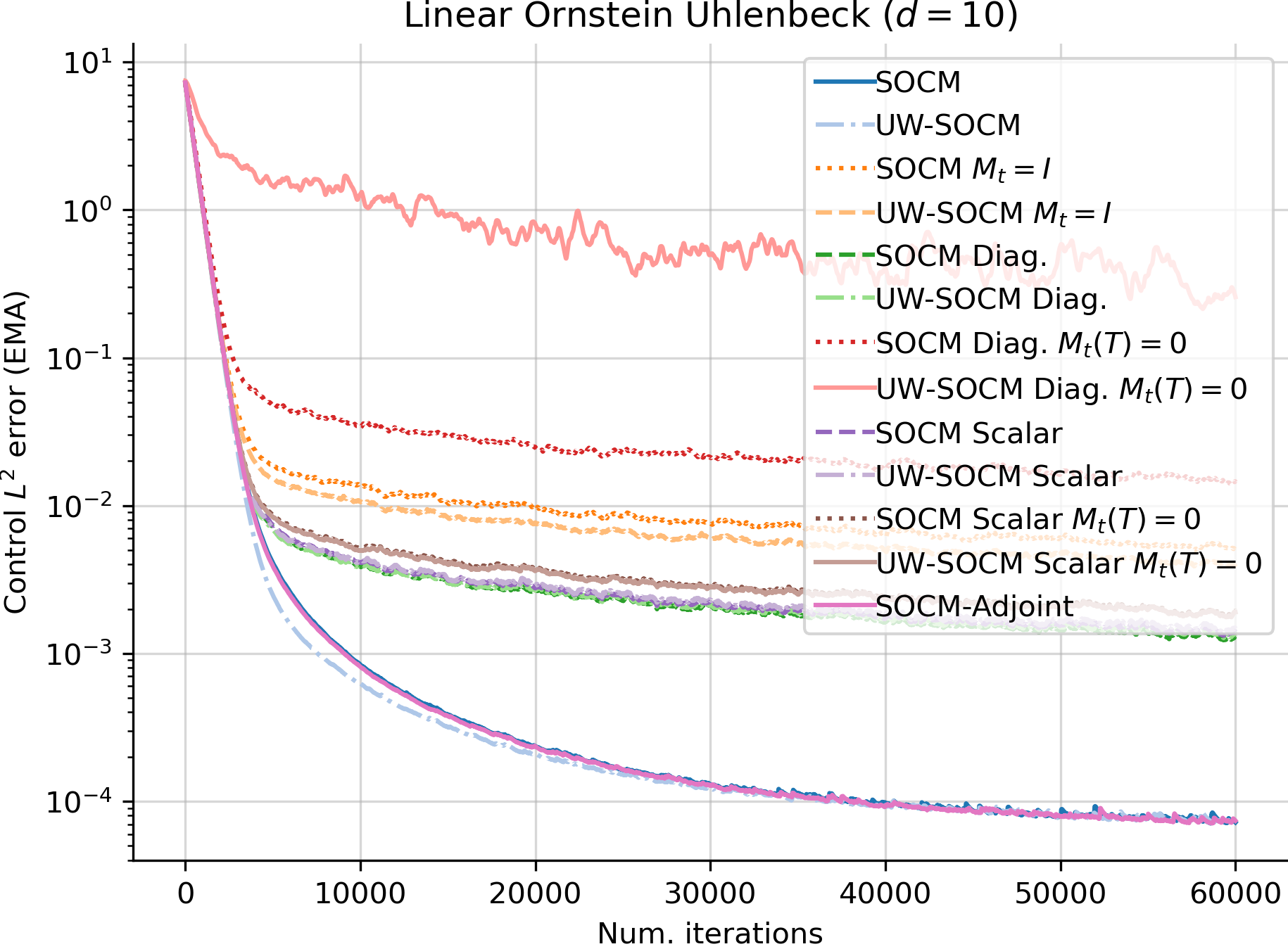}%
    }
    \caption{Control $L^2$ error incurred by each loss function throughout training, on five different settings.}
    \label{fig:control_l2_errors_SOCM_methods}
\end{figure}


\begin{figure} 
    \centering
    \makebox[\textwidth][c]{%
        \includegraphics[width=0.5\textwidth]{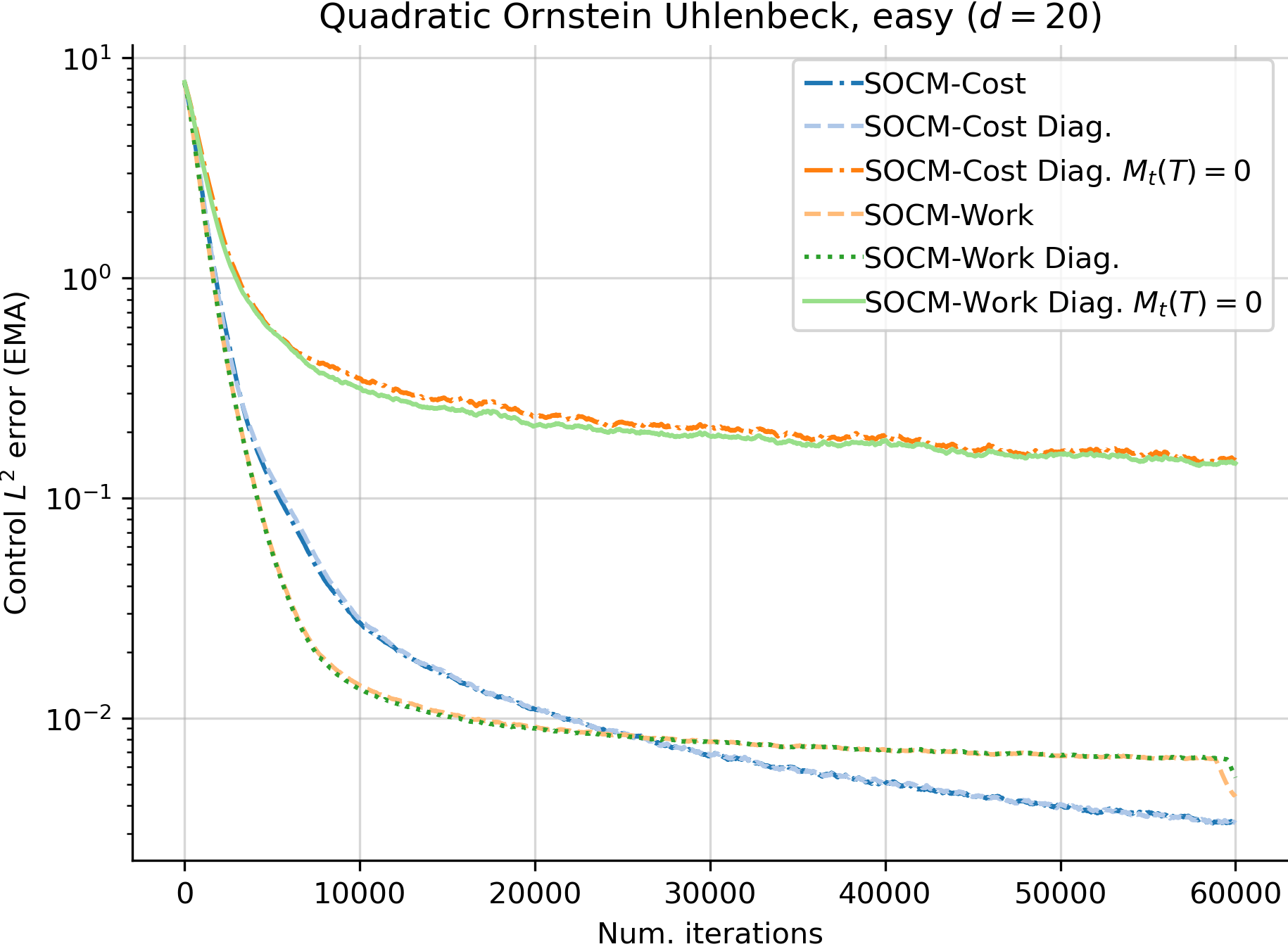}%
        \includegraphics[width=0.5\textwidth]{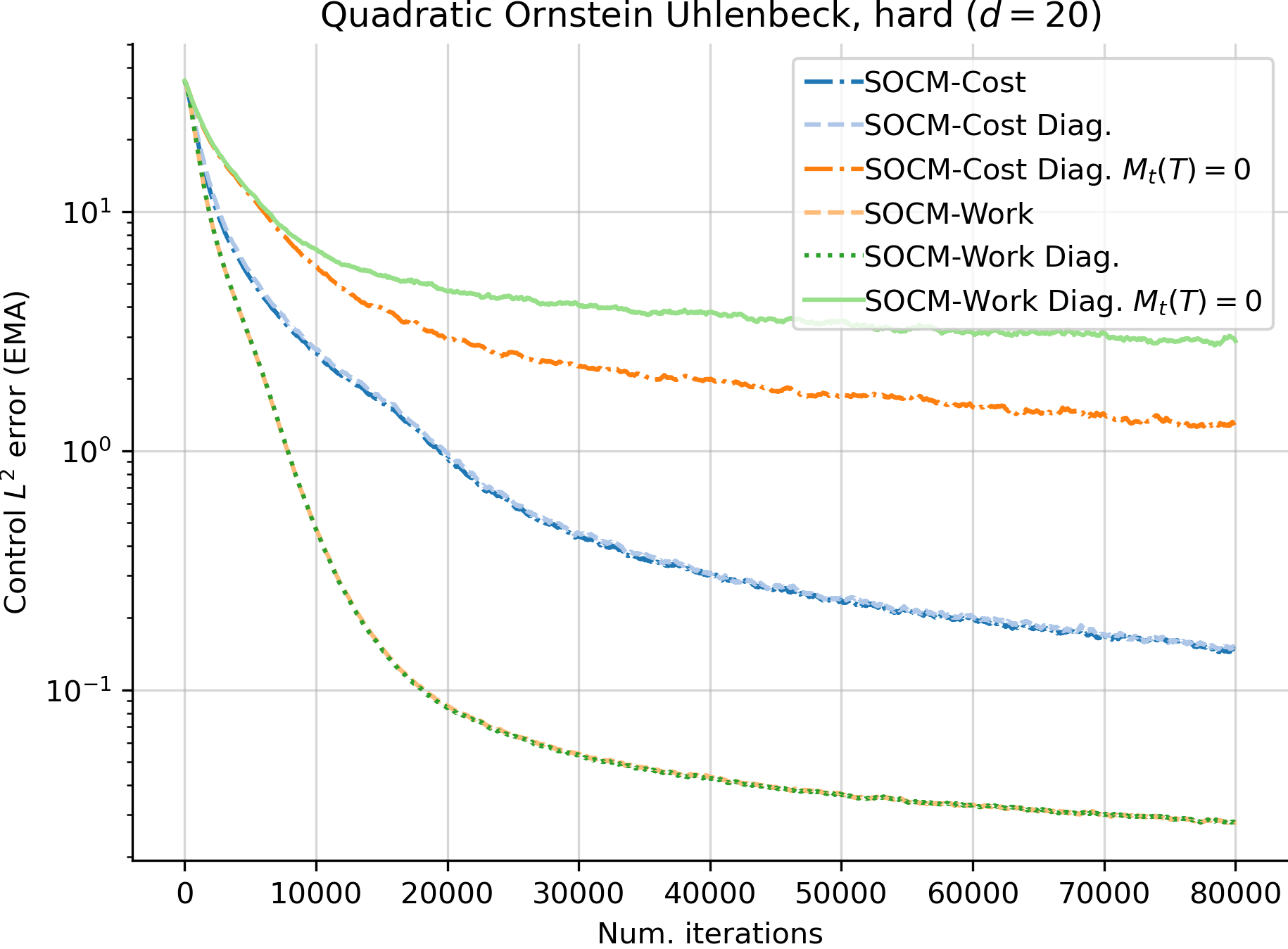}%
    }\\
    \makebox[\textwidth][c]{%
        \includegraphics[width=0.5\textwidth]{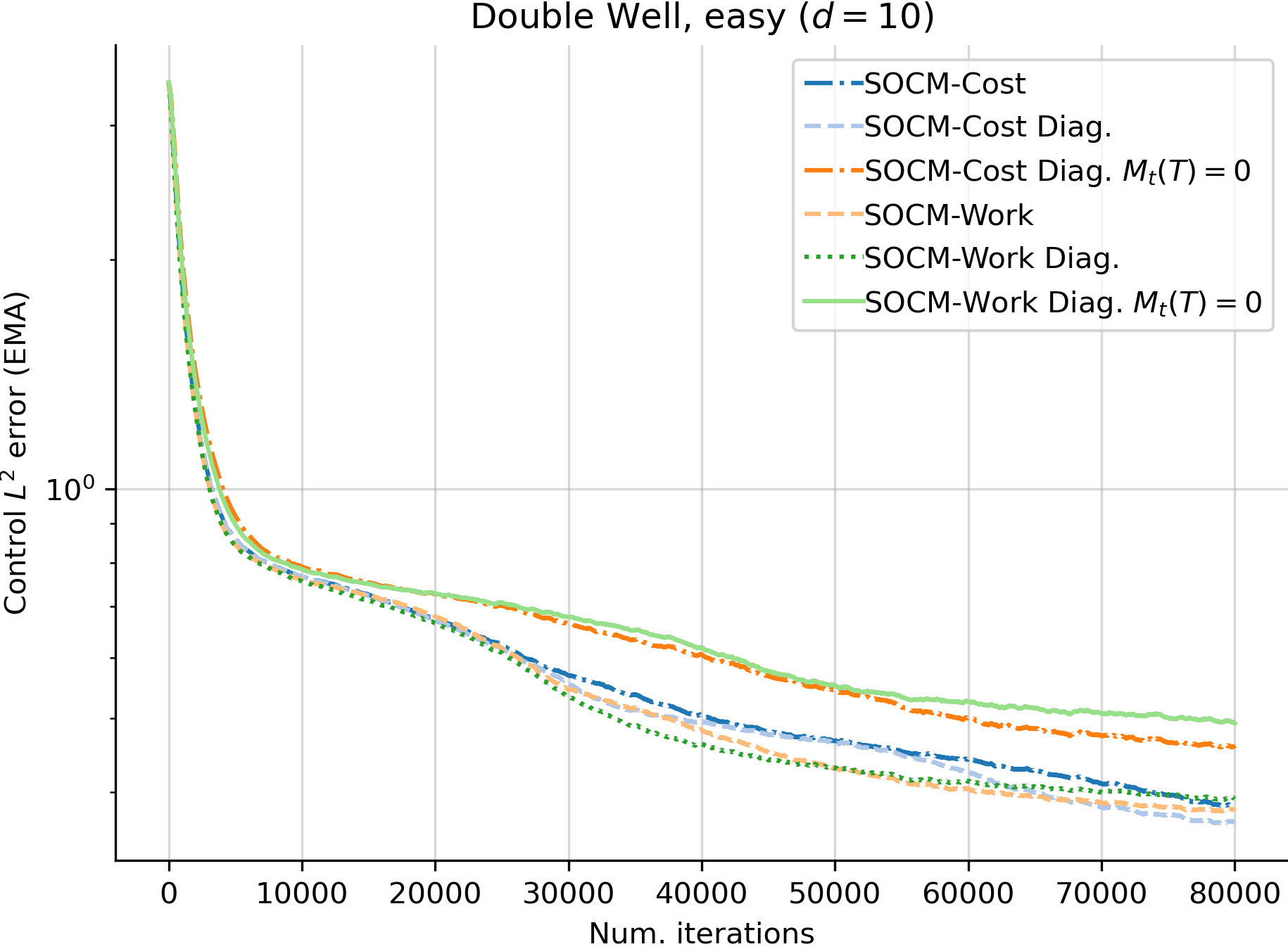}%
        \includegraphics[width=0.5\textwidth]{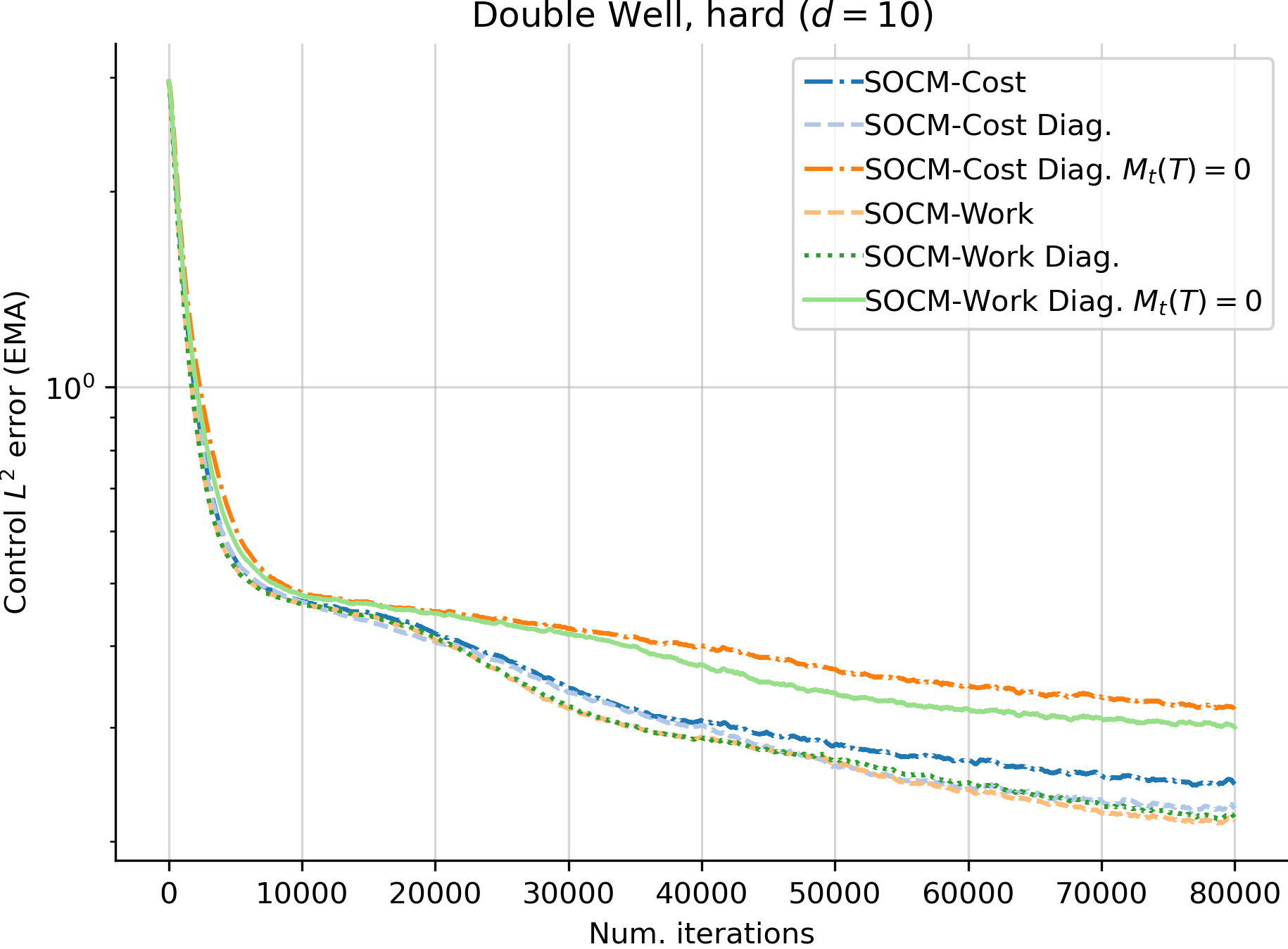}%
    }\\
    \makebox[\textwidth][c]{%
        \includegraphics[width=0.5\textwidth]{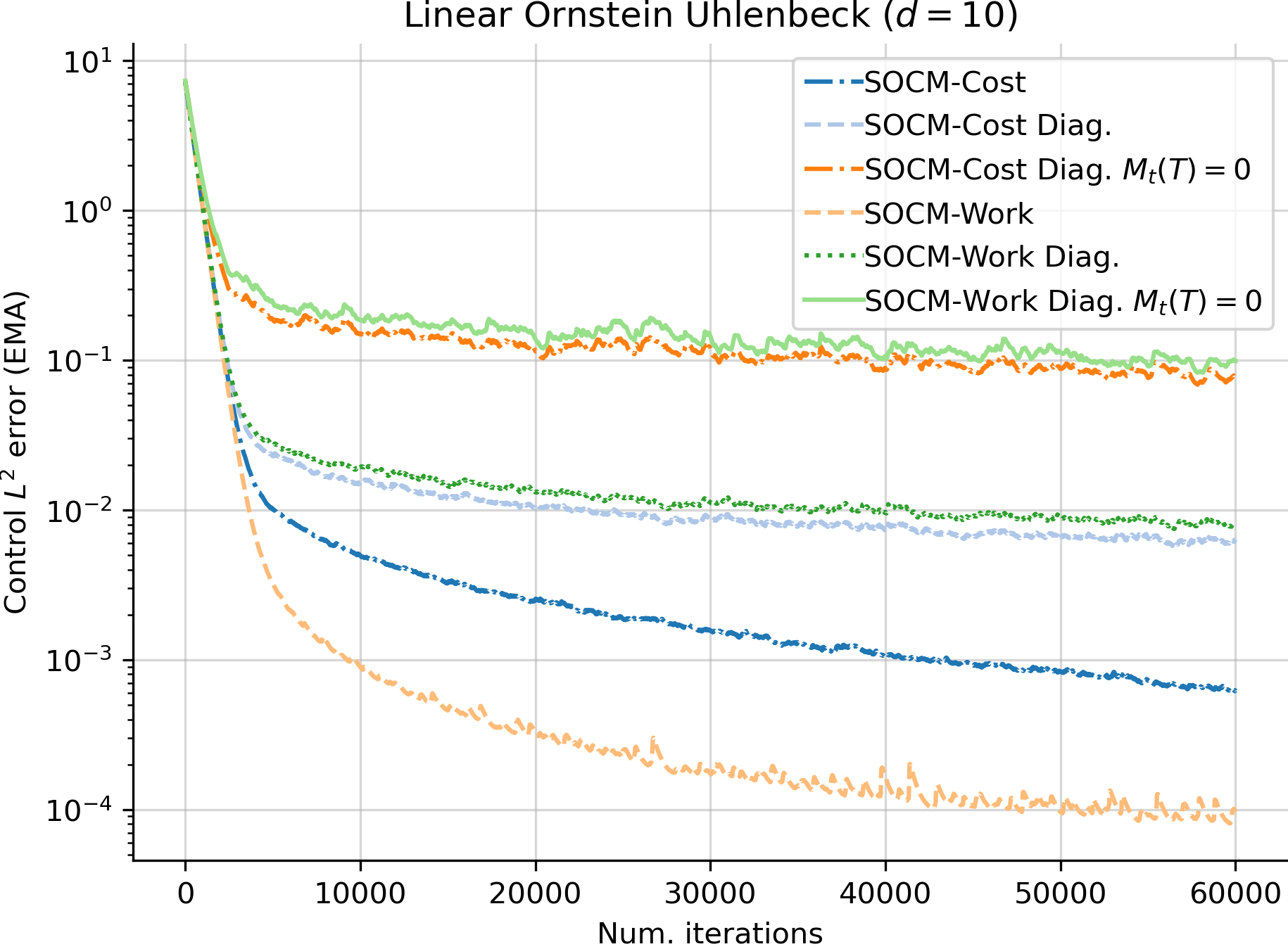}%
    }
    \caption{Control $L^2$ error incurred by each loss function throughout training, on five different settings.}
    \label{fig:control_l2_errors_SOCM_cost_work}
\end{figure}

